%% file: main.tex
\documentclass[11pt]{article}
\usepackage{amsfonts,amsmath,amsthm,dsfont,comment,mathpazo}
\usepackage[inline, shortlabels]{enumitem}
\usepackage[margin=1in]{geometry}

\usepackage[utf8]{inputenc}
\usepackage[T1]{fontenc}    
\usepackage[colorlinks,citecolor=blue]{hyperref}       
\usepackage{url}            
\usepackage{bm}

\usepackage{tikz}
\graphicspath{ {./imgs/} }
\usepackage{caption}
\usepackage{subcaption}
\usepackage{float}
\usetikzlibrary{positioning}
\tikzset{box1/.style={draw,rectangle,fill=blue!20,minimum height=0.58cm, minimum width=0.85cm}}
\tikzstyle{na} = [baseline=-.5ex]
\tikzstyle{every picture}+=[remember picture]
\tikzset{
  main/.style={circle, minimum size = 5mm, thick, draw =black!80, node distance = 10mm},
  connect/.style={-latex, thick},
  box/.style={rectangle, draw=black!100}
}

\usepackage{color}
\newcommand{\unlearning}{unlearning }
\newcommand{\accw}{ \mathring{\w}}
\newcommand{\adjust}{\emph{recomputation}}

\usepackage[square,sort,comma,numbers]{natbib}
\usepackage{commands}

\title{Machine Unlearning via Algorithmic Stability}

\author{Enayat Ullah 
\thanks{Johns Hopkins University. Email: \url{enayat@jhu.edu}}
\qquad
Tung Mai\thanks{Adobe Research. Email: \url{tumai@adobe.com}}
\qquad
Anup Rao \thanks{Adobe Research. Email: \url{anuprao@adobe.com}}
\qquad
Ryan Rossi\thanks{Adobe Research. Email: \url{rrossi@adobe.com}}  
\qquad
Raman Arora \thanks{Johns Hopkins University. Email: \url{arora@cs.jhu.edu}}
}

\usepackage{tocloft}
\date{}
\begin{document}

\maketitle

\pagenumbering{arabic}
\input{sections/abstract}
\newpage

\renewcommand{\baselinestretch}{0.98}\normalsize
\setcounter{tocdepth}{2}
{\small
\hypersetup{linkcolor=mydarkblue,linktocpage}
\tableofcontents
}
\renewcommand{\baselinestretch}{1.0}\normalsize
\newpage

\input{sections/introduction}
\input{sections/probSetup}
\input{sections/mainresults}
\input{sections/mainideas}
\input{sections/algorithms}
\input{sections/proofmainresults}
\input{sections/discussion}

\section*{Acknowledgements}
This research was supported in part by NSF BIGDATA award IIS-1838139 and NSF CAREER award IIS-1943251.

\bibliographystyle{alpha}
\bibliography{main}

\appendix

\input{sections/additionalrelatedwork}
\input{sections/proofslearning}
\input{sections/proofsunlearning}
\input{sections/runtime-with-proofs}
\input{sections/otheralgorithms}
\input{sections/lowerboundsempiricalrisk}
\input{sections/populationrisk}
\input{sections/approximateunlearning}
\input{sections/experiments}
\end{document}

%% file: sections/abstract.tex
\begin{abstract}
    We study the problem of machine unlearning and identify a notion of algorithmic stability, Total Variation (TV) stability, which we argue, is suitable for the goal of \textit{exact} unlearning. For convex risk minimization problems, we design TV-stable algorithms based on noisy Stochastic Gradient Descent (SGD). Our key contribution is the design of corresponding \emph{efficient} unlearning algorithms, which are based on constructing a (maximal) coupling of Markov chains for the noisy SGD procedure.
    To understand the trade-offs between accuracy and unlearning efficiency, 
    we give upper and lower bounds on excess empirical and populations risk of TV stable algorithms for convex risk minimization. 
    Our techniques generalize to arbitrary non-convex functions, and our algorithms are differentially private as well.
\end{abstract}

%% file: sections/introduction.tex
\section{Introduction}
\label{sec:intro}

User data is employed in data analysis for various tasks such as drug discovery, as well as in third-party services for tasks such as recommendations.
With such practices becoming ubiquitous, there is a growing concern that sensitive personal data can get compromised.
This has resulted in a push for
broader awareness of data privacy and ownership.
These efforts have led to several regulatory bodies enacting laws such as the European Union General Data Protection Regulation (GDPR) and California Consumer Act, which empowers the user with (among other things) the right to request to have their personal data be 
\emph{deleted} (see Right to be forgotten, Wikipedia \cite{wiki:Right_to_be_forgotten}).
However, currently it is unclear what it means to be \emph{forgotten} or have the data \emph{deleted} in a rigorous sense.
Nonetheless, there is a reasonable expectation that merely deleting a user's data from the database without undoing the computations derived from the said data is insufficient. 
In the settings where user data are directly utilized to build machine learning models for, say prediction tasks, a reasonable criterion is that the system's state is \textit{adjusted} to what it would have been if the user data were absent to begin with -- this is the criteria we adopt in our work. We refer to it as \emph{exact unlearning} (see \cref{defn:exact-unlearning} for a formal definition).

A straightforward way to comply with the requirement of exact unlearning is to recompute (or retrain, in context of machine learning). This, however, is often computationally expensive, and hence the goal is to design \emph{efficient} unlearning algorithms - herein and afterwards, we use efficient, in the context of unlearning, to mean that the runtime is smaller than recompute time.
Most of the prior work focuses on specific \textit{structured} problems (eg: linear regression, clustering etc.) and the proposed unlearning algorithms carefully leverage this structure for efficiency - see a discussion of useful algorithmic principles like linearity, modularity etc. which enable efficient unlearning, in \cite{ginart2019making}.
In this work, we consider a \textit{larger} class of problems: smooth convex empirical risk minimization (ERM), which includes many machine learning problems, eg: linear and logistic regression, as special cases (see \cref{sec:setup} for definitions). 
This class of smooth convex ERM is sufficiently rich and arguably lacks structure useful in prior works.
Gradient-based optimization is a key algorithmic technique used for smooth convex ERM (and most of machine learning and even beyond).
Unlike prior works where the (learning) algorithms for the specific problems (like linear regression) are tailored enough to the problem to be amenable to efficient unlearning, an optimization method hardly has any such useful structure.
In particular, the sequential nature of gradient descent makes it challenging to design non-trivial unlearning algorithms, at least those which satisfy an \emph{exact} unlearning criterion.
To elaborate, if the point to be deleted participates in some iteration, then the subsequent steps are dependent on the to-be-deleted point, and there is no known way but to redo the computations. 
It is natural to then ask whether we can design unlearning algorithms with non-trivial guarantees for this class of smooth convex ERM problems.

\paragraph{Problem statement (informal).} 
We consider a smooth convex ERM problem over a given initial dataset, in a setting wherein we observe a stream of edits (insertion or deletion) to the dataset. The goal is to design a learning algorithm that outputs an initial model and a (corresponding) unlearning algorithm that updates the model after an edit request. We require the following properties to hold: 1) exact unlearning --  at every time point in the stream, the output model is indistinguishable from what we would have obtained if trained on the \emph{updated} dataset (i.e., without the deleted sample or with the inserted sample); 2) the unlearning runtime should be \emph{small}; 3) the output models should be sufficiently accurate (measured in empirical risk).

\subsection{Our contributions}

\paragraph{Total variation stability.} We develop new algorithmic principles which \textit{enable} exact unlearning in very general settings. In particular, we propose a notion of algorithmic stability, called total variation (TV) stability - an algorithmic property, which for any problem, yields an \emph{in-principle} exact unlearning algorithm.
Such an algorithm might not be efficiently implementable computationally or due to the data access restriction (sequential nature of edits). 
To demonstrate the generality of our framework, we discuss, in \cref{sec:qsgd} how the previous work of \cite{ginart2019making} for unlearning in $k$-means clustering using randomized quantization can be interpreted as a special case of our framework - a TV stable method, followed by efficient coupling based unlearning.
We also note that the notion of TV-stability has appeared before in \cite{bassily2016algorithmic}, although in the seemingly unrelated context of adaptive data analysis.

\paragraph{Convex risk minimization.}
We make the above ideas of TV stability constructive in the special case of smooth convex ERM problems.
To elaborate, we give a TV stable learning algorithm, and a corresponding \textit{efficient} exact unlearning algorithm for smooth convex ERM. 
Informally, for $n$ data points, and $d$ dimensional model and a given $0<\rho\leq 1$, our method retrains only on $\rho$ fraction of edit requests, while satisfying exact unlearning and maintaining that the accuracy (excess empirical risk) is at most $\min\bc{\frac{1}{\sqrt{\rho n}}, \br{\frac{\sqrt{d}}{\rho n}}^{4/5}}$ (see \cref{thm:main-result} for precise statement). This implies that for the (useful) regime of accuracy greater than
$\min\bc{\frac{1}{\sqrt{n}}, \br{\frac{\sqrt{d}}{n}}^{4/5}}$, our algorithms provide a \textit{strict} improvement over the only known method of re-computation - see remarks after \cref{thm:main-result} for details.
Furthermore, we also give excess population risk bounds by leveraging known connections between generalization and algorithmic stability (see \cref{sec:pop-risk}).
Finally, we give preliminary lower bounds on excess empirical and population risk for TV stable algorithms for convex risk minimization.

\paragraph{Extensions.} Our results yield a number of interesting properties and extensions.
\begin{CompactItemize}
    \item   \textit{Privacy:} Even though privacy is not the goal of this work, some of our $\rho$-TV stable algorithms, those based on noisy SGD like \cref{alg:noisy-m-a-sgd}, \ref{alg:noisy-m-sgd} are $(\epsilon,\delta)$-differentially private (see \cref{defn:dp}) with $\epsilon = \rho \sqrt{\log{1/\delta}}$, for any $\delta>0$.  It is easy to see that these parameters can lie in the regime reasonable for \emph{good} privacy properties i.e. $\epsilon=O(1), \delta = \text{negl}(n)$. However, not all TV-stable algorithms, for instance \cref{alg:sub-sample-GD}, may have good privacy properties. 
    Our work therefore demonstrates interesting connections between techniques developed for differential privacy and the problem of unlearning. 
 \item \textit{Beyond Convexity}: Interestingly, 
    our unlearning techniques only require finite sum structure in the optimization problem (for exact unlearning) and  Lipschitzness (for runtime bounds). Therefore
    our unlearning algorithms yield provable unlearning for gradient-descent based methods for \emph{any} ERM or a finite sum optimization problem.
    This means that we can apply the unlearning algorithm even to non-convex problems, like training deep neural networks, and everytime the unlearning algorithm does not recompute, it still guarantees exact unlearning.
    As is typical, the accuracy in those cases is verified empirically.
    Furthermore, Lipschitzness can  be enforced by \textit{clipping} gradients - a popular heuristic in deep learning.
    In the worst-case (non-Lipschcitz) scenario, it is easy to see we will need to recompute everytime, however this is not the case in \emph{typical} situations - by which we mean when the deleted or inserted data point is not an \emph{outlier} (measured in terms of its gradient norm). 
    This also means that our unlearning efficiency is based on, and can be stated, in terms of, \textit{instance-dependent} Lipschitz parameters, rather than the worse case upper bound.
    In contrast, DP-training for non-convex models still need Lipschitzness or clipping of gradients.
    
    \item  \textit{Beyond gradient-based ERM}: 
    We also consider an approximate notion of unlearning, based on differential privacy, which has appeared in the literature~\cite{neel2020descent,guo2019certified}. With such a notion, we show a simple reduction to a DP algorithm, to handle unlearning requests, and show how to use \emph{group privacy} to trade-off accuracy and runtime. For convex ERM, this method performs competitively with existing works (see \cref{sec:approx-unlearning}).
\end{CompactItemize}    

Our proofs are simple and conceptual and so we present the key ideas in the main text and defer the proofs to the appendix.

\subsection{Related work}
The problem of exact unlearning in smooth convex ERM has not been studied before, and therefore the only baseline is re-computation (using some variant of gradient descent).
The most related are the works of \cite{ginart2019making} and \cite{neel2020descent}, which we discuss as follows.
\cite{ginart2019making} studied the problem of $k$-means clustering with exact unlearning in a streaming setting of deletion requests  - we borrow the setting (while also allowing insertions) and the notion of exact unlearning from therein.
We note that in \citep{ginart2019making}, the notion of efficiency is based on the amortized (over edits) unlearning time being at most the training time since that is a natural lower bound on the overall computational cost. We, on the other hand, do not place such a restriction and so our methods can have unlearning runtime smaller than the training time. Most importantly, the general framework here (of TV-stable methods and coupling based unlearning) captures the quantized-$k$-means algorithm of \cite{ginart2019making} as a special case (see \cref{sec:qsgd} for details).

The work of \cite{neel2020descent} focuses on unlearning in convex ERM problems, with a stream of edit requests, the same as here.
However there are two key differences. First, the notion of unlearning in \cite{neel2020descent} is approximate, based on $(\epsilon, \delta)$-differential privacy, whereas we focus on exact unlearning. Second, the unlearning runtime in \cite{neel2020descent} is deterministic, whereas that of ours is random. These are akin to Monte-Carlo vs. Las Vegas-style of guarantee discrepancy.
We refer the reader to an extended literature survey along with a detailed comparison to \cite{neel2020descent} in \cref{sec:detailed-related-work}.
We show therein that with the same unlearning time, the accuracy guarantees of \cite{neel2020descent} are better than us only in regimes where their approximate unlearning parameters and hence the notion, is very weak.

%% file: sections/probSetup.tex
\section{Problem setup and preliminaries}
\label{sec:setup}

\subsection{Streaming edit requests and exact unlearning}
\label{sec:streaming}
We describe the setup very generally.
Let $\cZ$ be the data space, $\Theta$ the output/parameter space, and  $\cM$ be the metadata/state space, which will be made clear later.
A procedure is a tuple $(\mathbf{A}(\cdot),\mathbf{U}(\cdot))$, where $\mathbf{A}: \cZ^* \rightarrow \Theta \times \cM$ is the \emph{batch} learning algorithm, and $\mathbf{U}:\Theta \times  \cM \times \cZ \rightarrow \Theta \times \cM$ is the unlearning algorithm which updates the current model (first argument) and meta data (second argument) given an edit request (third argument).
Examples of meta-data could be a compressed \emph{sketch} of the data points, or intermediate computations/state, which could be used upon edit time.
Let $\cA(\cdot)$ denote the first output of $\mathbf{A}$ i.e. $\cA(\cdot) = \mathbf{A}_1(\cdot)$. Similarly, let $\cU(\cdot)$ denote the first output of $\mathbf{U}$.
We remark that when we refer to the algorithm's output, we mean the model output and does not include the metadata.
Finally, given two sets $S$ and $S'$, we define $\Delta(S,S')$ to be the symmetric difference between these sets i.e. $\Delta(S,S') = \abs{S \backslash S'}+\abs{S' \backslash S}$.  We now define exact unlearning.
\begin{definition}[Exact unlearning]
\label{defn:exact-unlearning}
We say a procedure $(\mathbf{A},\mathbf{U})$ satisfies exact \unlearning{} if for any 
$S, S' \subset \cZ^*$ such that $\Delta(S,S')=1$,
$\mathbf{A}(S') = \mathbf{U}(\mathbf{A}(S),S'\backslash S \cup S\backslash S')$.
For randomized procedures, we want that for any measurable event $\mathcal{E} \subseteq \Theta \times \cM$, we have $\P{\mathbf{A}(S') \in \mathcal{E}} = \P{\mathbf{U}(\mathbf{A}(S),S'\backslash S \cup S\backslash S') \in \mathcal{E}}$
\end{definition}

\begin{remark}
\begin{enumerate}
    \item A relaxation of the above definition is to maintain that only the output and not the meta-data satisfy the above condition i.e. $\cA(S') = \cU(\mathbf{A}(S),S'\backslash S \cup S\backslash S',S)$. However, we will work with the stronger notion. This, with a slight difference, is referred to as \emph{perfect} unlearning in \cite{neel2020descent}.
    \item  Even though the above definition is for one edit request, it can be generalized for a stream of $k$ edit requests, by having that this condition holds inductively for every point in the stream. 
\end{enumerate}
\end{remark}

Let $S=S^0=\bc{\z_1,\z_2,\ldots, \z_n}, \z_i \in \cZ$ be the initial dataset. We observe $k$ edit requests, each being either an insertion or deletion request. We use $S^i$ to denote the set of data points available at time $i$ in the stream. For notational simplicity, as in \cite{neel2020descent}, we assume that at any point in the stream, the number of available data points is at least $n/2$ and at most $2n$. 

\subsection{Convex risk minimization}
We recall some basics from convex optimization.
Let $\cW \subset \bbR^d$ be a closed convex set such that $\text{diameter}(\cW)$ $\leq D$ where the diameter is measured in Euclidean distance. 
Let $\cZ$ be the instance space and let $f:\cW \times \cZ \rightarrow \bbR$ be an $L$-Lipschitz convex function in its first argument. 
For the constraint set $\cW$, given a point $\w$, a projection function $\cP:\bbR^d \rightarrow \bbR^d$ returns $\cP(\w) \in \arg \min_{\v \in \cW}\norm{\w-\v}$.  
The function $f$ is $L$-smooth in its first argument if $\norm{\nabla_\w f(\w_1,\z) - \nabla_\w f(\w_2,\z)}\leq L\norm{\w_1-\w_2} \ \forall \w_1,\w_2 \in \cW, \z \in \cZ$.
We will drop the subscript $\w$ in $\nabla$ from here on.
In this work, we will be concerned with smooth Lipschitz convex functions.

\paragraph{Empirical Risk Minimization (ERM).} Given data points $S=\bc{\z_1,\z_2,\ldots, \z_n}$, we look at the following class of problems, known as empirical risk minimization (ERM).
\begin{align}
\label{eqn:primal}
    \min_{\w \in \cW}\bc{\hat F_S(\w):=\frac{1}{n} \sum_{j=1}^n f(\w,\z_i)}
\end{align}

Let $\cA(S)$ be the output of algorithm $\cA$ on dataset $S$. We will give guarantees on \emph{expected excess empirical risk}, which is $\E{\hat F (\cA(S))} - F(\w_S^*)$, where $\w^*_S$ is the minimizer: $\w^*_S \in \arg \min_{\w \in \cW} \hat F_S(\w)$ and the expectation is taken with respect to the randomness in algorithm $\cA$. Finally, we note that the above class of problems is large enough to be applicable even beyond machine learning - for example: in statistics, estimation problems which reduce to optimization problems of the above form are called $M$-estimation. 

We discuss the \textit{related} notion of population risk and the problem of risk minimization along with our results in \cref{sec:pop-risk}.

\subsection{Total variation stability and maximal coupling}
We first state the definition of Total Variation (TV) distance between two distributions $P$ and $Q$. 
\begin{align*}
    \text{TV}(P,Q) = \sup_{\text{measurable sets }R} \abs{P(R)-Q(R)} = \frac{1}{2}\norm{\phi_P - \phi_Q}_1
\end{align*}
where the second equality holds if both distributions have probability densities with respect to a base measure which are denoted by $\phi_P$ an $\phi_Q$ respectively. 
We now define total variation stability (TV-stability), which is the notion of algorithmic stability we will use.
\begin{definition}[$\rho$-TV-stability]
 An algorithm $\cA$ is said to be $\rho$-TV-stable if 
\begin{align*}
    \sup_{S, S':\Delta(S,S')= 1} \text{TV}(\cA(S),\cA(S'))\leq \rho
\end{align*}
\end{definition}

\begin{remark}
\label{remark:tv_upto_k}
\begin{enumerate}
\item The above definition of TV stability considers the marginals of output and does not include the metadata.
\item  Suppose $S$ is a dataset of $n$ points, and $S'$ is a dataset of $n+k_2$ points such that $\abs{S \backslash S'} = k_1$. Then, if algorithm $\cA$ is $\rho$-TV stable, then by triangle inequality of TV and repeated applications of the above definition, we have that $\text{TV}(\cA(S),\cA(S'))\leq (2k_1+k_2)\rho$
\end{enumerate}
\end{remark}
We discuss the maximal coupling characterization of total variation distance, which is a key ingredient in the design of our unlearning algorithms.

 \paragraph{Coupling and total variation distance:} 
 A coupling between two probability distributions $P$ and $Q$ over a common measurable space $(\cX,\cB)$, where $\cB$ denotes the  sigma-algebra on $\cX$, is a distribution $\pi \in \bbP(\cX \times \cX, \cB \otimes \cB)$ such that the marginals along the projections $(x,y)\rightarrow x$ and $(x,y)\rightarrow y$ are $P$ and $Q$ respectively. Let $\Pi(P,Q)$ denotes the set of couplings between $P$ and $Q$. The following describes the maximal coupling characterization of total variation distance.
\begin{enumerate}
    \item For any coupling $\pi \in \Pi(P,Q)$, if the random variable $(p,q)\sim \pi$, then $\text{TV}(P,Q)\leq \P{p\neq q}$.
    \item There exists a ``maximal" coupling $\pi^*$ such that if $(p,q)\sim \pi^*$, then $\text{TV}(P,Q)= \P{p\neq q}$
\end{enumerate}
The above establishes that $\text{TV}(P,Q)= \inf_{\pi \in \Pi(P,Q)}\mathbb{P}_{(p,q)\sim \pi}\left[p\neq q\right]$.

As a final remark, in this work, we routinely deal with distances and divergence between probability distributions. In some cases, we abuse notation and write a divergence between random variables instead of probability distributions - these should be interpreted as the law of the random variables.

%% file: sections/mainresults.tex
\section{Main results}
\label{sec:main-results}

We state our main result on designing learning and unlearning algorithms in a stream of edit requests.

\begin{theorem}[Main Theorem]
\label{thm:main-result}

For any $\frac{1}{n} \leq \rho < \infty$, there exist a learning and a corresponding unlearning algorithm such that for any $f(\cdot, \z)$, which is $L$-smooth and $G$-Lipschitz convex function $\forall \ \z$, and a stream of edit requests,
\begin{enumerate}
    \item Satisfies exact unlearning at every time point in the stream of edit requests.
    \item At time $i$ in the stream, outputs $\hat \w_{S^i}$ with excess empirical risk bounded as,
    \begin{align*}
    \E{\hat F_{S^i}(\hat \w_{S^i}) - \hat F_{S^i}(\w_{S^i}^*)} \lesssim
    \min\bc{\frac{GD}{\sqrt{\rho n}},\br{\frac{L^{1/4}GD^{3/2} \sqrt{d}}{(\rho n)}}^{4/5}}
\end{align*}
\item For $k$ edit requests, the  expected unlearning runtime is $O(\max\bc{\min\bc{\rho,1} k \cdot \text{Training time},k})$. 
\end{enumerate}
\end{theorem}

We make some remarks about the result.

\paragraph{Training time:} 
Informally, what the above theorem says is that the algorithms satisfy exact unlearning and are accurate while only recomputing a $\rho$ fraction of times - this is indeed the nature of our algorithms.
Therefore, "Training time" here refers to the runtime of the learning algorithm.
If we measure training time in terms of number of gradient (oracle) computations, as is typical in convex optimization, then for the above accuracy, our algorithm has optimal oracle complexity in \textit{most} regimes (see details in \cref{sec:learning-runtime}).

\paragraph{Role of $\rho$:}
The external parameter $\rho$ controls the trade-off between accuracy and unlearning efficiency. In the extreme case where we don't care about unlearning efficiency and are fine with paying retraining computation for every edit request, then we can set $\rho>1$ as large as we want to get, as expected, arbitrary small excess empirical risk. However, the interesting case is when we set $\rho<1$: herein, we get an \textit{improved} (see below) unlearning time and yet a non-trivial accuracy, upto $\rho\gtrsim \frac{1}{n}$.

\paragraph{Strict improvement:} The above result may seem like a trade-off, but, as we argue below, is a \textbf{strict} improvement over the baseline of retraining after every edit request (which is the only other known method for exact unlearning for this problem). Let the target excess empirical risk be $\alpha > \alpha_0 =\min\bc{\frac{GD}{\sqrt{ n}},\br{\frac{L^{1/4}GD^{3/2} \sqrt{d}}{n}}^{4/5}}$. For any such $\alpha$, there exists a  $\rho<1$, such that our algorithms have $\rho k \cdot \text{Training time}(\alpha)$ expected unlearning time, which is smaller than $k\cdot\text{Training time}(\alpha)$ - the cost of retraining after every edit request. 
Furthermore, as remarked above, since our training time is optimal in number of gradient computations (for the said accuracy), the aforementioned improvement holds for re-computation with \textit{any} first-order optimization algorithm.
A small caveat is that we are comparing our \textit{expected} unlearning time with deterministic runtime of retraining. To summarize, with this caveat, we have a strict improvement in the \textit{low} accuracy regime, whereas in the high accuracy regime: $\alpha< \alpha_0$, our unlearning algorithms are as good as trivial re-computation.
However, this low accuracy regime is often the target in machine learning.
To elaborate, the goal is to minimize the population risk rather than empirical risk, and it is well known that this statistical nature of the problem results in an information-theoretic lower bound of $\frac{1}{\sqrt{n}}$ on excess population risk. We show in \cref{sec:pop-risk} that our algorithm guarantees an excess population risk of $\frac{1}{\sqrt{n}}+\alpha$, and so a very small $\alpha$ only becomes a lower order term in excess population risk.

\paragraph{Algorithms:}The first upper bound on accuracy in \cref{thm:main-result} is obtained by standard SGD, which, in each iteration samples a fraction of datapoints, called mini-batch, to compute the gradient, and performs the descent step - we call this \emph{sub-sample-GD}. 
The second upper bound is obtained using noisy accelerated mini-batch-SGD (\emph{noisy-m-A-SGD}), which is also used for differentially private ERM.
Our unlearning algorithm for \emph{sub-sample-GD} is rather straightforward, and most of the work is design of unlearning  algorithm for \emph{noisy-m-A-SGD}, which is based of efficient coupling of Markov chains corresponding to the learning algorithm. We describe the algorithms in detail in \cref{sec:algorithms}.

\paragraph{Sub-optimality within the TV stability framework:} If we consider  $L,G,D=O(1)$, and a simple model of computation wherein we pay a unit computation when we recompute, otherwise not, then the unlearning problem is equivalent to design of  $TV$-stable algorithms, and a corresponding (maximal) coupling  (see \cref{sec:optimal-transport} for more details). Our coupling construction for unlearning in \emph{noisy-m-A-SGD}, though efficient, is not maximal - this gap shows up in the accuracy bound (second term), which is $\br{\frac{\sqrt{d}}{\rho n }}^{1/5}$ worse than what we would have obtained via a maximal coupling i.e $\frac{\sqrt{d}}{\rho n}$.
We also note that in case we don't use acceleration, but rather vanilla noisy mini-batch SGD, and the "same" coupling construction for unlearning, then we obtain a worse accuracy bound of $\br{\frac{\sqrt{d}}{\rho n }}^{2/3}$ (see \cref{sec:noisy-m-sgd} for details). 
Finally, apart from closing the gap with the maximal coupling,  another potential improvement is by giving $\rho$-TV stable algorithms with \textit{better} accuracy. We discuss such upper and lower bounds as follows.

As pointed out, intermediate to the result in \cref{thm:main-result} is the design and analysis of $TV$-stable algorithms for smooth convex ERM. Our main result on upper bounds on accuracy of such algorithms is the following.

\begin{theorem}[Upper bound]
\label{thm:main-upper-bound}
For any $0<\rho<\infty$, there exists an algorithm which is $\min\bc{\rho,1}$-TV stable, such that for any $f(\cdot, \z)$ which is $L$-smooth and $G$-Lipschitz convex function $\forall \ \z$, and any dataset $S$ of $n$ points, outputs $\hat \w_S$ which satisfies the following.

\begin{align*}
    \E{\hat F_S(\hat \w_S) - \hat F_S(\w^*_S)} \lesssim GD\min\bc{\frac{1}{\sqrt{\rho n}}, \frac{\sqrt{d}}{\rho n}}
\end{align*}
\end{theorem}

We show that the condition $\rho \geq \frac{1}{n}$ in \cref{thm:main-upper-bound} is fundamental for any non-trivial accuracy, as evidenced by our lower bounds, with a matching dependence on $\rho$. Furthermore, we omit the regime $\rho\geq 1$ in our lower bound since it puts no constraint on the algorithm.

\begin{theorem}[Lower bound]
\label{thm:main-lower-bound}
For any $\rho$-TV-stable algorithm $\cA$, there exists a $G$-Lipschitz convex function $f$ and a dataset $S$ of $n$ points such the expected excess empirical risk is lower bounded as:
\begin{enumerate}
    \item   For any $0<\rho <1$, and any dimension $d$, $\E{\hat F_S(\cA(S)) - \hat F_{S}(\w_{S}^*)} \gtrsim  GD\min\bc{1,\frac{1}{\rho n}}$
    \item Assuming that $\cA(S)$ has a probability density function upper bounded by $K \leq O(2^d)$, then for $n>72,\frac{1}{n}\leq \rho\leq \frac{1}{4}$ and large enough $d$,  $\E{\hat F_S(\cA(S)) - \hat F_{S}(\w_{S}^*)} \gtrsim GD\min{\bc{1,\frac{1}{\sqrt{\rho n}}}}$
\end{enumerate}
\end{theorem}

In each of the lower bounds, the term $GD$ is trivial as it is attained if an algorithm outputs a constant regardless of the problem instance.
The first lower bound holds for \emph{all} problem instances without any assumptions on the relationship between the problem parameters $d$, $n$ and $\rho$.
Note that if we assume that the upper bound given by \cref{thm:main-upper-bound} were tight, in that case we would expect to derive a lower bound of 
$\frac{\sqrt{d}}{\rho n}$ whenever $\frac{\sqrt{d}}{\rho n} \leq \frac{1}{\sqrt{\rho n}} \iff d \leq \frac{1}{\rho n}$ - we would therefore need to shrink the class of problem instances explicitly.
Unfortunately, our techniques currently do not show improvement with this restriction.
The second result is obtained by a \textit{direct} analysis, where the key ingredient is the fact that normalized volume of spherical cap of a hypersphere goes to $0$ as $d\rightarrow \infty$, for a fixed width of the cap.
The condition that the probability distribution $\cA(S)$ has bounded density prevents it to have discrete atoms - this is not desirable especially since our (upper bound) algorithm \emph{sub-sample-SGD} outputs a mixture of discrete distributions, and therefore does not lie in this class.
Please see \cref{sec:lower-bounds} for derivations of the lower bounds.

%% file: sections/mainideas.tex
\section{Main ideas}

In this section, we discuss the key ideas to our approach. The first is identifying a notion of stability. For this, we connect the problem of unlearning to optimal transport, and specifying a simple model of computation, the notion of total variation stability arises naturally. The second is the design of TV stable algorithms for convex ERM. 
One of the algorithms we propose is an existing differential private solution, which we show to be TV stable as well. We also discuss certain important differences between our setup and that of differential privacy. Finally, the bulk of the work, is the design and analysis of efficient unlearning algorithms. We show that this problem can be reduced to efficiently constructing couplings between Markov chains, and we give such a construction using rejection sampling and reflection mappings. We now discuss these one by one.
 
\subsection{Total variation stability from optimal transport}
\label{sec:optimal-transport}

In this section, we give a didactic treatment of our approach to motivate the notion of total variation stability.
Consider neighbouring datasets $S$ and $S'$ and let $P = \cA(S)$ and $Q=\cA(S')$ for some randomized algorithm $\cA$. The algorithm first computes on $S$, and then observes edit requests which generate $S'$ as the current dataset. To satisfy exact \unlearning, we need a procedure which \emph{moves} $P$ to $Q$. 
This is akin to the well-studied optimal transport problem \cite{villani2008optimal}, which briefly explain below.
Given probability distributions $P$ and $Q$ over measurable space $\cX$,  and a cost function $c: \cX \times \cX \rightarrow \bbR$, the goal is to transport from $P$ to $Q$ using the minimum cost. 
Formally, let  $\Pi(P, Q)$ denote the set of couplings (or transport plans) of $P$ and $Q$; the modern (Kantorovich's) formulation asks for a transport plan $\pi$ which minimizes the expected cost: $ \min_{\pi \in \Pi(P,Q)} \mathbb{E}_{(x,y)\sim \pi} c(x,y)$.

\paragraph{A model of computation:} 
Note that there is of course the trivial coupling in which we generate independent samples from $P$ and $Q$ - this corresponds to re-computation,  which, as argued, is not an efficient method in general.
Instead, we should correlate $P$ and $Q$ so that transporting from $P$ to $Q$ can reuse the randomness (computation) used for $P$.
For this, we use the cost function in the optimal transport problem as a surrogate of modelling computation.
In the optimal transport problem, the cost is typically a \textit{distance} on the space, whereas we are concerned with computational cost. So is there a distance function which corresponds to computational cost? Note that
the sequential nature of the problem already gives us samples generated from $P$, so a natural question is, can we use this to transport to $Q$?
We can set the cost function as $c(x,y)= \begin{cases} 1 & \text{ if } x\neq y \\ 0 & \text{ otherwise}\end{cases}$.
This corresponds to an oracle which charges a unit computation if we use $y$ which is different from $x$, which can correspond to a recomputation. Under this simple model of computation, the optimal expected computational cost becomes exactly equal to the total variation distance between $P$ and $Q$:
$\inf_{\pi \in \Pi(P,Q)} \mathbb{1}\bc{x \neq y}$ - the maximal coupling characterization of total variation distance.

\paragraph{TV stability:}
The above establishes that if we want to transport $P$ to $Q$ using minimum computation cost, the expected computation cost cannot be smaller than the total variation distance between $P$ and $Q$. 
Intuitively, this means that is least $1-\text{TV}(P,Q)$ fraction of samples are \emph{representative} for both $P$ and $Q$.
From the sequential nature of our problem, when we generate $P$ - the output on dataset $S$, we don't know what $Q$ would be, since we don't know the incoming edit request. Hence a reasonable property to have in the algorithm is that its output is close in total variation distance \emph{uniformly} over all possible $Q$'s. This motivates our definition of total variation stability. 

\paragraph{Optimal transport vs unlearning:} 
Unlike the optimal transport problem wherein we are given $P$ and $Q$, and the task is to find a coupling, in our setup, we have to find an algorithm generating $P$ and $Q$ as well as the coupling.
Moreover, for a fixed $\rho$, there may be many algorithms which are $\rho$-TV stable. 
The goal therefore, is to find among these algorithms, the one with the maximum accuracy for the (convex ERM) problem, and for which we can design a corresponding \textit{efficient} unlearning algorithm.

\subsection{TV-stable learning algorithms and differential privacy}
In this section, we discuss the ideas underlying the design of TV-stable learning algorithms. We first give the definition of differential privacy (DP), which will be a key tool.
Differential privacy is a notion of data privacy, introduced in \cite{dwork2006calibrating}, defined as follows.
\begin{definition}[Differential privacy (DP)]
\label{defn:dp}
    An algorithm $\mathcal{A}$ satisfies $(\epsilon,\delta)$-differential privacy if for any two neighboring datasets $S$ and $S'$, differing in one sample, for any measurable event $\cE \in \text{Range}(\cA)$,
\begin{align*}
   \mathbb{P}(\mathcal{A}(S)\in \cE) \leq e^\epsilon \mathbb{P}(\mathcal{A}(S')\in \cE)+\delta
\end{align*}
\end{definition}
Intuitively, a differentially private algorithm promises that the output distributions are close in a specific sense: the likelihood ratios for all events for two neighbouring datasets is uniformly close to $e^{\pm \epsilon}$, upto a failure probability $\delta$. 
Note that we have identified that we want our outputs to be $\rho$-TV stable. A natural question is whether we can relate the $(\epsilon, \delta)$-DP notion to $\rho$-TV-stability. An easy to see direction is that any $\rho$-TV stable method is (at least) $(0,\rho)$-DP. 
Similarly, for the other direction, under additional assumptions, such relations can be derived. The important part is that certain widely used DP methods are TV stable as well. The primary example, which we will use in this work, is Gaussian mechanism. It is known that adding Gaussian noise of variance $\frac{\sqrt{\log{1/\delta}}}{\epsilon}$ to a $1$-sensitive function, provides $(\epsilon,\delta)$-DP \cite{dwork2014algorithmic}. It can be shown that the same method also provides $\rho$-TV stability, with $\rho = \frac{\epsilon}{\sqrt{\log{1/\delta}}}$. 

\paragraph{TV-stable algorithm:}
For the problem of TV stable convex empirical risk minimization, we propose two algorithms: \emph{sub-sample-GD} and \emph{noisy-m-A-SGD}. We show that the expected excess empirical risk of \emph{noisy-m-A-SGD} is better than that of \emph{sub-sample-GD}, in regimes of small dimension.
The algorithm~\emph{noisy-m-A-SGD} is essentially \emph{noisy-m-SGD} algorithm, which appeared in  \cite{bassily2014private} for DP convex ERM, with an additional Nesterov's acceleration on top. 
For ease of presentation, we discuss using \emph{noisy-m-SGD} (no acceleration), and in 
paragraph titled "Fast algorithms and maximal coupling" in the next section, it would become clear why adding acceleration helps us.
In \emph{noisy-m-SGD},
at iteration $j$, we sample a mini-batch $b_j$ uniformly randomly, use it to compute the gradient on the previous iterate $\w_{j}$ denoted as $\nabla \hat F_S(\w_{j},\z_{b_{j}})$ and update as follows:
$$\w_{j+1} = \w_j - \eta\br{\nabla \hat F_S(\w_{j},\z_{b_{j}}) + \theta_t}$$
where $\theta_j \sim \cN(0,\sigma^2\bbI_d)$ and $\sigma$ is set appropriately. 
We ignore the projection step in the current discussion.
This procedure can be viewed as sampling from a Markov chain depicted in Figure \ref{fig:markov_chain_noisy_m_sgd}.

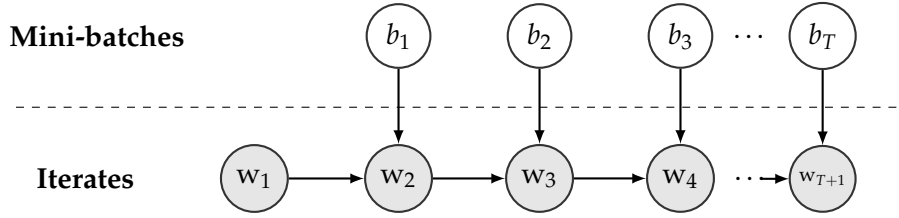
\begin{figure}[!htbp]
    \centering
\begin{tikzpicture}
  \node[box,draw=white!100] (Latent) {\textbf{Mini-batches}};
   \node (L0) [right=of Latent] {};
  \node[main] (L1) [right=of L0] {$b_1$};
  \node[main] (L2) [right=of L1] {$b_2$};
  \node[main] (L3) [right=of L2] {$b_3$};
  \node[main] (Lt) [right=of L3] {$b_T$};
   \node at (8.7, -1.9)   (c1) {};
  \node[main,fill=black!10] (O1) [below=of L1] {$\w_2$};
  \node[main,fill=black!10] (O0) [left=of O1] {$\w_1$};
  \node[main,fill=black!10] (O2) [below=of L2] {$\w_3$};
  \node[main,fill=black!10] (O3) [below=of L3] {$\w_4$};
  \node[main,fill=black!10,scale=0.69] (Ot) [below=of Lt] {$\w_{T+1}$};
  \node[box,draw=white!100,left=of O0] (Observed) {\textbf{Iterates}};
  \path (L3) -- node[auto=false]{\ldots} (Lt);
  \path (O0) edge [connect] (O1)
        (O1) edge [connect] (O2)
        (O2) edge [connect] (O3)
        (c1)  edge [connect](Ot)
        (O3) -- node[auto=false]{\ldots} (Ot);
  \path (L1) edge [connect] (O1);
  \path (L2) edge [connect] (O2);
  \path (L3) edge [connect] (O3);
  \path (Lt) edge [connect] (Ot);
  
  \draw [dashed, shorten >=-1cm, shorten <=-1cm]
      ($(Latent)!0.5!(Observed)$) coordinate (a) -- ($(Lt)!(a)!(Ot)$);
\end{tikzpicture}
   \caption{Markov chain for the \emph{noisy-m-SGD} algorithm}
    \label{fig:markov_chain_noisy_m_sgd}
\end{figure}

\paragraph{DP convex ERM and unlearning:}
We first discuss an important distinction in the differential privacy and our unlearning setup.
In the DP setup, we have a curator which possesses the dataset, and an analyst/adversary, against which the curator want to provide privacy. The analyst queries the dataset, and the curator provides DP answers to the queries. The curator can also reveal additional information pertaining to the algorithmic details, however, it is beneficial to the curator to only release limited information. In particular, the curator can chose to keep certain states of the algorithm secret.
This could be done in the case when only the marginals of the output satisfy a strong DP-guarantee. 
So, if the curator were to release the secret state as well, the adversary can correlate information and then the privacy level, which is now measured using the joint distribution of output and state, degrades.
In the \emph{noisy-m-SGD} algorithm for example, the output typically is the average or final iterate whereas the rest of iterates and mini-batch indices $b_j$'s are the secret state.

In the unlearning setup, there is no adversary per se, or in the idealized application, the curator is the adversary and the dataset owners wants it to have as little control as possible. 
It is therefore natural to
demand that the probability distribution of the entire state maintained by the algorithm, and not just the output be exactly identical after performing the unlearning operation. 
This, with slight differences, is referred to as perfect unlearning in \cite{neel2020descent}, and what our algorithms satisfy.
We have argued that designing TV stable algorithms is a good start, and for a moment suppose that the TV stability is same as DP.
Then should we measure TV stability between the joint distributions over the entire state?
This would limit the application of DP techniques in which keeping additional state hidden has stronger privacy property.
In that case, TV stability parameter, and hence the computational cost of unlearning would be large.
Interestingly, even though the previous work in differentially private convex ERM, for example \cite{bassily2014private}, argues that the released iterate (average/final iterate) is differentially private,  the analysis is typically carried out by first arguing, via a composition step, that \textit{all} iterates together are differentially private. This means that all iterates can be released without any additional cost of privacy.
This innocuous property arguably provides no benefit for privacy, but turns out to be extremely beneficial to us in unlearning. However, even though the all the iterates can be released, the mini-batches still need to kept secret.
We handle this in the unlearning algorithm using an estimation step - see paragraph titled ``Estimation of marginals" in \cref{sec:main-idea-unlearning}.

\subsection{Unlearning via (un)couplings}
\label{sec:main-idea-unlearning}
The final, though the most important piece, is the design of unlearning algorithms.
Recall that $S$ is the initial dataset,  $S'$ is the dataset after one edit request, and we want to design a transport from $P=\cA(S)$ to $Q=\cA(S')$, which means that we need to construct a coupling of $P$ and $Q$.
Broadly, there are two challenges: the first is the data access restriction - when generating a sample from $P$, we don't know what $Q$ would be, therefore, the coupling cannot be based on efficiently sampling from a joint distribution directly, but is limited to work with samples generated from $P$.
The other is that construction of the coupling should be computationally more efficient than drawing independent samples from $P$ and $Q$, which essentially amounts to our baseline of re-computation.

\paragraph{An efficient general approach:} 
We first setup some terminology - the diagonal of a coupling $\pi$ of two probability distributions, is the set $\bc{(p,q): p = q}$ where $(p,q)\sim \pi$, and similarly, the non-diagonal is the set $\bc{(p,q): p \neq q}$, $(p,q)\sim \pi$. We have that the measure of the non-diagonal, under a maximal coupling $\pi^*$, is $\mathbb{P}_{(p,q)\sim \pi^*} \mathbb{1}\bc{(p,q): p \neq q} = \text{TV}(P,Q)$.
This implies that when using $\rho$-TV stable algorithms, the probability measure of the diagonal under a maximal coupling, is \emph{large} - at least $1-\rho$.
At a high-level, our unlearning approach comprises of two stages: \emph{verification} and \adjust{}. 
We  first \emph{verify} whether our output on dataset $S$ (i.e. sample from $P$) \textit{falls} on the diagonal of \emph{any} maximal coupling of $P$ and $Q$ or not - if that is indeed the case, then the same sample for $Q$ suffices.
For computational efficiency, we require that verification be computationally much cheaper then recomputing (smaller that $\rho \cdot \text{recompute cost}$).
If the verification fails, we sample from the non-diagonal of any maximal coupling $P$ and $Q$, so that we have a valid transport.
As the name suggest, the computational cost of \adjust{} that we will shoot for is to be of the same order as (full) recompute. 
If we are able to design such a method, then we will show that for $k$ edit requests, the expected computational cost for unlearning is $k \cdot \text{verification cost} + k\rho\cdot  \text{recompute cost} \approx  k\rho \cdot \text{recompute cost}$.

The design of unlearning algorithm is dependent on the corresponding learning algorithm.
Since we proposed two learning algorithms, each has a corresponding unlearning algorithm. Herein, we will only discuss the more challenging case, which is for the \emph{noisy-m-SGD} (no acceleration) algorithm. 

\paragraph{Coupling of Markov chains:}
Our approach for unlearning is to construct a coupling of the optimization trajectories on neighbouring datasets.
We have discussed that the iterates from \emph{noisy-m-SGD} can be seen as generated from a Markov chain, depicted in Figure \ref{fig:markov_chain_noisy_m_sgd}.
Hence, for two neighbouring datasets, the iterates are sampled from two \emph{different} Markov chains $P$ and $Q$. 
Moreover, by design, we know that these Markov chains are $\rho$-TV close - we measure the total variation distance between joint distribution of marginals of iterates i.e. $\text{TV}\br{\bc{\w_j^P}_{j=1}^{T+1},\bc{\w_j^Q}_{j=1}^{T+1}} \leq \rho$.
The task is now to maximally couple these two Markov Chains. 
We remark that in the Markov chain literature, maximal coupling of Markov chains does not refer to the above but rather the setting wherein we have one Markov chain, but started at two different states, and the goal is to design a coupling such that their sampled states become and remain equal as soon as possible. In contrast, our notion of coupling of two Markov chains has also been recently studied by \cite{vollering2016maximal} and \cite{ernst2019mexit}, wherein they refer to this problem as design of \emph{uncoupling} or \emph{maximal agreement/exit couplings}.

\paragraph{Overview of unlearning algorithm:}
The learning algorithm saves all states depicted in Figure \ref{fig:markov_chain_noisy_m_sgd}. In the unlearning algorithm, we proceed sequentially: in iteration $j$, we first couple the mini-batches which amounts to replacing the deleted point by a uniformly random point, or inserting the new point in \emph{some} mini-batches: let the coupled mini-batches be $\bc{b^P_j}_{j=1}^T$ and $\bc{b^Q_j}_{j=1}^T$. 
We then compute an \textit{estimate} of marginal densities of $\w_j$ under $P$ and $Q$, via conditional densities under coupled mini-batches i.e. we compute $\phi_P(\w_j|b_j^Q)$ and $\phi_Q(\w_j|b_j^P)$ - note that these are just Gaussian densities evaluated at $\w_j,b_j^P$ and $\w_j,b_j^Q$ respectively, where  $\w_j$ is sample from $P$.
We then do a rejection sampling step wherin we draw a uniform random variable $u \sim \text{Unif}(0,1)$, and then check if $u \leq \frac{\phi_Q(\w_j|b_j^Q)}{\phi_P(\w_j|b_j^P)}$.
If the step succeeds, we accept $\w_j$ as a sample from $Q$ and move to the next iteration and repeat. If any of the rejection sampling step fails, say at step $t$, we generate $\w_{t+1}^Q$ by \emph{reflecting} $\w_{t+1}^P$ about the mid-point of means of the two Gaussians at step $t$ for $P$ and $Q$. After this reflection step, we abandon the rest of iterates from $P$ and generate the new iterates from $Q$ by continue retraining on dataset $S'$. 
This procedure is described as \cref{alg:unlearn-noisy-m-a-sgd} (please see \cref{sec:unlearn-noisy-m-a-sgd} for more details).
In the above, the rejection sampling steps comprise the verification stage, and if any of the rejection sampling fails, we move to re-computation. 
The reason why verification can be done efficiently here is due to the finite sum structure of the ERM problem. To elaborate, at any iteration, 
to compute the conditional density $\phi_Q(\w_j\vert b_j^Q)$, we need to compute the gradient with the new dataset $S'$ - this, using the gradient of the old dataset only requires subtracting the gradient at the deleted point, so $O_d(1)$ runtime as opposed to $O_d(m)$, if we were to compute from scratch, where $m$ is the mini-batch size. Moreover, throughout verification, this computation is done only for iterations which used the deleted point which are roughly $\frac{Tm}{n}$ iterations. Hence the total runtime of verification is $O_d\left(\frac{Tm}{n}\right)$ as opposed to $O_d\left(Tm\right)$ for re-computation.
Finally, if the probability of recompute is $\rho$, then expected unlearning time for $k$ edits  is $\approx \rho kTm + \frac{kTm}{n}$ - note that the second (verification time) is a lower-order term as long as $\rho\gtrsim \frac{1}{n}$ - furthermore, this $\rho\gtrsim \frac{1}{n}$ is the best possible within the TV stability framework for non-trivial accuracy as evidenced by our lower bounds (see \cref{thm:main-lower-bound}).
Please see \cref{sec:unlearning-runtime} for more details on runtime and efficient implementation using suitable data structures.

\paragraph{Fast algorithms and maximal coupling:}
The above procedure generates a coupling but not a maximal coupling - the measure of the diagonal under the coupling, and hence the probability to recompute, is $\sqrt{T}$ worse then the optimal, where $T$ is the number of iterations run of \emph{noisy-m-SGD}. This gives us that the faster the algorithm (in terms of iteration complexity) is, the smaller the probability to recompute, when using our coupling construction. This motivates why we use \textit{accelerated} mini-batch SGD, since it has a quadratically faster iteration complexity than vanilla mini-batch SGD. In \cref{sec:learning-runtime}, we also remark that using even faster algorithms like Katyusha \cite{allen2017katyusha} does not yield further improvements. Finally, the design of maximal coupling would (likely) be done via one step rejection sampling, instead of doing it iteratively. However, if the rejection sampling fails, sampling from the non-diagonal efficiently is tricky. We leave the question of obtaining a maximal coupling for future work.

\paragraph{Estimation of marginals:}
We remarked that we want to create maximal coupling of marginals of the output, and therefore measure TV distance between marginals, rather than the entire algorithmic state. 
Consider one, say $j^{\text{th}}$, iteration of \emph{noisy-m-SGD},
then $b_j$ is additional state, and we measure TV between marginals of $\w_j$ and $\w_j'$.
The distribution of $\w_j$ is such that, for any event $E$ in range of $\w_j$, $P(\w_j \in E) = \mathbb{E}_b P(\w_j \in E | b_j = b)$. 
To construct a coupling between the marginals via rejection sampling, we need to evaluate the ratio of marginal densities: $\frac{\phi_Q(\w_j)}{\phi_P(\w_j}$. However, the marginal is a mixture distributions with large (exponential in $m$ (mini-batch size)) number of components, and therefore even evaluating the marginal density is infeasible. One solution is to just consider $m=n$ i.e. full gradient descent, and then there is no additional state.
However, this makes the training runtime worse, and that means that we would be using a \textit{slower} learning algorithm than what we would have used if we were to simply recompute to unlearn.
Hence, to tackle this, as described in the previous paragraph, we evaluate the ratio of \textit{conditional} probability densities, where the conditioning is on the coupled mini-batch indices ($b_j^P$ and $b_j^Q$) i.e. $\frac{\phi_Q(\w_j|b_j^Q)}{\phi_P(\w_j|b_j^P)}$.
This corresponds to using unbiased estimates of the marginals densities. 
It is easy to verify, using convexity of the pointwise supremum for instance, that $\text{TV}((\w^P_j,b^P_j),(\w^Q_j,b^Q_j)) \geq \mathbb{E}_{(b_j^P,b_j^Q)} \text{TV}(\w_j^P|b_j^P, \w_j^Q|b_j^Q) \geq  \text{TV}(\w_j^P,\w_j^Q)$.
However, in general, this might still not be ideal since we are estimating with just one sample from the mixture and hence the estimation error would be large. However, we will show that since we are anyway not able to construct maximal couplings, we don't pay extra with this coarse estimate.

%% file: sections/algorithms.tex
\section{Algorithms}
\label{sec:algorithms}

In this section, we present the algorithms for learning and unlearning. In our algorithms, we use functions  ``save" and ``load", which vaguely means saving and loading the variables to and from memory respectively. In \cref{sec:runtime}, we explain what data structures to use for computational efficiency.
The proofs of results in this section are deferred to \cref{sec:proofs-learning}.

\subsection{TV-stable learning algorithms}
\label{sec:alg-learning}
\subsubsection{\emph{sub-sample-GD}}
\label{sec:alg-learning-sub-sample-GD}

The first algorithm, which is superior in high dimensions, called \emph{sub-sample-GD}, is just vanilla mini-batch SGD wherein at each iteration, a mini-batch of size $m$ is sub-sampled uniformly randomly. Furthermore, we save all the mini-batch indices, gradients and iterates to memory. 
We will see that the unlearning algorithm presented (\cref{alg:unlearn-subsample-gd}) uses all the saved iterates. However this is done only for ease of presentation - in \cref{sec:space-complexity}, we discuss a simple efficient implementation (of the unlearning algorithm), which doesn't need \emph{any} iterate, yet has the same unlearning time complexity.

\begin{algorithm}[h]
\caption{\emph{sub-sample-GD}$(\w_{t_0}, t_0)$}
\label{alg:sub-sample-GD}
\begin{algorithmic}[1]
\REQUIRE{Initial model $\w_{t_0}$, data points $\bc{\z_1,\ldots, \z_n}, T, m, \eta$}
\FOR{$t=t_0,t_0+1\ldots, T$}
\STATE Sample mini-batch $b_t$ of size $m$ uniformly randomly
\STATE $\g_t = \frac{1}{m}\sum_{j \in b_t} \nabla f(\w_t,\z_j)$
\STATE $\w_{t+1} = \cP\br{\w_t - \eta \g_t}$
\STATE Save($b_t,\w_t,\g_t$)
\ENDFOR
\ENSURE{$\hat \w_S = \frac{1}{T}\sum_{t=1}^{T+1}\w_t$}
\end{algorithmic}
\end{algorithm}

We now give guarantees on excess empirical risk for \emph{sub-sample-GD}.

\begin{proposition}
\label{prop:upper_bound_subsample}
Let $f(., \z)$ be an $L$-smooth $G$-Lipschitz convex function $\forall \ \z$.
Algorithm \ref{alg:sub-sample-GD}, run with $t_0=1, \eta = \min\bc{\frac{1}{2L}, \frac{D\sqrt{\rho n}}{GT}}, T=\frac{DL\sqrt{\rho n}}{G}$, and $m= \max{\bc{\frac{G\sqrt{\rho n}}{DL},1}}$, outputs $\hat \w_S$ which is $\min\bc{\rho,1}$-TV-stable and satisfies $\E{\hat F_S(\hat \w_S) - \hat F_S(\w^*_S)} \lesssim \frac{GD}{\sqrt{\rho n}}.$
\end{proposition}

\subsubsection{\emph{noisy-m-A-SGD}}

The second algorithm, superior in low dimensions, called  \emph{noisy-m-A-SGD} is mini-batch accelerated SGD with appropriate Gaussian noise added at each iteration. In the literature, this algorithm (with or without acceleration) is used for DP training of (non) convex models.
In each iteration, we save the mini-batch indices, the models, the gradients as well as the noise vectors to memory.

\begin{algorithm}[ht]
\caption{\emph{noisy-m-A-SGD}($\w_{t_0},t_0)$}
\label{alg:noisy-m-a-sgd}
\begin{algorithmic}[1]
\REQUIRE{Initial model $\w_{t_0}$, data points $\bc{\z_1,\ldots, \z_n}$, $T$, $\eta$, $m$}
\STATE $\w_{0} = 0$
\FOR{$t=t_0,t_0+1\ldots, T$}
\STATE Sample mini-batch $b_t$ of size $m$ uniformly randomly
\STATE Sample $\theta_t \sim \cN(0,\sigma^2\bbI_d)$
\STATE $ \accw_t =  (1 - \alpha_t)\w_t + \alpha_t  \w_{t-1} $ 
\STATE $\g_t =  \frac{1}{m}\sum_{j \in b_t}\nabla f(\accw_t,\z_j)$
\STATE $\w_{t+1} =  \cP\br{\accw_t - \eta\br{  \g_t +\theta_t}}$
\STATE Save($b_t,\theta_t,\w_t,\accw_t, \g_t$)
\ENDFOR
\ENSURE{$\hat \w_S =\w_{T+1}$}
\end{algorithmic}
\end{algorithm}

We now state our results for \emph{noisy-m-A-SGD}.

\begin{proposition}
\label{prop:upper-bound-noisy-m-a-sgd}
Let $f(., \z)$ be an $L$-smooth $G$-Lipschitz convex function $\forall \ \z$.
For any $0<\rho<\infty$,
Algorithm \ref{alg:noisy-m-a-sgd}, run with
$t_0=1, \eta = \min\bc{\frac{1}{2L}, \frac{D}{\br{\frac{G}{\sqrt{m}} +\sigma}T^{3/2}}}$, $\alpha_0=0, \alpha_t = \frac{1-t}{t+2}$,
$\sigma = \frac{8\sqrt{T}G}{n\rho}$, and $T \geq \frac{(n\rho)^2}{16m^2}$ outputs $\hat \w_S$ which is $\min\bc{\rho,1}$-TV stable and satisfies 
$$\E{\hat F_S(\hat \w_S) - \hat F_S(\w^*_S)} \lesssim  \frac{LD^2}{T^2}  + \frac{GD}{\sqrt{Tm}} + \frac{G D\sqrt{d}}{n \rho}.$$
\end{proposition}

Choosing $T$ and $m$ appropriately gives us the following corollary.
\begin{corollary}
\label{corr: upper-bound-m-A-sgd}
Let $f(., \z)$ be an $L$-smooth $G$-Lipschitz convex function $\forall \ \z$.
For any $0<\rho<\infty$,
Algorithm \ref{alg:noisy-m-a-sgd}, run with $t_0=1, m \geq \min\bc{\frac{d}{16}, \frac{1}{4}\br{\frac{(\rho n)^3 G\sqrt{d}}{LD}}^{1/4}}$,
$\eta = \min\bc{\frac{1}{2L}, \frac{D}{\br{\frac{G}{\sqrt{m}} +\sigma}T^{3/2}}}$, $\alpha_0=0, \alpha_t = \frac{1-t}{t+2}$,
$\sigma = \frac{8\sqrt{T}G}{n\rho}$, and $T=\max\bc{\frac{(\rho n)^2}{md}, \sqrt{\frac{LD\rho n}{G\sqrt{d}}}}$ outputs $\hat \w_S$ which is $\min\bc{\rho,1}$-TV stable and satisfies $\E{\hat F_S(\hat \w) - \hat F_S(\w^*_S)} \lesssim \frac{GD\sqrt{d}}{\rho n}$.
\end{corollary}

\begin{remark}
The choice of $T$ in  \cref{corr: upper-bound-m-A-sgd} yields that the largest mini-batch size that can be set, without hurting runtime, is $m = \br{\frac{(\rho n)^3 G}{\sqrt{d}^3 L D}}^{1/2} = \br{\frac{G}{LD}}^2T^3$. 
Furthermore, the condition $m \geq \min\bc{\frac{d}{16}, \frac{1}{4}\br{\frac{(\rho n)^3 G\sqrt{d}}{LD}}^{1/4}}$ yields $(\rho n) \geq \br{\frac{LD (\sqrt{d})^7}{256 G}}^{1/3}$.
\end{remark}

In \cref{prop:tv-stability-noisy-sgd} (in Appendix , \cref{sec:alg-learning}), we  show that the upper bound on total variation stability parameter of Algorithm \ref{alg:noisy-m-a-sgd} 
derived in Proposition
\ref{prop:upper-bound-noisy-m-a-sgd} is tight in all problem parameters, upto constants.

\subsection{Unlearning algorithms}
\label{sec:alg-unlearning}

We now discuss algorithms to handle edit requests which are based on efficiently constructing couplings, in some cases maximal couplings. 
An important component on constructing such couplings is, what we call \emph{verification}, wherein, at a high-level, we check if the current model is \emph{likely} after the edit request or not. If the verification is successful, we don't do any additional computation, otherwise we do a partial or full recompute (i.e. retrain), which we call \adjust{}.
The key insight is that verification can be done efficiently, and fails with small probability (depending on the TV-stability parameter).

We now discuss the two algorithms, one for handling unlearning in Algorithm \ref{alg:sub-sample-GD} and the other for Algorithm \ref{alg:noisy-m-a-sgd}, and show that the probability with which a recompute is triggered is small -  please see  \cref{sec:runtime}, for a finer analysis of runtime.
The proofs of results in this section are deferred to \cref{sec:proofs-unlearning}.

\subsubsection{Unlearning for \emph{sub-sample-GD}}

At the start of the stream, at every iteration of
\emph{sub-sample-SGD},
we sample a mini-batch of size $m$ out of $n$ points uniformly randomly, and then compute a gradient using these samples - note that this is the only source of randomness in the algorithm. 
As we progress along the stream observing edit requests, the number of available data points changes.
Therefore, if the algorithm were executed on this dataset of, say $\tilde n$ points, at every iteration it would have sub-sampled $m$ out of $\tilde n$ (and not $n$) points. 
The way to account for this discrepancy is to simply adjust the sub-sampling probability measure accordingly.

\paragraph{Coupling mini-batch indices:} 
The main idea to unlearning in \cref{alg:unlearn-subsample-gd} is to couple the sub-sample indices. 
For deletion, we just look at each mini-batch, and (literately) verify if the deleted point were used or not. If the deletion point was not used in any iterations, then we don't do anything, otherwise, we trigger a recompute. In the case of insertion, there is no such way of selecting iterations in which the point was sampled, because the inserted point was absent. However, we know that the new point would have been sampled with probability $m/(n+1)$. 
We can thus verify by selecting each iteration with the same probability.
We then replace a uniformly sampled point in the mini-batch of that step by the inserted point.
Algorithm \ref{alg:unlearn-subsample-gd} implements the above procedure.

\label{sec:unlearn-noisy-m-a-sgd-appendix-main}

\begin{algorithm}[!htbp]
\caption{Unlearning for \emph{sub-sample-GD}}
\label{alg:unlearn-subsample-gd}
\begin{algorithmic}[1]
\REQUIRE{Data point index $j$ to delete or data point $\z$ to insert (index $n+1$)}
\FOR{$t=1,2\ldots, T$}
\STATE Load$\br{ b_t,\g_t,\w_t}$
\IF{\textit{deletion} and $j \in b_t$}
\STATE \emph{sub-sample-GD}$(\w_t,t)$ \hfill{\textit{// Continue training on current dataset}}
\BREAK
\ELSIF{\textit{insertion}  and Bernoulli$\br{\frac{m}{n+1}}$}
\STATE Sample $i \sim \text{Uniform}(b_t)$
\STATE $\g_t' = \g_t-  \frac{1}{m}\br{\nabla f(\w_t,\z_i) -  \nabla f(\w_t,\z)}$
\STATE $\w_{t+1} =  \cP\br{\w_t - \eta\br{  \g_t' +\theta_t}}$
\STATE Save($\w_{t+1}, \g_t', b_t\backslash \bc{i} \cup \bc{n+1} $)
\STATE \emph{sub-sample-GD}$(\w_{t+1}, t+1)$ \hfill{\textit{// Continue training on current dataset}}
\BREAK
\ENDIF
\ENDFOR
\end{algorithmic}
\end{algorithm}

We state our main result for unlearning with \cref{alg:unlearn-subsample-gd} below.

\begin{proposition}
\label{prop:unlearn-subsample-GD}
(Algorithm \ref{alg:sub-sample-GD}, Algorithm \ref{alg:unlearn-subsample-gd})  satisfies exact unlearning. Moreover, for $k$ edits, Algorithm \ref{alg:unlearn-subsample-gd} recomputes with probability at most $2k\rho$.
\end{proposition}

\subsubsection{Unlearning for \emph{noisy-m-A-SGD}}
\label{sec:unlearn-noisy-m-a-sgd}

\begin{algorithm}[ht]
    \caption{Unlearning for \emph{noisy-m-A-SGD}}
    \label{alg:unlearn-noisy-m-a-sgd}
    \begin{algorithmic}[1]
        \REQUIRE{Data point index $j$ to delete or data point $\z$ to insert (index $n+1$)}
        \FOR{$t=1,2\ldots, T$}
        \STATE Load$\br{\theta_t, \w_t,\accw_t, b_t, \g_t}$
        \IF{\textit{deletion} and $j \in b_t$}
        \STATE Sample $i \sim \text{Uniform}([n]\backslash b_t )$
        \STATE $ \g_t' = \g_t -   \frac{1}{m}\br{\nabla f(\accw_t,\z_j) -  \nabla f(\accw_t,\z_i)}$
        \STATE Save($ \g_t', b_t\backslash \bc{j} \cup \bc{i} $)
        \ELSIF{\textit{insertion}  and Bernoulli$\br{\frac{m}{n+1}}$}
        \STATE Sample $i \sim \text{Uniform}(b_t)$
        \STATE $\g_t' = \g_t-  \frac{1}{m}\br{\nabla f(\accw_t,\z_i) -  \nabla f(\accw_t,\z)}$
        \STATE Save($ \g_t', b_t\backslash \bc{i} \cup \bc{n+1} $)
        \ELSE
        \CONTINUE
        \ENDIF
        \STATE $\xi_t = \g_t + \theta_t$
        \IF{Uniform$\br{0, 1} \geq \frac{\phi_{\cN(\g_t',\sigma^2\bbI)}(\xi_t)}{ \phi_{\cN(\g_t,\sigma^2\bbI)}(\xi_t)}$}
        \STATE $ \xi'_t = \text{reflect}(\xi_t,\g_t', \g_{t})$
        \STATE $\w_{t+1} = \w_t - \eta  \xi_t'$
        \STATE $\text{Save}(\xi'_t)$
        \STATE \emph{noisy-m-A-SGD($\w_{t+1},t+1$)} \hfill{\textit{ // Continue retraining on current dataset }}
        \BREAK
        \ENDIF
        \ENDFOR
    \end{algorithmic}
\end{algorithm}

\noindent
Our unlearning algorithm for \emph{noisy-m-A-SGD} is based on efficiently constructing a coupling of Markov chain describing \emph{noisy-m-A-SGD}, with large mass on its diagonal.
The key ideas have already been described in \cref{sec:main-idea-unlearning}, and we just fill in some details here.
We first describe how \cref{alg:unlearn-noisy-m-a-sgd} couples mini-batch indices while handling edit request.

\paragraph{Coupling mini-batch indices:} 
After observing a deletion request, in Algorithm \ref{alg:unlearn-noisy-m-a-sgd},
we look at all iterations in which the deleted point was sampled. We then replace the deleted point with a uniformly random point not already sampled in that iteration.  For insertion, at each step, we again replace a uniformly sampled point in the mini-batch of that step by the inserted point with probability $\frac{m}{n+1}$.

\paragraph{Reflection maps:}
We define the notion of \textit{reflection map}, which will be used in our coupling construction.

\begin{definition}[Reflection map]
Given a vector $\u$ and two vectors $\x$ and $\y$, the reflection of $\u$ under $(\x,\y)$, denoted as $\text{reflect}(\u,\x,\y)$, is defined as
\begin{align*}
   \text{reflect}(\u,\x,\y) = \x +(\y -\u)
\end{align*}
\end{definition}

Reflection coupling is a classical idea in probability, used to construct couplings between symmetric probability distributions \cite{lindvall1986coupling}.
The reflection map, given $\u, \x, \y$, reflects $\u$ about the mid-point of $\x$ and $\y$. 
The context in which we will use it is $\u$ would be a sampled point from a Gaussian under old dataset $S$ (on which the model was trained on), and $\x$ and $\y$ being the means of the Gaussian under  new dataset $S'$ (after edit request) and $S$ respectively. 
The map essentially exploits the spherical symmetry of the Gaussian to generate a \emph{good} sample for the distribution under $S'$. Please see \cref{sec:reflection} for some properties of the reflection map, which are used in the final proofs.

\paragraph{Iterative rejection sampling:} Our unlearning algorithm is based on iteratively verifying each model $\w_{t+1}$ using rejection sampling. To elaborate, at each iteration, we check if the noisy iterate, defined as $\bar \w_{t+1} = \accw_{t} - \eta (\g_{t} + \theta_{t})$ is a \emph{good} sample for the dataset $S'$, where $\g_t$ is the gradient computed on $\accw_t$ using a uniform sub-sample from $S$. 
To do this, we need to compute a ratio of \emph{estimated} marginal densities of $\w_{t+1}$ for both datasets, evaluated at the noisy iterate, and compare it with $\text{Uniform}(0,1)$. 
It it succeeds, we move to the next iteration and repeat. If any of the rejection sampling fails, we do a reflection, and continue retraining on $S'$.

\paragraph{Estimation of marginals:} 
We explain what we mean by \emph{estimated} marginal densities in the previous paragraph.
As remarked before, if we did not sub-sample mini-batches (i.e. used gradient descent), then we would simply use the marginal distribution of iterates for rejection sampling. However, that would amount to a worse runtime. Instead, we estimate the marginal densities as follows: 
fix all iterates before iteration $t$, and consider noisy iterate $\bar \w_{t+1} = \accw_t - \eta (\g_t + \theta_t)$. 
If we also fix the sampled mini-batch $b_t$, then $\bar \w_{t+1}$
is distributed as $\cN(\accw_t - \eta \g_t, \eta^2\sigma^2 \bbI)$.
However, once we unfix 
$b_t$, then $\w_{t+1}$ is mixture of Gaussians, with the number of components being exponential in $m$.  
Ideally, to do rejection sampling, we need to compute the marginal density of the distribution of $\w_{t+1}$ (and $\w_{t+1}'$ - the iterate for dataset $S'$) evaluated at $\accw_t - \eta(\g_t+\theta_t)$ - computing which however, is infeasible. 
Therefore, we just use the coupled mini-batches indices as a sample from the mixture and estimate the marginal density using the conditional density - this is done in line 15 of \cref{alg:noisy-m-a-sgd}, with a small change that we evaluate the ratio of conditional densities of noisy gradients rather than iterates, but it can be verified that the ratio is invariant to this shift and scaling.

Please see \cref{sec:unlearn-noisy-m-a-sgd-appendix} for a more formal treatment of the coupling procedure. We now state the main result for this section.

\begin{proposition}
\label{prop:unlearn-noisy-m-A-SGD}
(Algorithm \ref{alg:noisy-m-a-sgd}, Algorithm \ref{alg:unlearn-noisy-m-a-sgd}) satisfies exact unlearning. Moreover, for $k$ edits, Algorithm \ref{alg:unlearn-noisy-m-a-sgd} recomputes with probability at most $\frac{k\rho\sqrt{T}}{4}$
\end{proposition}

%% file: sections/proofmainresults.tex
\section{Proofs of main results}
In this section, we give the proofs of main results, stated in \cref{sec:main-results}, using the results in the preceding sections.

\subsection{Proof of Theorem \ref{thm:main-result}}
The proof follows by combining the guarantees for the two algorithms we present: \emph{sub-sample-GD} (Algorithm \ref{alg:sub-sample-GD}) and \emph{noisy-m-A-SGD} (Algorithm \ref{alg:noisy-m-a-sgd}), and their corresponding unlearning algorithms: Algorithm \ref{alg:unlearn-subsample-gd} and Algorithm \ref{alg:unlearn-noisy-m-a-sgd}. 
We discuss these one by one.
From \cref{prop:upper_bound_subsample}, we have that, given $0<\rho\leq 1$, \emph{sub-sample-GD} is $\rho$-TV stable and has excess empirical risk bounded by $O\br{\frac{GD}{\sqrt{\rho n}}}$. This holds at every point in the stream by assumption that the number of samples are between $\frac{n}{2}$ and $2n$. Furthermore, from \cref{prop:unlearn-subsample-GD}, we have that the unlearning algorithm 
satisfies exact unlearning at every point in the stream, proving the first part of the claim for \emph{sub-sample-GD}.
Moreover, it states that recompute probability for $k$ edit requests is $O(\rho k)$. Finally, from \cref{claim:runtime_sub_sample_GD}, we have that there exist efficient implementations, such that the runtime of unlearning for \emph{sub-sample-GD} is $O(\max\bc{k,\min\bc{\rho,1} k \cdot \text{Training time}}$, where "Training time" is the runtime of the corresponding learning algorithm - this means that re-computations overwhelm the total unlearning time. This establishes all the guarantees for one algorithm and recovers one of the upper bounds in the second claim.

The situation for the other algorithm is a little more involved.
From \cref{prop:upper-bound-noisy-m-a-sgd}, for dataset $S$ of $n$ points, we have that, given $0<\tilde \rho\leq 1$, \emph{noisy-m-A-SGD} is $\tilde \rho$-TV stable and its excess empirical risk is bounded as follows:

$$\E{\hat F_S(\hat \w_S) - \hat F_S(\w^*_S)} \lesssim  \frac{LD^2}{T^2}  + \frac{GD}{\sqrt{Tm}} + \frac{G D\sqrt{d}}{n \tilde\rho},$$

\noindent where $T$ is the number of iterations for \emph{noisy-m-A-SGD} algorithm, and $m$ the mini-batch size.
From Proposition \ref{prop:unlearn-noisy-m-A-SGD}, we have that the unlearning algorithm satisfies exact unlearning (establishing the first claim) and recomputes, for $k$ edit requests, with probability $O(\tilde \rho k \sqrt{T})$.
Finally, from \cref{claim:runtime_noisy_A_SGD}, we have that there exist efficient implementations, such that the runtime of unlearning for \emph{noisy-m-A-SGD} is $O(\max\{k,k \min\bc{\tilde \rho  \sqrt{T},1} \cdot \text{Training time}\})$.
In the statement of \cref{thm:main-result}, we want that the unlearning runtime be such 
that we recompute  for a $\rho$ fraction of edit requests (as opposed to something dependent on $T$). 
Therefore, we substitute $\tilde \rho = \frac{\rho}{\sqrt{T}}$, and this changes the excess empirical risk bound for  \emph{noisy-m-A-SGD}, as follows:

$$\E{\hat F_S(\hat \w_S) - \hat F_S(\w^*_S)} \lesssim  \frac{LD^2}{T^2}  + \frac{GD}{\sqrt{Tm}} + \frac{G D\sqrt{d}\sqrt{T}}{n \rho}.$$

We use the largest mini-batch size, which does not hurt runtime, which is $m=\br{\frac{G}{LD}}^2T^3$. This simplifies the upper bound to $\frac{LD^2}{T^2} +  \frac{G D\sqrt{d}\sqrt{T}}{n \rho} $. Optimizing the trade-off, we have $ \frac{LD^2}{T^2} = \frac{G D\sqrt{d}\sqrt{T}}{n \rho} \iff T = \br{\frac{LD (n\rho)}{G \sqrt{d}}}^{2/5}$, and the excess empirical risk becomes $\E{\hat F_S(\hat \w_S) - \hat F_S(\w^*_S)} \lesssim \frac{LD^2}{T^2} = \br{\frac{L^{1/4}GD^{3/2} \sqrt{d}}{(\rho n)}}^{4/5}$ -- this recovers the other term in the upper bound in \cref{thm:main-result}. However, note that  Proposition \ref{prop:upper-bound-noisy-m-a-sgd} has an additional condition that $T\geq \frac{(n\tilde \rho)^2}{16m^2}$ - we show that in our setting of $\tilde \rho$ and $m$, this condition is equivalent to the excess empirical risk of \emph{noisy-m-A-SGD} being smaller than that of \emph{sub-sample-GD}. Hence, the regime in which the aforementioned condition is violated is the same regime in which it is better to use the other \emph{sub-sample-GD} algorithm, and therefore
is benign. Setting $m=\br{\frac{G}{LD}}^2T^3$ and $\tilde \rho = \rho/ \sqrt{T}$, the condition simplifies as $T\geq \frac{(n\rho)^2}{16 T\cdot T^6 } \br{\frac{LD}{G}}^2 \iff T^8 \geq \frac{(n\rho)^2}{16} \br{\frac{LD}{G}}^4 \iff \br{\frac{LD (n\rho)}{G \sqrt{d}}}^{16/5} \geq \frac{(n\rho)^2}{16} \br{\frac{LD}{G}}^4 \iff  \br{\frac{\sqrt{d}}{(n\rho)}}^{4/5} \leq \frac{2}{\sqrt{n\rho}}\br{\frac{LD}{G}}^{1/5} \iff \br{\frac{L^{1/4}GD^{3/2} \sqrt{d}}{(\rho n)}}^{4/5} \leq \frac{2GD}{\sqrt{n\rho}}$, where the final inequality indicates that the expected excess empirical risk of \emph{noisy-m-A-SGD} is at most that of \emph{sub-sample-GD}, up to constants.
The above is established for dataset $S$ but holds for any dataset $S^i$ in the stream using the assumption that the number of samples are between $\frac{n}{2}$ and $2n$.

Combining the above arguments finishes the proof of \cref{thm:main-result}.
\hfill \qed

\subsection{Proof of Theorem \ref{thm:main-upper-bound}}
We give two algorithms, \emph{sub-sample-GD} (Algorithm \ref{alg:sub-sample-GD}) and \emph{noisy-m-A-SGD} (Algorithm \ref{alg:noisy-m-a-sgd}), one for each of the upper bounds.
From \cref{prop:upper_bound_subsample} and \cref{corr: upper-bound-m-A-sgd}, we have that, given $0<\rho < \infty$, these are $\min\bc{\rho,1}$-TV stable and their excess empirical risk is bounded is $O\br{\frac{GD}{\sqrt{\rho n}}}$ and $O\br{\frac{GD\sqrt{d}}{\tilde \rho n}}$  respectively.
Hence combining the above by taking a minimum, establishes the claimed result. 
\hfill \qed

\subsection{Proof of Theorem \ref{thm:main-lower-bound}}
In all the lower bounds, we have a $GD$ term - this is a trivial lower bound, since if an algorithm  is defined as $\cA(S)=0$ (or any constant), then this is perfectly $TV$ stable ($\rho=0$), and the expected excess empirical risk is upper bounded as $\hat F_S(\cA(S)) - \hat F_{S}(\w_{S}^*) \leq G \norm{\cA(S)-\w_{S}^*} \leq GD$, where the first inequality uses $G$-Lipschitzness of $\hat F_S$ and the second the fact the both $\cA(S)$ and $\w_{S}^*$ lie in a ball of diameter $D$. Hence, attaining an excess empirical risk of $GD$ is trivial, and we now focus on deriving the other terms in the bounds.

Firstly, as discussed in \cite{bassily2014private}, we  consider $G=D=1$, since a simple reduction gives a factor of $GD$ for general $G$ and $D$. Furthermore, similar to \cite{bassily2014private}, we show that the problem of TV-stable convex ERM is at least as hard as that of TV stable mean computation of a dataset with bounded mean - we state this reduction in \cref{prop:lower_bound_reduction}.
We now focus on showing accuracy lower bounds for $\rho$-TV-stable mean computation of dataset $S$ of size $n$, with mean $\frac{M}{2}\leq \norm{\mu(S)} \leq 2M$. The accuracy, denoted by $\alpha$, is defined as $\alpha^2 = \mathbb{E}\norm{\cA(S)-\mu(S)}^2$,
$\cA$ is a $\rho$-TV stable algorithm, and the expectation is taken over the algorithm's randomness.
The first part of \cref{thm:main-lower-bound} follows  \cref{thm:lower_bound11} which is based on a simple reduction argument.
This gives us that $\alpha \geq \frac{1}{\rho n}$ with $M = \frac{1}{\rho n}$. Plugging it in \cref{prop:lower_bound_reduction}, this gives us that excess empirical risk is lower bounded by $\Omega\br{\frac{1}{\rho n}}$.
Similarly, the second part follows from \cref{thm:lower_bound2} which gives us $\alpha \geq \frac{1}{\sqrt{\rho n}}$ with $M=\frac{1}{\sqrt{\rho n}}$ - the condition $\alpha \leq \frac{1}{4}$ in the statement of \cref{thm:lower_bound2} can be absorbed in the trivial lower bound $GD$.
\hfill \qed

%% file: sections/discussion.tex
\section{Discussion}
In this work, we presented the TV stability framework for machine unlearning and instantiated it to develop unlearning algorithms for convex risk minimization problems. Currently, our results indicate two gaps, and motivate the following future directions.
\begin{enumerate}
    \item \textbf{Optimal TV-stable algorithm:} Our upper and lower bound on excess empirical risk of TV stable algorithms don't match. 
    Hence, we either need to establish stronger lower bounds (arguably, more likely) or search for better algorithms.
    \item \textbf{Maximal coupling for unlearning:} Our coupling procedure for unlearning for \emph{noisy-m-A-SGD} is sub-optimal, in measure of its diagonal, by a $\sqrt{T}$ factor. A natural question is whether we can design an \emph{efficient} maximal coupling. 
    We note that if efficiency were not a criteria, then this can be done -  briefly, do a one step rejection sampling by computing the ratio of joint distribution iterates, if it fails, keep retraining, until the iterates generated is accepted by a rejection sampling.
    However, in this case, the expected number of retrains can be shown to be one, and so is trivial. The challenge in this case is to give an efficient procedure when the first rejection sampling fails.
     \item \textbf{Beyond smooth convex functions}: 
     The focus of this work was on smooth convex (loss) functions, but our techniques, and results for unlearning, extend to general non-convex functions.
     However, a careful investigation of trade-offs between accuracy and unlearning efficiency, in classes of, say strongly-convex, non-smooth or even \emph{some} non-convex functions, is an interesting future direction.
\end{enumerate}

%% file: sections/additionalrelatedwork.tex
\section{Additional related work}
\label{sec:detailed-related-work}
We survey the works on machine unlearning - 
\cite{cao2015towards} were one of the first papers to study the topic of machine unlearning. Their approach implements statistical query (SQ) algorithms by estimating the statistical queries using training data. Since the estimates are usually the mean of query evaluations computed on training data, unlearning is cheap, as we only need to subtract the evaluation on the deleted point.
\cite{bourtoule2019machine} studies this problem, with the goal to design systems to efficiently handle deletion requests.
Their approach, called SISA, is essentially a divide-and-conquer strategy, wherein the data is divided into disjoint sets, called \emph{shards}, and a model on each shard is trained separately and aggregated. Furthermore, they do several check-pointing of states for each shard. In the average case, this provides a speedup of $\frac{(R+1)S}{2}$ for $S$ shards and $R$ checkpoints per shard, over retraining. They however give no guarantees on accuracy with this divide-and-conquer training method.
\cite{guo2019certified} is another work which uses $(\epsilon, \delta)$-differential privacy like guarantee. They study unlearning in generalized linear models, and propose a Newtons-step based method, leveraging connections with influence functions. Their computational cost is $O(d^3)$ computations for one unlearning. They, however give no guarantees on excess empirical risk achieved by the training method.
Finally, the work of \cite{izzo2020approximate} studies batch unlearning in linear regression, with the goal to improve the computational cost of batch $k$ unlearning requests. Their method achieves a runtime of $O(k^2d)$ as opposed to $O(kd^2)$ for a naive approach. However, their notion of unlearning is again approximate, in the sense that model returned after unlearning is closest to the exact unlearning model among models in the $d$ dimensional subspace spanned by the to-be-deleted $k$ points. So it is easy to see that with larger $k$, the notion of approximation improves, which explains the $k^2$ term in the runtime as opposed to $k$.

\paragraph{Comparison with  \cite{neel2020descent}.}
Our algorithm guarantees provable exact unlearning with probabilistic  
runtime guarantees, whereas \cite{neel2020descent} give algorithms with deterministic runtime and provide only an approximate $(\epsilon, \delta)$-DP based unlearning guarantee -- the $\delta$ can be interpreted as probability of the failure event in Monte-Carlo guarantees.
To handle these discrepancies when comparing, our stated runtime is the in-expecatation runtime. For a fixed runtime, we will look at regimes of $\epsilon$ and $\delta$, when the accuracy guarantee of \cite{neel2020descent} is smaller than ours. We remind that a large $\epsilon,\delta$ means a weaker unlearning criterion.
We will see that with same runtime, the accuracy of \cite{neel2020descent}  is smaller than ours in the regime when their unlearning parameters and hence the notion, is rather weak.

Considering the Lipschitz, smoothness parameters and diameter as constants,  for smooth convex functions and $k$ edit requests, \cite{neel2020descent} (Theorem 3.4) achieve an excess empirical risk of $O\br{\frac{\sqrt{d}\sqrt{\log{1/\delta}}}{\epsilon n k}}^{2/5}$ with an unlearning runtime of $k^2$ full-gradient computations. On the other hand, our algorithms achieve an an excess empirical risk of $\min\bc{\frac{1}{\sqrt{\rho n}}, \br{\frac{\sqrt{d}}{\rho n}}^{4/5}}$ with $\rho k$ expected re-computations. Each re-computation takes $m\cdot T$ gradient computations where $m$ is the mini-batch size and $T$ the number of iterations. Therefore, in order to have the same runtime, we need $\rho k m T= k^2n \iff \rho = \frac{kn}{mT}$. 
Firstly, note that as as long as $d \leq (\rho n)^{3/4}$, \emph{noisy-m-A-SGD} has smaller excess empirical risk than \emph{sub-sample-GD} - this are the two regimes of interest.
We now set $m$ and $T$ for both the algorithms:
for Algorithm \ref{alg:sub-sample-GD}, $m=\frac{\rho n}{T}$ and  $T= \sqrt{\rho n}$. This gives us $\rho = \frac{kn}{\rho n} \iff \rho = \sqrt{k}$,  however $\rho$ is the Total Variation distance and is at most $1$. Hence in regime $d\geq (\rho n)^{3/4}$, our runtime is always smaller than \cite{neel2020descent}: $kn$ as opposed to $k^2n$ gradient computations. Even with $\rho=1$, our excess empirical risk is $\frac{1}{\sqrt{n}}$ and the excess empirical risk of \cite{neel2020descent} is smaller than ours when $\br{\frac{\sqrt{d}\sqrt{\log{1/\delta}}}{\epsilon n k}}^{2/5} \lesssim \frac{1}{\sqrt{n}} \iff \frac{\epsilon}{\sqrt{\log{1/\delta}}} \gtrsim \frac{\sqrt{d}n^{1/4}}{k}$. 
In the second regime $d< (\rho n)^{3/4}$, we use Algorithm \ref{alg:noisy-m-sgd}, wherein we have $mT = \frac{(\rho n)^2}{d}$. This gives us $\rho = \frac{kn d}{(\rho n)^2} \iff \rho = \sqrt{\frac{kd}{n}}$, and our excess empirical risk is $O\br{\br{\frac{\sqrt{d}}{\rho n}}^{4/5}} = O\br{\frac{1}{(nk)^{2/5}}}$.
Therefore, excess empirical risk of \cite{neel2020descent} is smaller than ours when $\br{\frac{\sqrt{d}\sqrt{\log{1/\delta}}}{\epsilon n k}}^{2/5} \lesssim \frac{1}{(nk)^{2/5}} \iff \frac{\epsilon}{\sqrt{\log{1/\delta}}} \gtrsim d$.
We therefore have that unless $k$ is very large, the accuracy of \cite{neel2020descent} is smaller than ours when $\epsilon$ and $\delta$,  take prohibitively large values which correspond to a weak notion of approximate unlearning.
We can similarly compare against Theorem 3.5 in \cite{neel2020descent}, which will yield qualitatively similar conclusions.

%% file: sections/proofslearning.tex
\section{Proofs for Section \ref{sec:alg-learning}}
\label{sec:proofs-learning}

\begin{proof}[Proof of Proposition \ref{prop:upper_bound_subsample}]
We first show that Algorithm \ref{alg:sub-sample-GD} is $\min\bc{1,\rho}$-TV stable for the aforementioned choice of number of iterations $T$ and mini-batch size $m$. 
Consider neighbouring dataset $S$ and $S'$ of $n$ points which differs in one sample, WLOG say the $n^\text{th}$ sample. 
Let $\cA(S):=\hat\w_S$ and  $\cA(S'):=\hat\w_{S'}$ denote the outputs of Algorithm \ref{alg:sub-sample-GD} on $S$ and $S'$ respectively.
Since in Algorithm \ref{alg:sub-sample-GD}, the randomness is only on indices, rather than actual data points, say that 
$S=\bc{1,2,\ldots,n}$.
Now we consider neighbouring dataset $S'$, which contains  $n+1$ or $n-1$ samples. 
We will now consider the case when $S'$ contains $n-1$ elements and the case with $n+1$ elements will follow analogously.
Let $n$ be the index present in $S$ but absent in $S'$ 
i.e. $S' = \bc{1,2,\ldots, n-1}$.
Let the sigma-algebra on these sets be the power sets of $S$ and $S'$ respectively, denoted by $\text{Pow}(S)$ and $\text{Pow}(S')$ respectively. Moreover, let $\mu_{n,m}$ denote the sub-sampling probability measure on $n$ points in $S$ i.e it sub-samples $m$ out of $n$ elements in $S$ uniformly randomly. Let $\mu_{n,m}^{\otimes T}$ denote the product measure of $T$ of $\mu_{n,m}$'s. 
We similarly define $\mu_{n-1,m}$ and $\mu_{n-1,m}^{\otimes T}$ for $S'$.

We first extend the sigma-algebra for the probability spaces so that the random variables $\mu_{n,m}$ and $\mu_{n-1,m}$, are defined on a common probability space. 
For this, we will just add an event where the index $n$ can be sampled under $\mu_{n-1,m}$ with probability $0$. We define $\mu_{n,m}'$ as follows: for any set $b \in \text{Pow}(S)$,
 $\mu_{n,m}'(b) = \begin{cases}\mu_{n-1,m}(b) & \text{ if } n \not \in b\\ 0 & \text{otherwise} \end{cases}$.
We similarly extend the sigma algebra for the product space with measure
$\mu_{n-1,m}^{\otimes T}$  to get $\mu_{n,m}^{'\otimes T}$.

Observe that for fixed initialization $\w_0$ and other parameters, Algorithm $\cA(S)$ and $\cA(S')$ is the \emph{same} (deterministic) map from ${\mathbf{b}} = (b_1,b_2,\ldots, b_T)$ where $b_j \in [n]^m$ to $\cW$. 
They only differ because of different measures on the input space.
Hence total variation distance between $\cA(S)$ and $\cA(S')$ is just the total variation distance between the push-forward measures $\cA(S)_\# \mu_{n,m}^{\otimes T}$ and $\cA(S')_\# \mu_{n,m}^{'\otimes T}$ which by using the fact that $\cA(S)\equiv\cA(S')  $ and data-processing inequality, is at most the total variation distance between $\mu_{n,m}^{\otimes T}$ and $\mu_{n,m}^{'\otimes T}$.
Now the total variation distance can be bounded as,

\begin{align*}
    \text{TV}(\cA(S), \cA(S')) & \leq  \text{TV}(\mu_{n,m}^{\otimes T}, \mu_{n,m}^{'\otimes T}) =  
    \sup_{{\mathbf{b}} \in \text{Pow}([n]^m)^T)} \abs{\mu_{n,m}^{\otimes T}({\mathbf{b}}) -  \mu_{n,m}^{'\otimes T}({\mathbf{b}})} 
    \\& = \mu_{n,m}^{\otimes T}\left({\mathbf{b}} \text{ such that at least one }b_j \text{ contains  }n \right)\\
    & \leq T \mu_{n,m}(b_1 \text{ contains } n) = \frac{Tm}{n}
\end{align*}
where the inequality follows using a union bound.

A similar argument works when $S'$ is an neighbouring dataset of $n+1$ elements, yielding a total variation bound of $\frac{Tm}{n+1} \leq \frac{Tm}{n}$.
Taking a uniform bound over all neighbouring datasets $S'$, we get that $\sup_{\Delta(S,S')=1}\text{TV}(\cA(S),\cA(S')) \leq \frac{Tm}{n}$. 
By definition of TV distance, we trivially have that $\sup_{\Delta(S,S')=1}\text{TV}(\cA(S),\cA(S'))\leq 1$. Therefore, setting $m = \frac{\rho n}{T}$, we get the desired result that the output of Algorithm \ref{alg:sub-sample-GD} is $\min\bc{\rho,1}$-TV stable.

We now proceed to the accuracy guarantee which follows directly by analysis of SGD. We first show that the sub-sampling procedure produces unbiased gradients and bound its variance. For a fixed model $\w$, we have that

\begin{align*}
    \mathbb{E}_b\frac{\sum_{j \in b} \nabla f(\w, z_j)}{m} = \sum_{\binom{n}{m} \text{ choices for }b}\frac{ \sum_{j \in b}\nabla f(\w, z_j)}{m \binom{n}{m} } = \frac{\binom{n-1}{m-1}}{m \binom{n}{m}} \sum_{j=1}^n \nabla f(\w, z_j) = \frac{\sum_{j=1}^n \nabla f(\w, z_j)}{n}
\end{align*}

where in the second equality, we use the observation that every $\z_j$ appears in exactly $\binom{n-1}{m-1}$ terms over all choices for $b$. We now bound its variance, denoted by a $\cV^2$ by direct computation.

\begin{align*}
       \cV^2 &= \mathbb{E}_b \norm{\frac{\sum_{j \in b} \nabla f(\w, z_j)}{m} -     \mathbb{E}_b \left[\frac{\sum_{j \in b} \nabla f(\w, z_j)}{m} \right]}^2  \\&=  \mathbb{E}_b \norm{\frac{\sum_{j \in b} \nabla f(\w, z_j)}{m}}^2 - \norm{ \mathbb{E}_b \left[\frac{\sum_{j \in b} \nabla f(\w, z_j)}{m} \right]}^2 \\
       & = \sum_{{n \choose m} \text{ choices for }b} \frac{1}{{n \choose m}} \frac{1}{m^2} \norm{\sum_{j \in b} \nabla f(\w, z_j)}^2 - \norm{\frac{\sum_{j=1}^n \nabla f(\w, z_j)}{n}}^2
\end{align*}

In the first term, expanding the square and summing over all choices of $b$, we get exactly ${n-1 \choose m-1}$ terms of the form $\norm{\nabla f(\w,\z_j)}^2$ for $j=1$ to $n$, and ${n-2 \choose m-2}$ cross terms of the form $\ip{\nabla f(\w,\z_i)}{\nabla f(\w,\z_j)}$ for $i \neq j$, $i, j=1$ to $n$. Similarly, expanding the second term produces both these kind of terms. Accumulating the coefficients of all the terms, we get
\begin{align*}
     \cV^2 &=\mathbb{E}_b \norm{\frac{\sum_{j \in b} \nabla f(\w, z_j)}{m} -     \mathbb{E}_b \left[\frac{\sum_{j \in b} \nabla f(\w, z_j)}{m} \right]}^2  
     \\&= \br{\frac{{n-1 \choose m-1}}{m^2 {n \choose m}} - \frac{1}{n^2}} \sum_{j=1}^n \norm{\nabla f(\w,\z_j)}^2 + \br{\frac{{n-2 \choose m-2}}{m^2 {n \choose m}} - \frac{1}{n^2}}\sum_{i,j=1, i\neq j}^n\ip{\nabla f(\w,\z_i)}{\nabla f(\w,\z_j)} \\
     & \leq \br{\frac{1}{mn} -\frac{1}{n^2}}nG^2 + \abs{\frac{m-1}{nm(n-1)} -\frac{1}{n^2}}\sum_{i,j=1, i\neq j}^n\norm{\nabla f(\w,\z_i)}\norm{\nabla f(\w,\z_j)} \\
     & \leq \br{\frac{1}{m} -\frac{1}{n}}G^2 + \abs{\frac{m-1}{nm(n-1)} -\frac{1}{n^2}}n(n-1)G^2\\
     & = \br{\frac{1}{m} -\frac{1}{n}}G^2 + \abs{\frac{m-1}{m} -\frac{(n-1)}{n}}G^2 \\
     & = \br{\frac{1}{m} -\frac{1}{n}}G^2  + \abs{\frac{1}{n}-\frac{1}{m}}G^2\\
     & = 2\br{\frac{1}{m} -\frac{1}{n}}G^2  \leq \frac{2G^2}{m}
\end{align*}
where in the first inequality we used Cauchy-Schwartz inequality, and the fact the $G$-Lipschitzness implies the gradient norms are bounded by $G$. Finally, in the second last equality and the last inequality we used the fact that $m\leq n$.

Since the sub-sampled gradients are unbiased, we can use the convergence guarantee of SGD on smooth convex function (see Theorem 4.1 in \cite{allen2018make}) which when using step size $\eta \leq \frac{1}{L}$ gives us
\begin{align*}
   \E{ \hat F_S(\hat \w_S) - \hat F_S(\w^*_S)} \leq O\br{\frac{ \eta\cV^2}{(1-\eta L)} + \frac{D^2}{\eta T}}
\end{align*}
    
Using step size $\eta \leq \frac{1}{2L}$, the right hand side simplifies to $2\eta \cV^2 + \frac{D^2}{\eta T} \leq \frac{4G^2\eta}{m} + \frac{D^2}{\eta T} =  \frac{4G^2T\eta}{\rho n} + \frac{D^2}{\eta T}$, where in the last equality, we substituted $m  = \frac{\rho n}{T}$ to ensure $\rho$ TV-stability. Balancing the trade off in $\eta$ gives us $\eta = \frac{D\sqrt{\rho n}}{GT}$. Therefore setting $\eta = \min\bc{\frac{1}{2L}, \frac{D\sqrt{\rho n}}{GT}}$ gives us 
 \begin{align*}
       \E{ \hat F_S(\hat \w_S) - \hat F_S(\w^*_S)} \leq O\br{\frac{GD}{\sqrt{\rho n}} + \frac{D^2L}{T}}
    \end{align*}
Setting $T = \frac{DL\sqrt{\rho n}}{G}$ achieves the claimed result.
\end{proof}

\begin{proof}[Proof of Proposition \ref{prop:upper-bound-noisy-m-a-sgd}]
We first prove the stability guarantee. For this, we use the R\`enyi-divergence based analysis used in differential privacy literature.
Let  $P$ and $Q$ be probability distributions such that $P$ is absolutely continuous with respect to $Q$ and have densities $\phi_P$ and $\phi_Q$, respectively.
For $\alpha \in (1,\infty)$, the $\alpha$ R\`enyi-divergence between $P$ and $Q$ is defined as follows \cite{renyi1961measures}:
\begin{align*}
    D_\alpha(P\Vert Q) = \frac{1}{\alpha-1} \ln \br{ \int f_P(x)^\alpha f_Q(x)^{1-\alpha} dx}
\end{align*}

Consider two neighbouring datasets $S=\bc{\z_j}_{j}$ and $S'=\bc{\z_j'}_j$ such that $\Delta(S,S')=1$, and let $\bc{b_t'}_{t=1}^T$ and $\bc{\w_t'}_{t=1}^T$ denote the mini-batch indices and iterates of Algorithm \ref{alg:noisy-m-a-sgd} on dataset $S'$ respectively. We look at iteration $t$, and fix all the randomness before $t$ i.e. fix $\w_t $ (and $\w_{t}'$), as well as randomness in sub-sampling mini-batch indices i.e. fix $b_t$. The $\alpha$-R\`enyi Divergence between $\w_{t+1}$ and $\w_{t+1}'$ can be bounded as,

{\begin{small}
\begin{align*}
    D_\alpha(\w_{t+1} \Vert \w_{t+1}') &=  D_\alpha\br{\cP\br{\accw_t - \eta \br{ \frac{\sum_{j \in b_t} \nabla f(\accw_t,\z_j)}{m} +\theta_t}} \Big\Vert \cP\br{\accw_t - \eta \br{ \frac{\sum_{j \in b_t} \nabla f(\accw_t,\z_j')}{m} +\theta_t}}}\\
    & \leq D_\alpha\br{\accw_t - \eta \br{ \frac{\sum_{j \in b_t} \nabla f(\accw_t,\z_j)}{m} +\theta_t} \Big\Vert \accw_t - \eta \br{ \frac{\sum_{j \in b_t} \nabla f(\accw_t,\z_j')}{m} +\theta_t}}\\
    & \leq D_\alpha\br{ \frac{\sum_{j \in b_t} \nabla f(\accw_t,\z_j)}{m} +\theta_t \Big \Vert  \frac{\sum_{j \in b_t} \nabla f(\accw_t,\z_j')}{m} +\theta_t}
    \leq \frac{2\alpha G^2}{m^2 \sigma^2}
\end{align*}
\end{small}}

where in the first and second inequality, we used post-processing property of R\`enyi divergence, and in the last inequality, we use the fact that datasets $S$ and $S'$ differ in at most one sample, therefore $\norm{\frac{\sum_{j \in b_t} \nabla f(\accw_t,\z_j)}{m}  - \frac{\sum_{j \in b_t} \nabla f(\accw_t,\z_j')}{m} }^2 \leq \frac{4G^2}{m^2}$. Hence the divergence is between two multivariate Gaussians of same variance and with the square of the separation of their means  at most $\frac{2\alpha G^2}{m^2\sigma^2}$. Therefore, the inequality follows by using the formula for  R\`enyi divergence between two such multivariate Gaussians.

We now unfix $b_t$, and use the fact the $b_t$ is a uniform sample of $m$ out of $n$ (or $n-1$ or $n+1$) indices. By privacy amplification by sub-sampling result in \citep{balle2018privacy}, for $\alpha\leq 2$, we will argue that the R\`enyi divergence upper bound amplifies to
$\frac{32\alpha G^2}{n^2 \sigma^2}$. There are certain subtleties about the application of this result, so we explain, as follows. The first is that Theorem $9$ stated in \citep{balle2018privacy},
when considering $\alpha\leq 2$, the right hand side 
simplifies as $  \frac{1}{\alpha-1}\log{1+\frac{m^2}{n^2}\frac{\alpha(\alpha-1)}{2}4\br{\exp{\frac{8G^2}{m^2\sigma^2}}-1}} \leq \frac{2\alpha m^2}{n^2}\br{\exp{\frac{8G^2}{m^2\sigma^2}}-1} \leq \frac{32\alpha G^2}{n^2 \sigma^2}$ where the last inequality use the numeric inequality $\exp{x} \leq 1+2x$ when $x\leq 1.256$; this means that we need the following condition $\frac{8G^2}{m^2\sigma^2} \leq 1.256$ - we will revisit this condition later. 
The second point is that Theorem $9$ in \citep{balle2018privacy}
holds integer $\alpha\geq 2$, which only leaves us with $\alpha=2$.
In the subsequent part of the proof, we will need to take $\alpha \rightarrow 1$. This discrepancy can be accounted for by using the fact the $\alpha$-R\`enyi Divergence is non-decreasing for $\alpha \in [0,\infty]$ (see Theorem 3 in \cite{van2014renyi}). Therefore the result holds for all $\alpha \leq 2$, and we can replace the upper bound to be $\frac{64G^2}{n^2\sigma^2}$
The third and final point is that even though the amplification result in \cite{balle2018privacy} is established under the neighbouring relation that one point is replaced between datasets, it can be shown that the same result holds (perhaps upto constants) when the neighbouring relation is add/delete one data-point; see Lemma 3, \cite{abadi2016deep} for example. 
We now use adaptive sequential composition property of R\`enyi divergence (Proposition 1 in \cite{mironov2017renyi}) which linearly accumulates the divergence across iterations, yielding that the R\`enyi divergence between the iterates $(\w_1,\w_2,\ldots,\w_T)$ and $(\w_1',\w_2',\ldots,\w_T')$ is bounded as, $D_\alpha((\w_1,\w_2,\ldots,\w_T) \Vert (\w_1',\w_2',\ldots,\w_T')) \leq \frac{64T G^2}{n^2 \sigma^2}$. An application of data-processing inequality gives us the same upper bound on the R\`enyi divergence between the final iterates $\hat \w_S$ and $\hat \w_S'$.
Moreover, this holds uniformly over all neighbouring datasets $S'$. We now use the result that $\lim_{\alpha \rightarrow 1}D_\alpha(\hat\w_S \Vert \hat\w_S') = D_{\text{KL}}(\hat\w_S \Vert \hat\w_S')$ where $D_{\text{KL}}$ denotes the KL-divergence (see Theorem 5 in \cite{van2014renyi}). Hence we get that $ D_{\text{KL}}(\hat\w_S \Vert \hat\w_S') \leq \frac{64TG^2}{n^2 \sigma^2}$. Finally, we use Pinsker's inequality to further lower bound the left hand side by total variation distance, which yields $\text{TV}(\hat\w_S \Vert \hat\w_S') \leq \sqrt{\frac{ D_{\text{KL}}(\hat\w_S \Vert \hat\w_S')}{2}} \leq \frac{8\sqrt{T}G}{n\sigma}$. As remarked before, this is a uniform bound over all neighbouring datasets. Finally, as before, we trivially have that $\text{TV}(\hat\w_S \Vert \hat\w_S') \leq 1$; therefore setting $\sigma = \frac{8\sqrt{T}G}{n\rho}$ gives us that the algorithm's output is $\min\bc{\rho,1}$ TV-stable.

We now proceed to the accuracy guarantee. This follows simply by guarantee of Accelerated SGD on smooth convex functions. We have already shown in \cref{prop:upper_bound_subsample} that the gradients computed by sub-sampling are unbiased and its variance bounded by $\frac{2G^2}{m}$. The mean-zero Gaussian noise added preserves unbiasedness but the variance is bounded as,
    \begin{align*}
        \cV^2 = \E{\norm{\frac{\sum_{j \in b_t} \nabla f(\accw_t,\z_j)}{m} +\theta_t -\nabla \hat F_S(\accw_t) }}^2 &=
         \E{\norm{\frac{\sum_{j \in b_t} \nabla f(\accw_t,\z_j)}{m} -\nabla \hat F_S(\accw_t) }}^2 + \E{\norm{\theta_t}}^2 \\
         & \leq \frac{2G^2}{m} + \sigma^2d
    \end{align*}

We now use Theorem 2 from \cite{lan2012optimal} - they use notation $\bc{\beta_t}_t$ and $\bc{\gamma_t}_t$ for the step size schedule of Accelerated SGD and set $\beta_t = \frac{t+1}{2}$ and $\gamma_t = \frac{t+1}{2}\gamma$. Even though the updates of their A-SGD seem different than us, it can be verified that they are the same with $\alpha_t = \beta_{t+1}(1-\beta_{t}^{-1}) = \frac{1-t}{t+2}$ with $\alpha_0=0$ and $\eta = \gamma$. 
Finally, using step-size $\eta \leq \frac{1}{2L}$, and appealing to Theorem 2 in \cite{lan2012optimal}, we get,
\begin{align*}
   \E{ \hat F(\hat \w_S) - \hat F(\w^*)} \leq O\br{ T\eta\cV^2 + \frac{D^2}{\eta T^2}} = O\br{\eta T \br{\frac{2G^2}{m} +\sigma^2d} + \frac{D^2}{\eta T^2}}
\end{align*}

Let $\tilde G^2 = \frac{2G^2}{m} +\sigma^2d$, balancing the trade-off in $\eta$ gives us $\eta = \frac{D}{\tilde G T^{3/2}}$. Therefore, setting $\eta = \min\bc{\frac{1}{2L},\frac{D}{\tilde GT^{3/2}}}$ gives us 
\begin{align*}
   \E{ \hat F(\hat \w_S) - \hat F(\w^*)} 
   &\leq O\br{\frac{LD^2}{T^2} + \frac{\tilde GD}{\sqrt{T}}} \leq O\br{\frac{LD^2}{T^2}  + \frac{GD}{\sqrt{Tm}} +  \frac{\sigma\sqrt{d} D}{\sqrt{T}}} \\
  &\leq O\br{\frac{LD^2}{T^2}  + \frac{GD}{\sqrt{Tm}} + \frac{G D\sqrt{d}}{n \rho}} 
\end{align*}

Finally, note that when using the amplification lemma, we arrived at the condition $\frac{8G^2}{m^2\sigma^2} \leq 1.256$. Substituting $\sigma = \frac{8\sqrt{T}G}{n\rho}$, this reduces to $\frac{(n\rho)^2}{8m^2T} \leq 1.256 \iff T \geq \frac{(n\rho)^2}{16m^2}$. 

\end{proof}

\begin{proof}[Proof of \cref{corr: upper-bound-m-A-sgd}]
We start with the result in Proposition \ref{prop:upper-bound-noisy-m-a-sgd}, and balance the two trade-offs:
the first between the terms $\frac{GD}{\sqrt{mT}}$ and  $\frac{GD\sqrt{d}}{\rho n}$, and the second between $\frac{GD\sqrt{d}}{\rho n}$ and  $\frac{LD^2}{T^2}$. 
Note that as long as $\frac{GD}{\sqrt{mT}} \geq \frac{LD^2}{T^2} \iff m \leq \frac{T^3 G^2}{(LD)^2}$, the second term is larger than the first.
Optimizing the trade-off between second and third term gives us
$\frac{GD}{\sqrt{mT}} = \frac{GD\sqrt{d}}{\rho n} \iff T = \frac{(\rho n)^2}{md}$.
Similarly, optimizing the trade-off between the first and third term gives us
$\frac{GD\sqrt{d}}{\rho n} = \frac{LD^2}{T^2} \iff T = \sqrt{\frac{LD(\rho n)}{G\sqrt{d}}}$. Hence setting $T = \max\br{\frac{(\rho n)^2}{md},\sqrt{\frac{LD(\rho n)}{G\sqrt{d}}} }$ yields an expected excess empirical risk of $O\br{\frac{GD\sqrt{d}}{n \rho}}$.

We now look at the given condition $T\geq \frac{(n\rho)^2}{16m^2}$ given in \cref{prop:upper-bound-noisy-m-a-sgd}.
We have set  $T = \max\br{\frac{(\rho n)^2}{md},\sqrt{\frac{LD(\rho n)}{G\sqrt{d}}} }$, there we need to ensure that $\frac{(\rho n)^2}{md} \geq \frac{(\rho n)^2}{16m^2} \iff m \geq \frac{d}{16}$, as well as $\sqrt{\frac{LD(\rho n)}{G\sqrt{d}}} \geq  \frac{(\rho n)^2}{16m^2}  \iff  m \geq \frac{1}{4}\br{\frac{(\rho n)^3 G\sqrt{d}}{LD}}^{1/4}$ - this recovers the condition $m \geq \min \bc{\frac{d}{16}, \frac{1}{4}\br{\frac{(\rho n)^3 G\sqrt{d}}{LD}}^{1/4}}$ in the Proposition statement.
Combining all the above arguments, we get that for any $m \geq \min \bc{\frac{d}{16}, \frac{1}{4}\br{\frac{(\rho n)^3 G\sqrt{d}}{LD}}^{1/4}}$ , setting $T = \max\bc{\frac{(\rho n)^2}{md},\sqrt{\frac{LD(\rho n)}{G\sqrt{d}}}}$, yields an expected excess empirical risk of $O\br{\frac{GD\sqrt{d}}{n \rho}}$.
\end{proof}

\begin{remark}

Note that in the above proof, if we use the stronger variance bound of $2L^2\br{\frac{1}{m}-\frac{1}{n}}$ from sub-sampling (derived in the proof of Proposition \ref{prop:upper_bound_subsample}), we get that when doing full-gradient descent, the variance, as expected is zero, which yields a running time of $T = \sqrt{\frac{LD\rho n}{G\sqrt{d}}}$.
\end{remark}

\begin{proposition}
\label{prop:tv-stability-noisy-sgd}
There exists neighbouring datasets $S$ and $S'$ of $n$ points, and smooth $G$-Lipshcitz convex functions $f$ and constraint set $\cW$ such that the total variation distance between iterates produced by Algorithm \ref{alg:noisy-m-a-sgd} run on datasets $S$ and $S'$, denoted by $\bc{\w_1,\w_2,\ldots, \w_T}$ and $\bc{\w_1',\w_2',\ldots, \w_T'}$ respectively, is bounded as $\text{TV}(\br{\w_1,\w_2,\ldots, \w_T}, \br{\w_1',\w_2',\ldots, \w_T'}) \geq  \min\bc{\Omega\br{\frac{G\sqrt{T}}{n\sigma}},1}$.
\end{proposition}

\begin{proof}[Proof of \cref{prop:tv-stability-noisy-sgd}]
We first prove this without projection - let the constraint set $\cW=\bbR^d$, and so the projection $\cP$ is the identity map. Also, for simplicity, let the initial model be $0$. Consider data sets $S$ and $S'$ such that all points are $0$ but the $n^{\text{th}}$ differing point. Let the $n^{\text{th}}$ point of $S$ be $-G\e_1$ and that of $S'$ be $G\e_1$,  where $\e_1$ is the first canonical basis vector. Let the function $f(\w,\z) = \ip{\w}{\z}$. The gradients are just data points $\z$, therefore gradients are $0$ on all but the differing points, wherein in the differing point in dataset $S$, the gradient is a constant $-G\e_1$ and for dataset $S'$, it is $G\e_1$.
Consider the map $\Psi: (x_1,x_2,\ldots, x_T) \rightarrow x_T$; using data processing inequality and this map, we have that 
\begin{align*}
    \text{TV}(\br{\w_1,\w_2,\ldots, \w_T}, \br{\w_1',\w_2',\ldots, \w_T'}) \geq    \text{TV}(\Psi\br{\w_1,\w_2,\ldots, \w_T}, \Psi\br{\w_1',\w_2',\ldots, \w_T'}) = \text{TV}(\w_T, \w_T')
\end{align*}
We now focus on bounding the total variation distance between the last iterates. Furthermore, by data-processing inequality, we can get rid of the step size scaling, and therefore can consider the last iterates as just the sum of all gradients. By simple calculations, we get that $\w_T$ is a mixture of multivariate Gaussians, all with variance $T\sigma^2 \bbI$ but with varying means: $G\e_1, 2G\e_1, \ldots, TG\e_1$, similarly for $\w_T'$. We denote the mixtures probabilities by $\pi_i$ where the $i^{th}$ conditional distribution, denoted by $\w_T^i$ and $\w_T'^i$ respectively, has means $iG\e_1$ and $-iG\e_1$ respectively. Also, we denote the conditional probability densities of the $i^{th}$ distribution by $\phi_S^i(\w)$ and $\phi_{S'}^i(\w)$ respectively.
We will show that the total variation between these mixtures is expected total variation distance between the mixture components. This follows due the symmetry between these two mixtures, which implies that the set that achieves the total variation distance is $\bc{\w: \w_1 \geq 0}$. 
We can therefore write the total variation distance as,

\begin{align*}
    \text{TV}(\w_T \Vert \w_T') &= \frac{1}{2} \norm{\phi_S(\w) - \phi_{S'}(\w)}_1 = \int_{\w_1\geq 0} \phi_S(\w) -\phi_{S'}(\w)d\w = \int_{\w_1\geq 0} \sum_{i} \pi_i (\phi^i_S(\w) -\phi^j_{S'}(\w))d\w \\
    & =  \sum_{i} \pi_i \int_{\w\geq 0}  (\phi^i_S(\w) -\phi^j_{S'}(\w))d\w = \sum_{i}\pi_i \text{TV} (\w_T^i, \w_T'^i) \\
    & \gtrsim \sum_i \pi_i \frac{2Gi}{m \sqrt{T}\sigma} = \frac{2G}{m\sqrt{T}\sigma}\E{i} =\frac{2G\sqrt{T}}{n\sigma}
\end{align*}

where in the inequality, we use the fact that $\w_T^i$ and $\w_T'^i$ are Gaussians with means separated by $2G$, and variance being $T\sigma^2\bbI$ and use the lower bound result on TV between high-dimensional Gaussians \cite{devroye2018total}. Finally, in the last equality, we compute the Expected value of $i$ under the mixture distribution - recall that $i$ is a sum of $T$ Bernoulli random variables with bias $\frac{m}{n}$, the expectation of which is $\frac{Tm}{n}$.

We now argue why projection doesn't change the above claim. Note the with the projection, all the Gaussians in the mixture are truncated forming a discrete distributions at the boundary of the constraint set. 
The probability mass on either sides of the (original) mean is unchanged. Hence $\bc{\w: \w_1 \geq 0}$ is still the witness set of total variation distance between the mixtures, and the total variation distance in both constrained/unconstrained cases is the same.
The same holds for the total variation between the corresponding mixture components. 
These observations suffices for application of proof of the unconstrained case. Finally, since TV distance, by definition is upper bounded by $1$ - this gives a trivial lower bound of $1$, and hence the TV distance is lower bound by $\min\bc{\Omega\br{\frac{G\sqrt{T}}{n \rho}},1}$.
\end{proof}

%% file: sections/proofsunlearning.tex
\section{Proofs for Section \ref{sec:alg-unlearning}}
\label{sec:proofs-unlearning}
We introduce some notation and setup the roadmap.
In the start of the stream, we have a model trained on the initial dataset of $n$ samples. We then observe an insertion or deletion request. We enumerate the data points from $1$ to $n$, and without loss of generality, assume that the $n^\text{th}$ sample is to be deleted, and the inserted sample has index $n+1$.
We want to show that the unlearning algorithm satisfies exact unlearning at every time point in the stream, and what suffices is to argue that this holds for one edit request, since by mathematical induction it then holds for the entire stream.
For one edit request, we will show the following:  
1. unlearning (deletion/insertion) algorithm is a valid transport, and 
2. the probability of recompute is small, and we will see that together these will imply, that it is a coupling, with large enough measure of the diagonal.

Let $\mu_{n,m}$ denote the sub-sampling probability measure 
to sample $m$ out of $n$ elements uniformly randomly.  
In the deletion and insertion algorithms, we replace \emph{some} mini-batch indices in  \emph{some} iterations - let these operations be denoted by $\textsc{Del}$ and  $\textsc{Ins}$ respectively. 
To elaborate, $\textsc{Del}$ is a (deterministic) map from $\br{[n]^m}^T$ to $\br{[n-1]^m}^T$ and $\textsc{Ins}$ is a map from $\br{[n]^m}^T$ to $\br{[n+1]^m}^T$. 
For an input $b \in ([n]^m)^T$, we have that ${\mathbf{b}} \sim \mu_{n,m}^{\otimes T}$. 
Furthermore, define $\mu_{n,m}^{\text{del} \otimes T} := \textsc{Del}_\# \mu_{n,m}^{\otimes T}$ and $\mu_{n,m}^{\text{ins} \otimes T} := \textsc{Ins}_\# \mu_{n,m}^{\otimes T}$. 
An important observation is that in the unlearning Algorithm \ref{alg:unlearn-subsample-gd}, the sub-sampled indices ${\mathbf{b}}$ are drawn from a product distribution $ \mu_{n,m}^{\otimes T}$ and in each iteration of \cref{alg:unlearn-subsample-gd} or \cref{alg:unlearn-noisy-m-a-sgd}, the maps $\textsc{Del}$ and  $\textsc{Ins}$ act \emph{component-wise} and \emph{symmetrically}.
This implies that $\textsc{Del}({\mathbf{b}}) = [\text{del}(b_1),\text{del}(b_2),\ldots, \text{del}(b_T)]$ where $\text{del}: [n]^m \rightarrow  [n-1]^m$ is the function which describes one iteration of the unlearning algorithm for handling mini-batch indices. 
We similarly have function $\text{ins}: [n]^m \rightarrow   [n+1]^m$ for insertion.
We finally define $\mu_{n,m}^{\text{del}} := \text{del}_\# \mu_{n,m}$ and $\mu_{n,m}^{\text{ins}} := \text{ins}_\# \mu_{n,m}$
- these are the probability measures induced on the sub-sampling indices by deletion and insertion operations, respectively.

\subsection{Unlearning for \emph{sub-sample-GD}}

We first show that $\mu^{\text{del}}_{n,m}$, the probability distribution, induced at a given iteration during deletion, over mini-batch indices $b \in [n]^m$ is a transport.

\begin{claim}[Deletion]
\label{claim:sub-sample-GD-deletion}
For any set $b \in [n]^m$, we have that $\mu^{\text{del}}_{n,m}(b) = \mu_{n-1,m}(b)$
\end{claim}
\begin{proof}

First note that if the verification is unsuccessful, then a recompute is triggered and therein at each iteration, we drawn $b \sim \mu_{n-1,m}$.
Therefore, $\mu^{\text{del}}_{n,m}(b)= \mu_{n-1,m}(b)$ follows trivially.
We now argue for the other case.
The verification is successful if the deleted point was not present in any of iterations, i.e. at any iteration the sub-sample batch $b_t$ doesn't contain the deleted point $\z$. 
The measure $\mu_{n,m}^{\text{del}}$ is therefore just the probability under the original sub-sampling measure $\mu_{n,m}$ conditioned on the event that $\z \not \in b$. 
We therefore have,
\begin{align*}
    \mu_{n,m}^{\text{del}}(b) &= \mu_{n,m}(b|\bc{\z \not \in b}) =\frac{ \mu_{n,m}(b \cap \bc{\z \not \in b}) }{\mu_{n,m}(\bc{\z \not \in b})}
\end{align*}
By direct computation, $\mu_{n,m}(\bc{\z \not \in b}) = 1-\mu_{n,m}(\bc{\z \in b}) = 1 - \frac{{n-1 \choose m-1}}{{n \choose m}} = 1- \frac{m}{n}$. We now look at two choices for $b$. First suppose $z \in b$, then the numerator $\mu_{n,m}(b \cap \bc{\z \not \in b}) = 0$, which gives us that $ \mu_{n,m}^{\text{del}}(b) = 0 = \mu_{n-1,m}(b)$. We now look at a $b$ such that $z \not \in b$. We have,

\begin{align*}
      \mu_{n,m}^{\text{del}}(b) &=\frac{ \mu_{n,m}(b) }{\mu_{n,m}(\bc{\z \not \in b})}=\frac{1/{n \choose m}}{1- m/n} = \frac{n}{n-m}\frac{(n-m)!m!}{n!} \\&=\frac{(n-m-1)!m!}{(n-1)!} = \frac{1}{{n-1 \choose m}} = \mu_{n-1,m}(b)
\end{align*}
\end{proof}

Similarly, for insertion, we show that
$\mu^{\text{ins}}_n$, the probability distribution, induced at a given iteration during insertion, over mini-batch indices $b \in [n]^m$, is a is valid transport.

\begin{claim}
\label{claim:sub-sample-gd-insertion}
For any set $b \in [n+1]^m$, we have that $\mu^{\text{ins}}_{n,m}(b) = \mu_{n+1,m}(b)$
\end{claim}
\begin{proof}[Proof of \cref{claim:sub-sample-gd-insertion}]
Let $\nu$ denote the uniform probability measure over $n+1-m$ elements.
Given $b$, we consider two cases based of whether last/inserted index $n+1$ lies in $b$ or not.
In the first case, we know that the outcome of $\text{Bernoulli}(m/(n+1))$ must have been $1$ i.e. the iteration was selected. Furthermore, in that case, the inserted point would have replaced some other point not in $b$ - the total number of possibilities are $n+1-m$. Let $E_i$ be event that the inserted point replaced the $i^{\text{th}}$ data point, whose index we denote by $s_i$. Note that the events $E_i's$ are disjoint and the event $b$ is $\cup_{i=1}^{n+1-m} E_i$.
Furthermore, $\mu^{\text{ins}}_{n,m}(E_i) = \mu^{\text{ins}}_{n,m}(\text{original subsample is } b\backslash \bc{n+1}\cup \bc{s_i} | \bc{s_i} \text{ replaced}) \mu^{\text{ins}}_{n,m}(\bc{s_i} \text{replaced})) = \mu_{n,m}(b\backslash \bc{n+1}\cup \bc{s_i} | \bc{s_i}) \nu(\bc{s_i}) = \frac{1}{{n \choose m-1}}\frac{1}{(n+1-m)}$.
We therefore have that 

\begin{align*}
\mu^{ins}_{n}(b) &=\frac{m}{n+1} \mu^{ins}_{n}(\cup_{i=1}^{n+1-m} E_i) = \frac{m}{n+1}\sum_{i=1}^{n+1-m}\mu^{ins}_{n}(E_i) = \frac{m}{n+1}\sum_{i=1}^{n+1-m}\frac{1}{{n \choose m-1}}\frac{1}{(n+1-m)}  \\
& = \frac{m}{n+1}\frac{1}{{n \choose m-1}} = \frac{m (m-1)! (n-(m-1))!}{(n+1) n!} = \frac{1}{{n+1 \choose m}} = \mu_{n+1,m}(b)
\end{align*}
In the other case, we know that Bernoulli($m/(n+1)$) resulted in $0$, so there is no replacement. Therefore, we have
\begin{align*}
\mu^{ins}_{n}(b)  = \br{1-\frac{m}{n+1}}\frac{1}{{n \choose m}} = \frac{(n+1-m)(n-m)!m!}{n!(n+1)} = \frac{1}{{n+1 \choose m}} = \mu_{n+1,m}(b)
\end{align*}
\end{proof}

\paragraph{Coupling.} We formally describe the coupling constructed by the unlearning \cref{alg:unlearn-subsample-gd}. We first the discuss deletion case - consider datasets $S$ and $S'$ of sizes $n$ and $n-1$ respectively,  and wlog assume that the last sample of $S$ differs. 
We first sample $\mathbf{b} = [b_1,b_2,\ldots, b_T] \sim \mu_{n,m}^{\otimes T}$.
We set $\mathbf{b}^{(1)}=\mathbf{b}$.
For each $j \in T$,
if $n \in b_j$, then sample $b^{(2)}_j \sim \mu_{n-1,m}$, otherwise set $b^{(2)}_j  = b_j$.
This produces the coupled mini-batches $(\mathbf{b}^{(1)}, \mathbf{b}^{(2)})$ for deletion.

For insertion, we have datasets $S$ and $S'$ of sizes $n$ and $n+1$ respectively, and again assume that the last of point of $S'$ differs. Sample $\mathbf{b} = [b_1,b_2,\ldots, b_T] \sim \mu_{n,m}^{\otimes T}$.
and set $\mathbf{b}^{(1)}=\mathbf{b}$. 
Now sample $\bc{c_{j}}_{j=1}^{T}$, where $c_{j} \sim \text{Bernoulli}\br{\frac{m}{n}}$, if $c_{j}=1$, then sample uniformly a point in $b^{(2)}_j$, and replace it with $n+1$. Otherwise  set $b^{(2)}_j = b_j$, which gives us the coupled mini-batches $(\mathbf{b}^{(1)}, \mathbf{b}^{(2)})$.

It is easy to see that the above procedure is how Algorithm \ref{alg:unlearn-subsample-gd} handles insertions and deletions going from  $S$ to $S'$.
We first show that this is a valid coupling. 

\begin{claim}
\label{claim:subsample-gd-coupling}
For the coupling described above, for any $b$,
\begin{enumerate}
    \item $  \P{\mathbf{b}^{(1)} = b} = \mu_{n}^{\otimes T}(b)$
    \item $\P{\mathbf{b}^{(2)} = b} = \mu_{n-1}^{\otimes T}(b) \ \ \text{(deletion)}$, \ \  $\P{\mathbf{b}^{(2)} = b} = \mu_{n+1}^{\otimes T}(b) \ \ \ \text{(insertion)}$
\end{enumerate}
\end{claim}

\begin{proof}[Proof of \cref{claim:subsample-gd-coupling}]
Follows immediately from Claims \ref{claim:sub-sample-GD-deletion} and  \ref{claim:sub-sample-gd-insertion} .
\end{proof}

We now show that the probability of disagreement under the above coupling is upper bounded by $k$ times TV-stability parameter of \cref{alg:sub-sample-GD}.

\begin{claim}
\label{claim:unlearn_couple_subsample_GD}
For the $\rho$-TV stable \cref{alg:sub-sample-GD},
under the coupling described above, the following holds
\begin{align*}
    \bbP_{(\mathbf{b}^{(1)},\mathbf{b}^{(2)})}[\mathbf{b}^{(1)} \neq \mathbf{b}^{(2)}] \leq \rho
\end{align*}
\end{claim}

\begin{proof}
For deletion, we have,
\begin{align*}
    \mathbb{P}_{(\mathbf{b}^{(1)}, \mathbf{b}^{(2)})}[\mathbf{b}^{(1)} \neq \mathbf{b}^{(2)}] =  \mathbb{P}_{(\mathbf{b}^{(1)}, \mathbf{b}^{(2)})}[ \exists j \in [T]: \mathbf{b}^{(1)}_j \neq \mathbf{b}^{(2)}_j] = \mathbb{P}_{\mathbf{b}}[\exists j \in [T]:n \in \mathbf{b}_j]
\leq \frac{Tm}{n} 
\end{align*}

For insertion, we have
\begin{align*}
    \mathbb{P}_{(\mathbf{b}^{(1)}, \mathbf{b}^{(2)})}[\mathbf{b}^{(1)} \neq  \mathbf{b}^{(2)}] &= \bbP_{{\mathbf{b}}, \mathbf{c}}[ \exists j \in [T] : c_j=1] \leq  \frac{Tm}{n}
\end{align*}

In Proposition \ref{prop:upper_bound_subsample}, we showed that the total variation distance of the algorithm under  change of one point is at most $\frac{Tm}{n} =\rho$, which completes the proof.
\end{proof}

We are now ready to prove \cref{prop:unlearn-subsample-GD}.

\begin{proof}[Proof of \cref{prop:unlearn-subsample-GD}]
The following argument is for deletion, but the insertion case follows similarly.
Consider dataset $S$ and $S'$ of $n$ points and $n-1$ points respectively, differing in one sample.
As in the proof of \cref{prop:upper_bound_subsample}, we embed the randomness for Algorithm \ref{alg:sub-sample-GD} executed on $S$ and $S'$ into a common probability space.
Therefore, similar to the proof of \cref{prop:upper_bound_subsample} given the datasets (and other parameters),
Algorithm \ref{alg:sub-sample-GD}, $\cA(S)$
is a deterministic map from  sub-sampled indices ${\mathbf{b}} = (b_1,b_2,\ldots, \b_T)$ to the model: $\cA(S) :{\mathbf{b}} \rightarrow \cW$, where $b_j \in [n]^m$, for both datasets.
Hence, what suffices is to show that the input probability measure $\mu_{n,m}^{\otimes T}$ is transported to the one that would have been produced on the current dataset $S'$ i.e $\mu_{n-1,m}^{\otimes T}$ - this follows from \cref{claim:subsample-gd-coupling}.
Hence it follows that the output generated by Algorithm \ref{alg:sub-sample-GD} has the same measure as $\cA(S)_\# \mu_{n-1,m}^{\otimes T}$, which proves first part of the claim.
The probability of recompute, being at most $\rho$, for one edit, follows directly from \ref{claim:unlearn_couple_subsample_GD}. Finally, from  \cref{remark:tv_upto_k}, for $k$ edits, and the assumption the number of samples throughout the stream is between  $n/2$ and $2n$, the recompute probability is at most $2k\rho$.
\end{proof}

\subsection{Unlearning for \emph{noisy-m-A-SGD}}

\subsubsection{Coupling mini-batches}
In this section, we show that Algorithm \ref{alg:unlearn-noisy-m-a-sgd} transports sub-sampling probability measures while handling edit requests.
We remind that $\mu^{\text{del}}_{n,m}$ denotes the probability measure induced on the sub-sampled indices by the deletion procedure, in any iteration.
We show that, for any mini-batch, the probability mass of the mini-batched indices under $\mu^{\text{del}}_{n,m}$ is same as that under the sub-sampling measure $\mu_{n-1,m}$.

\begin{claim}
\label{claim:sgd_deletion}
For any set $b \in [n-1]^m$, we have that $\mu^{\text{del}}_{n,m}(b) = \mu_{n-1,m}(b)$
\end{claim}
\begin{proof}[Proof of \cref{claim:sgd_deletion}]
Firstly, note that deletion uses additional randomness which is used to uniformly sample one element from $n-(m-1)$ elements - let $\nu$ denote the uniform probability measure on $n-(m-1)$ elements.
Let $E$ be the event that the $n^{th}$ was sub-sampled originally, and therefore replaced upon verification. By direct computation $\mu_{n,m}(E) = \frac{m}{n}$. We can therefore write $\mu^{\text{del}}_{n,m}(b)$ as follows
\begin{align*}
    \mu^{\text{del}}_{n,m}(b) &= \mu^{\text{del}}_{n,m}(b|E) \mu_{n,m}(E) + \mu^{\text{del}}_{n,m}(b|E^c) \mu_{n,m}(E^c) \\
\end{align*}
Under event $E$, we have the deleted index was replaced. But it can be any element of $b$ that arised out of this replacement. Hence we decompose the event $b|E$ into events $E_i$'s, where $E_i$ corresponds to the event that $b_i$ was replaced. We have that $b|E = \cup_{i=1}^m E_i$, and furthermore, due to the uniform measure, $\mu^{\text{del}}_{n,m}(E_i) = \mu^{\text{del}}_{n,m}(E_j) \forall i,j$. Note that in the event $E_i$, we require that the original sub-sampling measure on $n$ points $\mu_{n,m}$ to have produced the set $b\backslash b_i \cup \bc{n}$ and then a uniform $b_i$ is drawn upon replacement. 
Therefore, $\mu^{\text{del}}_{n,m}(E_i) = \mu_{n,m}(b \backslash b_i \cup \bc{n}) \nu(b_i) = \frac{1}{{n-1 \choose m-1}} \frac{1}{n-1-(m-1)}$. Similarly, when the event $E^c$ occurs, probability of outputting $b$ corresponds to the event when $b$ was generated  using the original sub-sampling measure $\mu_m$ (and no additional randomness used upon verification). Therefore, we get $\mu^{\text{del}}_{n,m}(b|E^c) = \mu_{n,m}(b | E^c) = \frac{1}{{n-1 \choose m}}$. Plugging these in, and with simple calculations, we have
\begin{align*}
    \mu^{\text{del}}_{n,m}(b)  &= \sum_{i=1}^m\mu^{\text{del}}_{n,m}(E_i) \mu_{n,m}(E_i) + \mu^{\text{del}}_{n,m}(b|E^c) \mu_{n,m}(E) \\
    &= \sum_{i=1}^m\mu_{n,m}(b \backslash b_i \cup \bc{n}) \nu(b_i)  \frac{m}{n} + \frac{1}{{n-1 \choose m}} \br{1-\frac{m}{n}} \\
    & =\frac{m}{{n-1 \choose m-1}} \frac{1}{n-1-(m-1)}\frac{m}{n}+ \frac{1}{{n-1 \choose m}} \br{1-\frac{m}{n}}\\
    & = \frac{1}{{n-1 \choose m}} + \frac{m}{n}\br{\frac{m}{{n-1 \choose m-1}(n-1-(m-1))} - \frac{1}{{n-1 \choose m}}} \\
      & = \frac{1}{{n-1 \choose m}} + \frac{m}{n}\br{\frac{m(m-1)!(n-1-(m-1))!}{(n-1)!(n-1-(m-1))} - \frac{1}{{n-1 \choose m}}} \\
      & = \frac{1}{{n-1 \choose m}} + \frac{m}{n}\br{\frac{1}{{n-1 \choose m}} - \frac{1}{{n-1 \choose m}}} = \frac{1}{{n-1 \choose m}} = \mu_{n-1,m}(b)  \qedhere
    \end{align*} 
\end{proof}

Similarly, for insertion, we now show that the probability mass of any mini-batch under $\mu^{\text{ins}}_{n,m}$, the probability measure induced by insertion on the $n+1$ data points, is same as that under $\mu_{n+1,m}$.

\begin{claim}
\label{claim:sgd_insetion}
For any set $b \in [n+1]^m$, we have that $\mu^{\text{ins}}_{n,m}(b) = \mu_{n+1,m}(b)$
\end{claim}
\begin{proof}[Proof of \cref{claim:sgd_insetion}]
Same as that of \cref{claim:sub-sample-gd-insertion}.
\end{proof}

\subsubsection{Lemmas for reflection coupling}
\label{sec:reflection}
We state and prove some results about reflection mapping and couplings.
\begin{lemma}
\label{lem:reflection_coupling_basic}
Let $P$ and $Q$ be probability distributions over $\bbR^d$. 
Let $\psi:\bbR^d\rightarrow\bbR^d$ be a bijection such that $\phi_P(\psi(x))= \phi_Q(x)$, $\phi_P(\psi^{-1}(x)) = \phi_Q(x)$ and  $\abs{\det\left(\frac{d\psi(x)}{dx}\right)}=1$, where $\frac{d\psi(x)}{dx}$ is the Jacobian of the multivariate map $\psi$.
Let $x \sim P$ be a sample from $P$. Let $y=x$ if $\text{Unif}(0,1)\leq \frac{\phi_Q(x)}{\phi_P(x)}$, otherwise $y=\psi(x)$. Then $(x,y)$ is a maximal coupling of $P$ and $Q$.
\end{lemma}
\begin{proof}
We first show that $y$ is a sample from $Q$. Let $E$ be an event in the range of $Q$. Let accept be the event when $u \sim \text{Unif}(0,1)$, $u \leq  \frac{\phi_Q(x)}{\phi_P(x)}$. We have,
\begin{align*}
    \P{y \in E} & = \P{y\in E, \text{accept}} + \P{y\in E, \text{reject}} \\
    & = \mathbb{E}_{x,u}\left[\mathbb{1}\bc{x \in E} \mathbb{1}\bc{u \leq \frac{\phi_Q(x)}{\phi_P(x)}}\right] +  \mathbb{E}_{x,u}\left[\mathbb{1}\bc{\psi(x) \in E} \mathbb{1}\bc{u > \frac{\phi_Q(x)}{\phi_P(x)}}\right] \\
    & = \mathbb{E}_{x}\left[\mathbb{1}\bc{x \in E} \P{\bc{u \leq \frac{\phi_Q(x)}{\phi_P(x)}} \Big\vert x}\right] +  \mathbb{E}_{x}\left[\mathbb{1}\bc{\psi(x) \in E} \P{\bc{u > \frac{\phi_Q(x)}{\phi_P(x)}}\Big \vert x}\right] \\
    & =  \int_{\bbR^d}\mathbb{1}\bc{x \in E} \min \bc{1, \frac{\phi_Q(x)}{\phi_P(x)}}\phi_P(x)dx \\&+  \int_{\bbR^d}\mathbb{1}\bc{\psi(x) \in E} \br{1 - \min \bc{1, \frac{\phi_Q(x)}{\phi_P(x)}}}\phi_P(x)dx \\
    & =  \int_{\bbR^d}\mathbb{1}\bc{x \in E} \min \bc{\phi_P(x), \phi_Q(x)}dx +  \int_{\bbR^d}\mathbb{1}\bc{\psi(x) \in E} \max \bc{0,\phi_P(x) - \phi_Q(x)}dx
\end{align*}
For the second term, we now do change of variable - let $v  = \phi(x)$ - using the given properties of $\psi$, we have $\phi_P(x) = \phi_P(\psi^{-1}(v)) = \phi_Q(v)$ and $\phi_Q(x) = \phi_P(v)$. Furthermore $dv = \abs{\det\left(\frac{d\psi(x)}{dx}\right)}dx = dx$. Finally, we are integrating over $\bbR^d$, and since $\phi$ is a bijection, it can flip the limits of some of the coordinates, however, that is taken into account with using the absolute value of the determinant of the Jacobian.
The second term therefore becomes $\int_{\bbR^d}\mathbb{1}\bc{v \in E} \max \bc{0,\phi_Q(v) - \phi_P(v)}dv$. We now combine the integrands of both the terms, and  substitute $v=x$ as the variable in the second term. This gives us,
\begin{align*}
    \P{y \in E} & = \int_{\bbR^d} \mathbb{1}\bc{x \in E} \br{\min \bc{\phi_P(x), \phi_Q(x)} +  \max \bc{0,\phi_Q(x) - \phi_P(x)}}dx
\end{align*}
Note that for a fixed $x$, if $\phi_P(x) \leq \phi_Q(x)$, the integrand becomes $ \mathbb{1}\bc{x \in E} \br{\phi_P(x) +  \phi_Q(x) - \phi_P(x)} = \mathbb{1}\bc{x \in E} Q(x)$. On the other hand, if $\phi_P(x) > \phi_Q(x)$, the integrand becomes $\mathbb{1}\bc{x \in E}\phi_Q(x)$. Hence, for all cases, we get that,
$$ \P{y \in E} =  \int_{\bbR^d} \mathbb{1}\bc{x \in E} \phi_Q(x)dx = Q(E)$$

We now show that it is a maximal coupling i.e. the probability of accept is $1-\text{TV}(P,Q)$. We have,
\begin{align*}
    \P{\text{accept}} &= \mathbb{E}_{x,u}\left[\mathbb{1}\bc{u \leq \frac{\phi_Q(x)}{\phi_P(x)}}\right] = \int_{\bbR^d} \min\bc{1,\frac{\phi_Q(x)}{\phi_P(x)}}\phi_P(x)dx \\
    & = \int_{\bbR^d} \min\bc{\phi_P(x),\phi_Q(x)}dx  = 1-\text{TV}(P,Q) \qedhere
\end{align*}
\end{proof}

\begin{lemma}
\label{lem:reflection}
Let $P$ and $Q$ be two isotropic probability distributions over $\bbR^d$ with means $\mu_P$ and $\mu_Q$ such that for any vectors $\x,\y$, $\phi_P(\x) = \phi_Q(\y)$ if $\norm{\x-\mu_P} = \norm{\y-\mu_Q}$.
Given vector $\u$ in $\bbR^d$,  the reflection of $\u$ under $(Q,P)$,
$\v = \text{reflect}(\u,\mu_Q,\mu_P) = \mu_Q +(\mu_P -\u)$, 
satisfies:
\begin{enumerate}
  \item Invertibility: $\u = \u_Q + (\mu_P-\v)$
    \item $\phi_Q(\v) = \phi_P(\u)$ and $\phi_P(\v) = \phi_Q(\u)$ 
    \item $\abs{\text{det}\br{\frac{d \ \text{reflect}(\u,\mu_Q,\mu_P) }{d\u}}}=1$
\end{enumerate}

\end{lemma}
\begin{proof}[Proof of Lemma \ref{lem:reflection}]
The proofs follows immediately using the given assumptions.
\end{proof}

\subsubsection{Coupling Markov chains} 
\label{sec:unlearn-noisy-m-a-sgd-appendix}
We setup some notation to describe the coupling that \cref{alg:unlearn-noisy-m-a-sgd} constructs.
The following discussion is for deletion of index $n$, but it can be verified that the arguments naturally extend to the insertion case. 
We remind that $\mu_{n,m}$ denotes the distribution of sampling $m$ elements uniformly randomly from $[n]$, and mini-batches $b_j \sim \mu_{n,m}$.
Furthermore, we will use  $\mathbf{b}_j = [b_1,b_2,\ldots,b_j]$ denote the set of indices upto $j$.
For dataset $S$ and mini-batch indices $b$, let the gradient $\nabla \hat F_S(\w,\z_{b}) := \frac{1}{|b|}\sum_{j\in b} \nabla f(\w,\z_j)$.
Define $\tilde \w_{j+1} := \accw_j - \eta \nabla F_S(\accw_j,\z_{b_j})$,
$\bar \w_{j+1}  := \tilde \w_{j+1} -\eta \theta_j$ and $\w_{j+1} := \cP(\bar \w_{j+1})$. Note that $\tilde \w_j$ is also function of $b_j$ but this dependency is not highlighted for notational simplicity.

The iterates and the mini-batches $[(\bar \w_2,b_1),(\bar \w_3,b_2),\ldots, (\bar \w_{T+1},b_T)]$ produced by Algorithm \ref{alg:noisy-m-a-sgd} is a sample from a $T$-step first order Markov Chain over an uncountable state space $\bbR^d \times [n]^*$.
We remark that $\bar \w_1$ is a constant initialization, and so isn't considered.
Let $P$ be the joint distribution over the $T$ iterates $\times$ mini-batches. The joint density of $P$ can be factored as,
\begin{align*}
    \phi_P((\bar \w_2,b_1),(\bar \w_3,b_2),\ldots, (\bar \w_{T+1},b_T)) = \phi_{P}(\bar \w_2,b_1)\phi_{P}(\bar\w_3,b_2 | \w_2) \ldots \phi_{P}(\bar\w_{T+1},b_T|\w_{T},\w_{T-1})
\end{align*}
where $\phi_P(\bar \w_2,b_1) = \bar \phi_P(\bar \w_2 | b_1)\mu_{n,m}(b_1)$ and $\bar \phi_P(\bar \w_2|b_1)$ is the density of $\cN(\tilde \w_2, \eta^2\sigma^2\bbI)$. 
Similarly, the conditionals $\phi_P(\bar \w_j,b_{j-1}| \w_{j-1},\w_{j-2}) = \bar \phi_P(\bar \w_j | b_{j-1},\w_{j-1},\w_{j-2}) \mu_{n,m}(b_{j-1})$. 
Furthermore, let $\tilde P$ denote the marginal of $[\bar \w_2,\bar \w_2, \ldots, \bar \w_{T+1}]$, the joint density of which can be factored as, 
\begin{align*}
    \tilde \phi_{\tilde P}(\bar \w_2, \bar \w_3, \ldots, \bar \w_{T+1}) = \phi_{\tilde P}(\bar \w_2)\phi_{\tilde P}(\bar\w_3 | \w_2) \ldots \phi_{\tilde P}(\bar\w_{T+1}| \w_{T},\w_{T-1})
\end{align*}
where $\phi_{\tilde P}(\bar \w_2) = \mathbb{E}_{b_1}\phi_P(\bar \w_2,b_1)$, and the conditional $\phi_{\tilde P}(\bar \w_j| \w_{j-1}, \w_{j-2}) = \mathbb{E}_{b_{j-1}} \bar \phi_P(\bar \w_j, b_{j-1} | \w_{j-1},\w_{j-2})$.
Finally, given a fixed mini-batch sequence $\mathbf{b}=\bc{b_1,b_2,\ldots,b_T}$, let $P_{\mathbf{b}}$ denote the joint conditional distribution of $\bc{\bar \w_2,\bar \w_3, \ldots, \bar \w_{T+1}}$ given $\mathbf{b}$. In this case, $P_{\mathbf{b}}$ factorizes as:
\begin{align*}
    \phi_{P_{\mathbf{b}}}(\bar \w_2, \bar \w_3, \ldots, \bar \w_{T+1}) = \phi_{P_{b_1}}(\bar \w_2)\phi_{ P_{b_2}}(\bar\w_3 | \w_2) \ldots \phi_{ P_{b_T}}(\bar\w_{T+1}| \w_{T},\w_{T-1})
\end{align*}
where $\phi_{P_{b_1}}(\bar \w_2) = \bar \phi_P(\w_2 | b_1)$ and $\phi_{ P_{b_{j-1}}}(\bar\w_j | \w_{j-1},\w_{j-2}) = \bar \phi_P(\bar \w_j | b_{j-1}, \w_{j-1},\w_{j-2})$.
We similarly have a Markov Chain to generate the iterates for dataset $S'$ - call this joint distribution over iterates and mini-batches as $Q$, the marginals over iterates as $\tilde Q$ and for a given $\mathbf{b} \sim \bc{\mu_{n-1,m}}^{\otimes T}$, the conditionals over the iterates as $Q_{\mathbf{b}}$.

\begin{figure}[!htpb]
    \centering
    \begin{tikzpicture}
    \node[box,draw=white!100] (Latent) {\textbf{Mini-batches}};
    \node (L0) [right=of Latent] {};
    \node[main] (L1) [right=of L0] {$b_1$};
    \node[main] (L2) [right=of L1] {$b_2$};
    \node[main] (L3) [right=of L2] {$b_3$};
    \node[main] (Lt) [right=of L3] {$b_T$};
    \node[main,fill=black!10] (O1) [below=of L1] {$\bar \w_2$};
    \node[main,fill=black!10] (O0) [left=of O1] {$\bar \w_1$};
    \node[main,fill=black!10] (O2) [below=of L2] {$\bar \w_3$};
    \node[main,fill=black!10] (O3) [below=of L3] {$\bar \w_4$};
    \node[main,fill=black!10,scale=0.69] (Ot) [below=of Lt] {$\mathbf{\bar w}_{T+1}$};
    \node at (8.5, -2.5)   (c0) {};
    \node at (8.6, -1.9)   (c1) {};
  \node[box,draw=white!100,left=of O0] (Observed) {\textbf{Iterates}};
  \path (L3) -- node[auto=false]{\ldots} (Lt);
  \path (O0) edge [connect] (O1)
        (O0) edge[bend right] [connect] (O2)
        (O1) edge [connect] (O2)
        (O1) edge[bend right] [connect] (O3)
        (O2) edge [connect] (O3)
        (c0) edge[bend right]  [connect] (Ot)
        (c1) edge [connect] (Ot)
        (O3) -- node[auto=false]{\ldots} (Ot);
  \path (L1) edge [connect] (O1);
  \path (L2) edge [connect] (O2);
  \path (L3) edge [connect] (O3);
  \path (Lt) edge [connect] (Ot);

  \draw [dashed, shorten >=-1cm, shorten <=-1cm]
      ($(Latent)!0.5!(Observed)$) coordinate (a) -- ($(Lt)!(a)!(Ot)$);
\end{tikzpicture}
   \caption{Markov chain for \emph{noisy-m-A-SGD} Algorithm}
    \label{fig:markov_chain_noisy_m_A_SGD}
\end{figure}
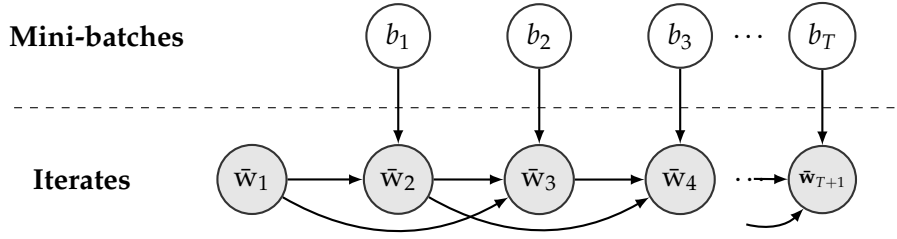

We now describe how the unlearning Algorithm \ref{alg:unlearn-noisy-m-a-sgd} constructs a coupling between $P$ and $Q$ to generate $(\mathbf{\bar w}^{(1)},\mathbf{\bar w}^{(2)})$. We first describe the coupling of mini-batch indices.
Sample $\mathbf{b} \sim (\mu^{m}_{n})^{\otimes T}$, let $\mathbf{b}^{(1)} = \mathbf{b}$. 
We now look at all $b_j^{(1)} \in \mathbf{b}^{(1)}$: if $n \not \in b_j^{(1)}$, then let $b_j^{(2)} = b_j^{(1)}$, otherwise for each such $b_j^{(1)}$, we replace $n$ by randomly sampling an index from $[n] \backslash b_j^{(1)}$, and call this $b_j^{(2)}$. We then define the ordered set $\mathbf{b}^{(2)} = \bc{b_j^{(2)}}_{j=1}^T$.
From \cref{claim:sgd_deletion}, this is a valid coupling of mini-batch indices.
Sample $\mathbf{\bar w} = [\bar \w_2,\bar \w_3,\ldots, \bar \w_{T+1}] \sim P_{\mathbf{b}^{(1)}}$, which corresponds to training with Algorithm \ref{alg:noisy-m-a-sgd} on dataset $S$. 
Set $\mathbf{\bar w}^{(1)}:=\mathbf{\bar w}$. 
To generate $\mathbf{\bar w}^{(2)}$, we do rejection sampling steps at each iteration. At the first step, we sample $u_1\sim  \text{Unif}(0,1)$, and check if $u_1 \leq \frac{\phi_{Q_{\mathbf{b}^{(2)}}}(\bar \w_2)}{\phi_{P_{\mathbf{b}^{(1)}}}(\bar \w_2)}$. If the step succeeds, then we proceed to the second iteration, wherein we again do a step of rejection sampling with ratio of conditional densities and so on. However, if anyone of the rejection sampling step fails, lets say the $t^{\text{th}}$ step, then we do a \emph{reflection} of iterate $\bar \w_{t+1}$ about the mid-point of the means of $P_{\mathbf{b}^{(1)}}(\cdot|\w_{t},\w_{t-1})$ and $Q_{\mathbf{b}^{(2)}}(\cdot|\w_{t},\w_{t-1})$ , which are  $\tilde \w^{(1)}_{t+1} = \w_{t} - \eta \nabla F_{S}(\w_{t},\z_{b_t})$ and $\tilde \w^{(2)}_{t+1} = \w_{t} - \eta \nabla F_{S'}(\w_{t},\z_{b_t'})$ respectively.
Set $\bar \w^{(2)}_{t+1} = \text{reflect}(\bar \w_{t+1},\tilde \w^{(2)}_{t+1},\tilde \w^{(1)}_{t+1})$. 
After the reflection, we continue training on dataset $S'$ which corresponds to continue sampling from the $(t+1)^{\text{th}}$ step of the Markov chain for $Q_{\mathbf{b}^{(2)}}$ conditioned on the $t^{\text{th}}$ sample being $\bar \w^{(2)}_{t}$.
This generates the random variables $\mathbf{\bar w}^{(1)}$ and $\mathbf{\bar w}^{(2)}$.

We now show that this is indeed a coupling.

\begin{lemma}
\label{lem:m-sgd_coupling}
For any measurable set $E \subseteq \bbR^{dT}$, $\P{\mathbf{\bar w}^{(2)} \in E} = \tilde Q(E)$
\end{lemma}
\begin{proof}[Proof of \cref{lem:m-sgd_coupling}]
We will first show that $\P{\mathbf{\bar w}^{(2)}\in E \vert  (\mathbf{b}^{(1)}, \mathbf{b}^{(2)})} = Q_{\mathbf{b}^{(2)}}(E)$. The proof is based on induction on the length of the Markov chain $T$. 
Define $\mathbf{\bar w}_T^{(2)} := \bc{\bar \w_2^{(2)},\cdots, \bar \w_{T-1}^{(2)}}$.
The key to the proof is the observation that the marginals ${P_{b_1^{(1)}}}(\cdot)$ and ${Q_{b_1^{(2)}}}(\cdot)$ are Gaussian $\cN(\tilde \w_2^{(1)},\eta^2\sigma^2\bbI)$ and $\cN(\tilde \w_2^{(2)},\eta^2\sigma^2\bbI)$ respectively, and the conditionals ${P_{b_j^{(1)}}}(\cdot | \w_{j})$ and ${Q_{b_j^{(2)}}(\cdot| \w_{j})}$ are also Gaussian $\cN(\tilde \w_{j+1}^{(1)},\eta^2\sigma^2\bbI)$ and $\cN(\tilde \w_{j+1}^{(2)},\eta^2\sigma^2\bbI)$. 

For $T=1$, we only care about the marginals ${P_{b_1^{(1)}}}(\cdot)$ and ${Q_{b_1^{(2)}}}(\cdot)$, which as argued before, are normally distributed.
From Lemma \ref{lem:reflection}, we have established that the reflection map satisfies the conditions in Lemma \ref{lem:reflection_coupling_basic}.
Combining these,
we have that the base case $T=1$ follows from the reflection coupling result stated as Lemma \ref{lem:reflection_coupling_basic}.

We proceed to the induction step.
There are two cases, depending on whether we do a rejection sampling in the $T^{th}$ step or not: we call these "rej-sample" and "no-rej-sample" respectively. If we do a rejection sampling, we further have two cases (1a). accept: either all rejection samplings, including the one in the $T^{\text{th}}$ step are accepts, (1b). reflect: all rejection samplings, except the one in the $T^{\text{th}}$ step are accepts, and in the $T^{\text{th}}$ step, we reflect. Finally, if we don't do a rejection sampling step, we have the third case (2). reject: some rejection sampling prior to $T$ results in reject; in this case, the $T^{\text{th}}$ sample $\bar\w_{T+1}^{(2)} \sim Q_{b_T^{(2)}}(\cdot | \w_{T}^{(2)}, \w_{T-1}^{(2)})$. 
Cases (1) and (2) partition the whole event space for $T$ draws, whereas cases (1a) and cases (1b) partitions the space of the $T^{\text{th}}$ draw, conditioned on the first event. Also note that case (1) vs (2) distinction is measurable w.r.t. the natural filtration generated by the Markov chain upto $T-1$ draws.

Note that conditioned on the events "rej sample" as well as $\mathbf{w}^{(2)}_{T-1}$, the last step is just a one-step reflection coupling method. To elaborate, the conditionals $Q_{b_T^{(2)}} (\cdot | \w_{T}^{(2)},\w_{T-1}^{(2)})$ and $P_{b_T^{(2)}} (\cdot | \w_{T}^{(2)},\w_{T-1}^{(2)})$ used in the $T^{\text{th}}$ rejection sampling are Gaussians, which along with the reflection map satisfies properties of Lemma \ref{lem:reflection_coupling_basic}, as in the base case. 
Let $E_{last} = \bc{y | \exists \mathbf{x}: (\mathbf{x},y) \in E}$ be the projection of $E$ on the last co-ordinate and $E_{y} = \bc{ \mathbf{x}| (\mathbf{x},y) \in E}$.
According to \cref{lem:reflection}, the conditional distribution of $\w_{T+1}^{(2)}$ is $Q_{b_T^{(2)}} (\cdot | \w_{T}^{(2)},\w_{T-1}^{(2)})$:
\begin{align*}
     &\P{\bar \w_{T+1}^{(2)}\in E_{last} \vert \text{rej-sample}, { \mathbf{\bar w}_{T-1}^{(2)}}, (\mathbf{b}^{(1)},\mathbf{b}^{(2)})} \\
     &=  \P{\bar \w_{T+1}^{(2)}\in E_{last}, \text{accept} \vert \text{rej-sample}, { \mathbf{\bar w}_{T-1}^{(2)}}, (\mathbf{b}^{(1)},\mathbf{b}^{(2)})}  \\
      &\qquad +\P{\bar \w_{T+1}^{(2)}\in E_{last}, \text{reflect} \vert \text{rej-sample}, { \mathbf{\bar w}_{T-1}^{(2)}}, (\mathbf{b}^{(1)},\mathbf{b}^{(2)})} \\
      &= \Large\int_{\bbR^{d}} \mathbb{1}\br{\bar \w_{T+1}^{(2)}\in E_{last}}\phi_{Q_{b_T^{(2)}}}(\bar \w_{T+1}^{(2)}| \w_{T}^{(2)},  \w_{T-1}^{(2)})
    d\bar \w_{T+1}^{(2)}
\end{align*}

For the "no-rej-sample" case, we have:
\begin{align*}
    &\P{\bar \w_{T+1}^{(2)}\in E_{last}\vert \text{no-rej-sample}, { \mathbf{\bar w}_{T-1}^{(2)}}, (\mathbf{b}^{(1)},\mathbf{b}^{(2)})} =
    \underset{\bar \w_{T+1}^{(2)}}{\mathbb{E}}  \left[ \mathbb{1}\br{\bar \w_{T+1}^{(2)}\in E_{last}}\right]\\
      & = \Large\int_{\bbR^{d}} \mathbb{1}\br{\bar \w_{T+1}^{(2)}\in E_{last}}  \phi_{Q_{b_T^{(2)}}}(\bar \w_{T+1}^{(2)}|\w_{T}^{(2)},\w_{T-1}^{(2)}) d\bar \w_{T+1}^{(2)}\\
\end{align*}

We will now combine the two cases. Let $\phi^{(r)}(\cdot)$ and $\phi^{(nr)}(\cdot)$ denote the densities of $\mathbf{\bar w}_{T-1}^{(2)}$ under the "rej-sample" and "no-rej-sample" events respectively.
\begin{align*}
    &\P{\mathbf{\bar w}^{(2)} \in E | (\mathbf{b}^{(1)},\mathbf{b}^{(2)})} 
    \\&=     \P{\bar \w_{T+1}^{(2)}\in E_{last} \, \Big\vert \, \text{rej-sample}, { \mathbf{\bar w}_{T-1}^{(2)}}, (\mathbf{b}^{(1)},\mathbf{b}^{(2)})} \P{\mathbf{\bar w}_{T-1}^{(2)} \in E_{\bar \w_{T+1}^{(2)}},\text{rej-sample}}\\
    &+ \P{\bar \w_{T+1}^{(2)}\in E_{last} \, \Big\vert \, \text{no-rej-sample}, { \mathbf{\bar w}_{T-1}^{(2)}}, (\mathbf{b}^{(1)},\mathbf{b}^{(2)})} \P{\mathbf{\bar w}_{T-1}^{(2)} \in E_{\bar \w_{T+1}^{(2)}},\text{no-rej-sample}}\\
    & = \int_{\bbR^{dT}}\mathbb{1}\bc{\bar \w_{T+1}^{(2)}\in E_{last}}\mathbb{1}\bc{\mathbf{\bar w}_{T-1}^{(2)} \in E_{\bar \w_{T+1}^{(2)}}} \\&\cdot\bc{\mathbb{1}_\text{rej-sample}\bc{\mathbf{\bar w}_{T-1}^{(2)}} \phi^{(r)}(\mathbf{\bar w}_{T-1}^{(2)}) + \mathbb{1}_{\text{no-rej-sample}}\bc{\mathbf{\bar w}_{T-1}^{(2)}}\phi^{(nr)}(\mathbf{\bar w}_{T-1}^{(2)})}
    \phi_{Q_{{b_T^{(2)}}}}(\bar \w_{T+1}^{(2)}|\w_{T}^{(2)},\w_{T-1}^{(2)})  d \mathbf{\bar w}_{T}^{(2)}\\
    & = \int_{\bbR^{dT}}\mathbb{1}\bc{\mathbf{\bar w}^{(2)} \in E}
    \phi_{Q_{\mathbf{b}_{T-1}^{(2)}}}(\mathbf{\bar w}_{T-1}^{(2)})
    \phi_{Q_{{b^{(2)}_T}}}(\bar \w_{T+1}^{(2)}| \w_{T}^{(2)},\w_{T-1}^{(2)})  d \mathbf{\bar w}_{T}^{(2)}\\
        & = \int_{\bbR^{dT}}\mathbb{1}\bc{\mathbf{\bar w}^{(2)} \in E}
    \phi_{Q_{\mathbf{b}^{(2)}}}(\mathbf{w}_{T}^{(2)})
    d \mathbf{\bar w}_{T}^{(2)} = Q_{\mathbf{b}^{(2)}}( E)
\end{align*}
where the third equality uses the induction hypothesis that $\mathbf{\bar w}_{T-1}^{(2)}$, conditioned on $\mathbf{b}^{(1)}$ and $\mathbf{b}^{(2)}$, is distributed as $Q_{\mathbf{b}_{T-1}^{(2)}}$.
Finally, we integrate with respect to the coupling generating $ (\mathbf{b}^{(1)},\mathbf{b}^{(2)})$; we get

\begin{align*}
    \P{\mathbf{\bar w}^{(2)} \in E} &=  \sum_{(\mathbf{b}^{(1)},\mathbf{b}^{(2)})}\P{\mathbf{\bar w}^{(2)} \in E | (\mathbf{b}^{(1)},\mathbf{b}^{(2)})} \P{\mathbf{b}^{(1)},\mathbf{b}^{(2)}}\\
    & =\sum_{(\mathbf{b}^{(1)},\mathbf{b}^{(2)})}Q_{\mathbf{b}^{(2)}}( E) \P{\mathbf{b}^{(1)},\mathbf{b}^{(2)}} = \sum_{\mathbf{b}^{(2)}}Q_{\mathbf{b}^{(2)}}( E) \P{\mathbf{b}^{(2)}} = \tilde Q(E)
\end{align*}

This completes the proof.
\end{proof}
We now show that not only the marginals over the iterates, but the entire state maintained by the algorithm, which includes the mini-batching indices is transported.
\begin{lemma}
\label{lem:noisy-m-sgd-transport}
For any measurable event in $\br{\bbR^d \times [n]^m}^{\otimes T}$, we have
\begin{align*}
    \P{(\mathbf{\bar w}^{(2)}, \mathbf{b}^{(2)}) \in E} = Q(E)
\end{align*}
\end{lemma}
\begin{proof}[Proof of \cref{lem:noisy-m-sgd-transport}]
We first decompose the event $E \subseteq \br{\bbR^d \times [n]^m}^{\otimes T}$ as two events, $E=E_1 \times E_2$ where $E_1 \subseteq \bbR^{dT}$ and $E_2 \subseteq \br{[n]^m}^T$. We have
\begin{align*}
    \P{(\mathbf{\mathbf{\bar w}}^{(2)}, \mathbf{b}^{(2)}) \in E} &=  
    \mathbb{E}_{\mathbf{b}^{(1)}}
    \P{\mathbf{\mathbf{\bar w}}^{(2)} \in E_1 |\mathbf{b}^{(1)},  \mathbf{b}^{(2)} } \P{\mathbf{b}^{(2)} \in E_2} \\
     & = \mathbb{E}_{\mathbf{b}^{(1)}} \tilde Q_{\mathbf{b}^{(2)}}(E_1)\mu_{n,m}^{\text{del}\otimes T}(E_2) \\
     & = \mathbb{E}_{\mathbf{b}^{(1)}} \tilde Q_{\mathbf{b}^{(2)}}(E_1)\mu_{n-1,m}^{\otimes T}(E_2) \\
     & = \tilde Q_{\mathbf{b}^{(2)}}(E_1)\mu_{n-1,m}^{\otimes T}(E_2) = Q(E)
\end{align*}

where the second and third equality follows from \cref{lem:m-sgd_coupling} and \cref{claim:sgd_deletion}, and the final equality follows from the definition of event $E$ and probability distribution $Q$.
\end{proof}

We now lower bound the probability of accepting at all rejection sampling steps.

\begin{lemma}
Let ``\text{accept}" be the event in which all rejection sampling result in accepts so there is no reflection or recompute. The probability of accept is lower bounded as,
\label{lem:noisy-m-sgd-accept}
$$\P{\text{accept}} \geq 1 - \frac{\sqrt{T}\rho}{8}$$
\end{lemma}
\begin{proof}[Proof of \cref{lem:noisy-m-sgd-accept}]
We evaluate the probability that all rejection sampling steps result in accepts. We first do it conditioned on $\mathbf{b}^{(1)},\mathbf{b}^{(2)}$
\begin{align*}
    \P{\text{accept}| \mathbf{b}^{(1)},\mathbf{b}^{(2)}} &=   \underset{\bar \w_T^{(1)}, \bar \w_{T+1}^{(2)}, \bc{u_j}_{j=1}^T}{\mathbb{E}}  \left[ \prod_{j=1}^T\mathbb{1}_{\text{accept}}(u_j)\right]\\
    & = \underset{\mathbf{\bar w}_T^{(1)}, \mathbf{\bar w}_{T+1}^{(2)}}{\mathbb{E}}  \left[ \prod_{j=1}^T \P{u_j \leq \frac{\phi_{Q_{b^{(2)}_j}} (\bar \w_{j+1}^{(1)} | \w_{j}^{(1)}, \w_{j-1}^{(1)}) }{\phi_{P_{b^{(1)}_j}} (\bar \w_{j+1}^{(1)} | \w_{j}^{(1)}, \w_{j-1}^{(1)})  } \Big \vert \mathbf{\bar w}_T^{(1)}, \mathbf{\bar w}_T^{(2)}}\right]\\
    & = \int_{\bbR^{dT}}   \prod_{j=1}^T \min\bc{\phi_{P_{b^{(1)}_j}} (\bar \w_{j+1}^{(1)} | \w_{j}^{(1)})  \w_{j-1}^{(1)}) , \phi_{Q_{b^{(2)}_j}} (\bar \w_{j+1}^{(1)} |  \w_{j}^{(1)}), \w_{j-1}^{(1)})} d \mathbf{\bar w}^{(1)}_{T} \\
    &  = \prod_{j=1}^T \br{\int_{\bbR^{d}}  \min\bc{\phi_{P_{b^{(1)}_j}} (\bar \w_{j+1}^{(1)} |  \w_{j}^{(1)}), \w_{j-1}^{(1)}) , \phi_{Q_{b^{(2)}_j}} (\bar \w_{j+1}^{(1)} |  \w_{j}^{(1)}), \w_{j-1}^{(1)})} d \bar \w^{(1)}_{j+1} }\\
\end{align*}

The term 
\begin{align*}
&\int_{\bbR^{d}}  \min\bc{\phi_{P_{b^{(1)}_j}} (\bar \w_{j+1}^{(1)} |  \w_{j}^{(1)}), \w_{j-1}^{(1)}) , \phi_{Q_{b^{(2)}_j}} (\bar \w_{j+1}^{(1)} |  \w_{j}^{(1)}), \w_{j-1}^{(1)})} d \bar \w^{(1)}_{j+1} \\
&= 1 -\text{TV}_{\mathbf{\bar w}_{j}^{(1)}, \mathbf{b}^{(1)},\mathbf{b}^{(2)} }(P_{b_j^{(1)}},Q_{b_j^{(2)}})
\end{align*}
where the notation $\text{TV}_{x}(\cdot, \cdot)$ denotes the conditional TV between the arguments, conditioned on the subscript.
Let $\gamma_j= \frac{\Delta(b_j^{(1)},b_j^{(2)})}{2}$ i.e. the number of elements differing in $b_j^{(1)}$ and $b_j^{(2)}$ . Note that if $\gamma_j = 0$, then $b_j^{(1)}= b_j^{(2)}$ and $P_{b_j^{(1)}} = Q_{b_j^{(2)}}$, and hence $\text{TV}_{\bar \w_{j-1}^{(1)}, \mathbf{b}^{(1)},\mathbf{b}^{(2)} }(P_{b_j^{(1)}},Q_{b_j^{(2)}}) = 0$.  In the other case, $\gamma_j = 1$, which corresponds to the case when the deleted point was used in the $j^{\text{th}}$ mini-batch. In this case, the means of $P_{b_j^{(1)}}$ are $Q_{b_j^{(2)}}$ at separated by at most $\frac{2G\eta}{m}$ - this follows as in the proof of Proposition \ref{prop:upper-bound-noisy-m-a-sgd}. In particular, fixing previous iterates and $\bar \w_{j-1}^{(1)}$ and mini-batch indices $\mathbf{b}^{(1)},\mathbf{b}^{(2)}$, using the fact that gradients are in norm bounded by $G$, $P_{b_j^{(1)}}$ and $ Q_{b_j^{(2)}}$ are Gaussians with variance $\eta^2\sigma^2\bbI$
and means separated by either $\frac{2G\eta\gamma_j}{m}$ or $0$, depending on $\gamma_j$. Combining the two cases, and using TV between Gaussians formula \cite{devroye2018total}, we have $1- \text{TV}_{\bar \w_{j-1}^{(1)}, \mathbf{b}^{(1)},\mathbf{b}^{(2)} }(P_{b_j^{(1)}},Q_{b_j^{(2)}}) \geq \br{ 1- \frac{G\eta}{\eta\sigma  m}}^{\gamma_j} = \br{ 1- \frac{G}{\sigma  m}}^{\gamma_j}$. We therefore get $\int_{\bbR^{d}}  \min\bc{\phi_{P_{b^{(1)}_j}} (\bar \w_{j+1}^{(1)} |  \w_{j}^{(1)}), \w_{j-1}^{(1)}) , \phi_{Q_{b^{(2)}_j}} (\bar \w_{j+1}^{(1)} |  \w_{j}^{(1)}), \w_{j-1}^{(1)})} d \bar \w^{(1)}_{j+1}  \geq \br{ 1- \frac{G}{\sigma m}}^{\gamma_j}$. Plugging this in the conditional probability of accept expression, we get

\begin{align*}
    \P{\text{accept}| \mathbf{b}^{(1)},\mathbf{b}^{(2)}} &\geq  \prod_{j=1}^T \br{ 1- \frac{G}{\sigma m}}^{\gamma_j} =  \br{1- \frac{G}{\sigma m}}^{\sum_{j=1}^T \gamma_j} \geq 1 - \frac{G\sum_{j=1}^T \gamma_j}{\sigma m}
\end{align*}

We now integrate with respect to $\mathbf{b}^{(1)},\mathbf{b}^{(2)}$. Note that $\sum_{j=1}^T \gamma_j$ is the number of mini-batches which contain the deleted point. Since in each mini-batch, $m$ points are selected uniformly randomly from $n$,  $\mathbb{E}_{\mathbf{b}^{(1)},\mathbf{b}^{(2)}}{\gamma_j} = \frac{m}{n}$, which gives us $\mathbb{E}_{\mathbf{b}^{(1)},\mathbf{b}^{(2)}}\sum_{j=1}^T \gamma_j = \frac{Tm}{n}$. Hence,
\begin{align*}
    \P{\text{accept}} \geq 1 - \frac{G \mathbb{E}_{\mathbf{b}^{(1)},\mathbf{b}^{(2)}}\sum_{j=1}^T \gamma_j}{\sigma m} = 1- \frac{GT}{\sigma n} = 1 -\frac{\sqrt{T}\rho}{8}
\end{align*}
where the last equality follows from plugging in $\sigma = \frac{8\sqrt{T}G}{n \rho}$ as in \cref{prop:upper-bound-noisy-m-a-sgd}.
\end{proof}

We are now ready to prove \cref{prop:unlearn-noisy-m-A-SGD}.

\begin{proof}[Proof of  \cref{prop:unlearn-noisy-m-A-SGD}]
We need to show that upon deletion and insertion, the probability distribution of the entire state maintained by the algorithm, which is all iterates as well as mini-batches indices is transported - this, for one deletion, follows from Lemma \ref{lem:noisy-m-sgd-transport} (which is for unprojected iterates), together with the fact that projection is a deterministic operation.
Moreover, as before, the above argument also holds for insertion and generalizes arbitrary edit requests.

We now proceed to bound the probability to recompute.
This follows directly combining \cref{lem:noisy-m-sgd-accept}, \cref{prop:upper-bound-noisy-m-a-sgd} and Remark \ref{remark:tv_upto_k}.
From Remark \ref{remark:tv_upto_k}, upon $k$ edits, the total variation distance is at most $k$ times total variation distance upon 1 edit. 
Since the algorithm is $\rho$-TV stable (Proposition \ref{prop:upper-bound-noisy-m-a-sgd}), and the assumption that the number of samples are between $n/2$ and $2n$, the total variation distance upon $k$ edits is at most $2k\rho$. 
Hence, using \cref{lem:noisy-m-sgd-accept}, we have that the probability to recompute is probability of ``reject"  is at most $\frac{k\rho\sqrt{T}}{4}$.
\end{proof}

%% file: sections/runtime-with-proofs.tex
\section{Runtime and space complexity}
In this section, we discuss, in detail, the learning and unlearning runtime of the algorithms, as well as their space complexity. 
\label{sec:runtime}

\input{sections/subsections/runtime-learning}
\input{sections/subsections/runtime-unlearning-with-proof}
\input{sections/subsections/runtime-space}

%% file: sections/subsections/runtime-learning.tex
\subsection{Learning runtime}
\label{sec:learning-runtime}

In this work, we did not aim to carefully optimize the runtime for training/learning algorithm, as long as the algorithm achieves the rate in \cref{thm:main-upper-bound}.
However, we briefly discuss the runtime of each algorithm, and highlight easy improvements, where possible. Algorithm \ref{alg:sub-sample-GD} requires $mT$ = $\rho n$ stochastic gradient computations. 
On the other hand, for Algorithm \ref{alg:noisy-m-a-sgd}, if $m\leq \sqrt{\frac{LD \rho n}{G \sqrt{d}}}$, it requires $mT = m\br{\frac{\rho n}{md}} =  \frac{(\rho n)^2}{d}$ stochastic gradient computations - setting larger $m$ only hurts the total runtime, without any advantage. Note that total stochastic gradient descent computations of \emph{noisy-m-SGD} (i.e. without acceleration, see \cref{sec:noisy-m-sgd}) is also $\frac{(\rho n)^2}{d}$; however, the key advantage of acceleration is that it allows setting larger mini-batch sizes: $T^3$ as opposed to $T$, which leads to smaller number of iterations: $\sqrt{\frac{\rho n}{\sqrt{d}}}$ as opposed to $\frac{\rho n}{\sqrt{d}}$ and hence a smaller probability of recompute.
From \cite{woodworth2016tight}, we know that mini-batch SGD (with or without acceleration) is optimal for smooth convex composite/ERM optimization in the low accuracy regime: when accuracy $\alpha \gtrsim \frac{1}{\sqrt{n}}$. In this regime, an algorithm makes at least $\Omega\br{\frac{1}{\alpha^2}}$ calls to a stochastic gradient oracle.
It can be then verified that for our accuracy, our algorithms make the optimal number of oracle calls.

For Algorithm \ref{alg:noisy-m-a-sgd}, as discussed,  
faster algorithms lead to better unlearning times. 
It is natural to ask what happens if we additionally introduce variance reduction techniques on top of acceleration to yield even faster runtimes.
In particular, what if we use Katyusha \cite{allen2017katyusha}, which has optimal runtime in terms of stochastic gradient computations. 
We argue that even though it improves the runtime of the learning algorithm, it does not yield improvement for unlearning beyond what we have with acceleration. 
From Corollary 5.8 in \cite{allen2017katyusha}, setting largest allowed $m=\sqrt{n}$, we get that $Tm = O\br{\frac{\sqrt{n}}{\sqrt{\epsilon}}}$ -- in our case, $\epsilon = \frac{\sqrt{d}}{\rho n}$, which yields $Tm = O\br{\sqrt{n}\sqrt{\frac{\rho n}{\sqrt{d}}}}$ stochastic gradient computations. Note that this is smaller than that of \emph{noisy-m-A-SGD} (unless $\rho$ is very small), however $T = \sqrt{\frac{\rho n}{\sqrt{d}}}$ --  same as that of  \emph{noisy-m-A-SGD}, and hence yields no improvement in unlearning time.
However, note that using Katyusha would give us optimal oracle complexity even in the \emph{high} accuracy regime. 

%% file: sections/subsections/runtime-unlearning-with-proof.tex
\subsection{Unlearning runtime}
\label{sec:unlearning-runtime}

We now look at how much compute it takes for Algorithm \ref{alg:unlearn-subsample-gd} and \ref{alg:unlearn-noisy-m-a-sgd} to handle the edit requests. 
We first give a general result, which holds for any TV-stable algorithm with the unlearning algorithm being the one which constructs a coupling with acceptance probability at least $1-\rho$. 
We give in-expectation bounds 
on the number of times verification fails or a full or partial recompute is triggered.

\begin{proposition}
\label{prop:runtime_main}
For a coupling based unlearning algorithm with acceptance probability at least $1-\rho$, for $k$ edit requests, the expected number of times recompute is triggered is at most $4k \rho$. 
\end{proposition}
\begin{proof}[Proof of Proposition \ref{prop:runtime_main}]
We first setup some notation. 
In the general setup, for $k$ edit requests, let $s$ be the number of times a recompute is triggered. 
Let $\bc{Z_1,Z_2,\ldots, Z_s}$ be a set of random variables, where each $Z_i$ denotes how many edit requests the $i^{\text{th}}$ recompute can \emph{handle}. 
To elaborate, $Z_i$ takes value $j$, if upon $j$ edit requests, a recompute is triggered. 
The $Z_i$'s comprises to the randomness used in the algorithm like mini-batching indices or Gaussian noise, as well as the randomness used for rejection sampling.
It is important to note that $Z_i's$ are not necessarily independent. 
In particular, in Algorithm \ref{alg:unlearn-noisy-m-a-sgd}, we reuse the Gaussian noise upto the iteration in which rejection sampling fails, and only use fresh/independent Gaussian noise in the later steps.
However, note that we have exact unlearning, and the output at each step is $\rho$-TV stable (w.r.t. all the randomness used). Hence, since the above description of the distribution of $Z_i$'s depend only on the TV stability parameter,
it follows that $Z_i's$ are (marginally) identical.

We now use the fact that the unlearning algorithm constructs a coupling with acceptance probability at least $1-\rho$ to describe the probability distribution of $Z_i$. 
We have that upon one edit request, the probability that a recompute is triggered is at most $\rho$. 
This means that $Z_i < 2$ with probability $\leq \rho$. 
Using Remark \ref{remark:tv_upto_k}, this generalizes as $Z_i < j$ with probability at most $(j-1)\rho$.
Note that in our setup, we observe at most $k$ requests, so $Z_i$ taking values larger than $k$ is not meaningful. However if $k\rho < 1$, it means that probability that $Z_i$ takes values smaller than $k$ is less than $1$, and therefore there is a positive probability of $Z_i$ being larger than $k$. 
To remedy this, we define another random variable $X_i$'s which takes values in the set $\bc{1,2,\ldots, k}$. 
Furthermore, for any $i$,  $\P{X_i = j} = \P{Z_i = j}$ for $j\leq k$, but $\P{Z_i = j} = \sum_{l=k}^\infty \P{X_i=l}$. By construction, this ensures that $1\leq Z_i\leq k$, when we observe at most $k$ requests.

We want upper bounds on $s$ conditioned on the fact that $k$ requests are addressed i.e. $X_1+X_2 \ldots, X_s \geq k$. 
For this we write $s$ as $s:=\min_q\bc{2k \geq X_1+X_2 \ldots, X_q \geq k}$. 
The first inequality holds trivially since we ensured that $X_i\leq k$.
 It is easy to see that $s$ is a stopping time with respect to the filtered probability space of the stochastic process $\bc{X_i}_{i\in \bbN}$. 
Furthermore, since $X_i$'s are identical, we can apply Wald's equation to get,

\begin{align*}
     &2k \geq \mathbb{E}[X_1+X_2 + \ldots X_s] = \mathbb{E}[s] \mathbb{E}[X_1] =  \mathbb{E}[s]  \sum_{j=1}^{k} \bbP\bc{X_i \geq j} =
      \mathbb{E}[s]  \sum_{j=1}^{\infty} \bbP\bc{Z_i \geq j} 
     \\& \geq \mathbb{E}[s]  \sum_{j=1}^{1+1/\rho} \bbP\bc{Z_i \geq j}  =  \mathbb{E}[s]  \sum_{j=1}^{1+1/\rho}\br{1- \bbP\bc{Z_i < j}} \geq   \mathbb{E}[s]\br{\frac{1}{\rho} - \sum_{j=1}^{1+1/\rho}(j-1)\rho} \\&\geq  \mathbb{E}[s]\br{\frac{1}{\rho} - \int_{j=0}^{1/\rho}j\rho dj}   \geq \mathbb{E}[s]\br{\frac{1}{\rho} - \frac{1}{2\rho^2}\rho} 
     \geq  \frac{\mathbb{E}[s]}{2\rho}
\end{align*}
This gives us that $\mathbb{E}[s] \leq 4k\rho$. 
\end{proof}

Next, we look at the runtimes for Algorithm \ref{alg:unlearn-subsample-gd} and Algorithm \ref{alg:unlearn-noisy-m-a-sgd} to handle one deletion or insertion request. For this, we look at the runtime of verification, i.e., deciding if recompute needs to be triggered or not. We  show how in the standard algorithmic model of computation (say, word RAM model), using suitable data structures, this can be done efficiently. Furthermore, as standard in convex optimization, we can use  Nemirovski-Yudin's model of computation \cite{nemirovsky1983problem} which counts the number of accesses to the first-order (gradient) information of the function, and a projection oracle.
Let $\mathfrak{G}$ denote the compute cost for one gradient access or projection in the standard model of computation -- 
we assume that both oracles require the same compute. In the rest of the discussion, we provide runtime as a function of the problem parameters ignoring all constants. Furthermore, since we assumed that the number of samples at any point in the stream is between $\frac{n}{2}$ and $2n$, we will just work with $n$ samples, and everything would still be the same, up to constants.

\paragraph{Verification runtime of Algorithm \ref{alg:unlearn-subsample-gd}.} For Algorithm \ref{alg:unlearn-subsample-gd},
note that for deletion, for every iteration, we need to check if the used mini-batch $b_t$ contained the requested point. A brute force search takes $O(m)$ time, whereas if we sort when we save the mini-batch indices $b_t$, we can do a binary search in $O(\log{m})$ time;
we can even do constant time search by storing a dictionary/hash table, giving us an $O(T)$ total time. 
The most efficient way however is to store a dictionary of sample to mini-batch iterations that the sample was used in. For this, it takes $O(1)$ time lookup for every edit request.
For insertion, similarly, at every iteration, we first sample from a Bernoulli with bias $m/n$ which takes constant time, giving us $O(T)$ total time. However, equivalently, we just sample one Bernoulli with bias $Tm/n$ and recompute based on its outcome. This gives us an $O(1)$ time lookup for every edit request.

\paragraph{Verification runtime of Algorithm \ref{alg:unlearn-noisy-m-a-sgd}.} For Algorithm \ref{alg:unlearn-noisy-m-a-sgd}, we can similarly search in constant time whether the deleted point was used in any iteration or not.
For every iteration in which the deleted point is in the mini-batch, we need to compute a gradient at a new point, so as to replace the deleted point. Sampling a point uniformly from a discrete universe takes linear time (in the size) in the worst case, but with some pre-processing can be done in logarithmic/constant time. For example, when saving the mini-batch indices $b_t$, if we save a sorted list of the indices not sampled, using binary search, we can sample in $O(\log{n-(m-1)})$ time. The more efficient way is, if we save a probability table, then we can use Alias method to sample in $O(1)$ time \cite{walker1977efficient}. Hence for such iterations, we query \emph{two} gradients, and it takes $O(d)$ compute to add/subtract this gradients.  
Since the total number of iterations in which a deleted point was sampled in, in expectation, is $\frac{Tm}{n}$, the expected total compute is $\frac{Tm(\mathfrak{G}+d)}{n}$. 

We now  consider the computational cost of rejection sampling.
In Algorithm \ref{alg:unlearn-noisy-m-a-sgd}, 
at every iteration we check if Unif$\br{0,1} \leq  \frac{\phi_{\cN(\g_{t}',\sigma^2\bbI)}(\xi_t)}{\phi_{\cN(\g_{t},\sigma^2\bbI)}(\xi_t)}$, where $\phi_{\cN(\g_{t},\sigma^2\bbI)}(\cdot)$ and $\phi_{\cN(\g_{t}',\sigma^2\bbI)}(\cdot)$ are probability densities evaluated at the sampled point $\xi_t$. We thus need to compute this ratio of probability densities -- since these are Gaussian densities, the ratio is just the following the expression:
\begin{align*}
      \frac{\phi_{\cN(\g_{t}',\sigma^2\bbI)}(\xi_t)}{\phi_{\cN(\g_{t},\sigma^2\bbI)}(\xi_t)} = \frac{\frac{1}{(\sqrt{2\pi \sigma^2})^d}\exp{-\frac{\norm{\g_{t}'-\xi_t}^2}{2\sigma^2}}}{\frac{1}{(\sqrt{2\pi \sigma^2})^d}\exp{-\frac{\norm{\g_{t}-\xi_t}^2}{2\sigma^2}}} = \exp{\frac{1}{2\sigma^2}\br{\norm{\g_{t}-\xi_t}^2 - \norm{\g_{t}'-\xi_t}^2}}.
\end{align*}

It takes $O(d)$ time to do the above computation. Moreover, we only need to compute the ratio in iterations where the means differ -- 
these  correspond to the iterations where the deleted point was sampled or the inserted point would have been sampled. By a direct computation, the expected number of such iterations is $\frac{Tm}{n}$. 
This gives us a computational cost of $\frac{Tmd}{n}$ for rejection sampling, and hence the expected runtime of verification is $\frac{Tm(\mathfrak{G}+d)}{n}$.

We now state bounds on runtime for both unlearning algorithms.

\begin{claim}
\label{claim:runtime_sub_sample_GD}
    The expected total unlearning runtime of Algorithm \ref{alg:unlearn-subsample-gd} for $k$ edit requests is 
    $O\br{\max\bc{k,\min\bc{\rho,1} k \cdot \text{Training time}}}$.
\end{claim}
\begin{proof}[Proof of \cref{claim:runtime_sub_sample_GD}]
The total runtime of Algorithm \ref{alg:unlearn-subsample-gd} is the time for verification plus the runtime for recomputation, whenever a recompute is triggered. 
The recomputation time is just the training time, and in the model considered, excepted cost of one recomputation takes $O\br{T m \br{\mathfrak{G}+d}}$ time, since at every iteration, $m$ gradients are computed and vectors added. 
As discussed in \cref{sec:runtime}, the expected verification time for Algorithm \ref{alg:unlearn-subsample-gd} is $O(1)$. 
From \cref{prop:unlearn-subsample-GD}, the unlearning \cref{alg:unlearn-subsample-gd} recomputes with probability $O(\min\bc{\rho,1})$ for one edit request.
Therefore, using Proposition \ref{prop:runtime_main} which bounds the number of recomputes, we have that the expected total runtime is bounded as $kO(1) + 4k\min\bc{\rho,1} \cdot O\br{T m \br{\mathfrak{G}+d}} \leq O\br{\max\br{1, \min\bc{\rho,1}T m \br{\mathfrak{G}+d}}k}$.
For a sufficiently large $\rho$, the unlearning time of Algorithm \ref{alg:unlearn-subsample-gd} is clearly dominated by the training time. 
In particular, in the corresponding batch Algorithm \ref{alg:sub-sample-GD}, we set $m=\frac{\rho n}{T}$, giving a total runtime of $O\br{\max\br{1, \min\bc{\rho^2,1}n \br{\mathfrak{G}+d}}k}$. 
Hence for $\rho \gtrsim \frac{1}{\sqrt{n(\mathfrak{G}+d})}$, the total runtime in expectation is at most $O(\min\bc{\rho,1} \cdot k\cdot \text{Training time})$. 
In the other case, the expected total runtime is just $O(k)$.
\end{proof}

\begin{claim}
\label{claim:runtime_noisy_A_SGD}
    The expected total unlearning runtime of Algorithm \ref{alg:unlearn-noisy-m-a-sgd} for $k$ edit requests is 
    $O\br{\max\bc{k,\min\bc{ \rho\sqrt{T},1} \cdot k \cdot \text{Training time}}}$.
\end{claim}

\begin{proof}[Proof of \cref{claim:runtime_noisy_A_SGD}]
As before, the total runtime of Algorithm \ref{alg:unlearn-subsample-gd} is the time for verification plus the runtime for recomputation, whenever a recompute is triggered. 
As discussed in \cref{sec:runtime}, the expected verification time for Algorithm \ref{alg:unlearn-noisy-m-a-sgd} is $O\br{\frac{Tm(\mathfrak{G}+d)}{n}}$. 
The recomputation in this case may be partial but it also includes a reflection.
The reflection operation with $d$ dimensional vectors takes $O(d)$ compute. 
Furthermore, we upper bound the partial recomputation time by worst-case full recomputation time, giving a recomputation time 
$O\br{T m \br{\mathfrak{G}+d}}$.
From \cref{lem:noisy-m-sgd-accept}, we have that the unlearning coupling is not maximal but recomputes with probability $\min\bc{\rho \sqrt{T},1}$.
Finally, by Proposition \ref{prop:runtime_main} we have that the expected total runtime is bounded as $kO\br{T m \br{\mathfrak{G}+d}} + 4k\min\bc{\rho \sqrt{T},1} \cdot O\br{ \frac{Tm(\mathfrak{G}+d)}{n}} \leq O\br{\max\br{\min\bc{\rho \sqrt{T},1},\frac{1}{n}}kmT(\mathfrak{G}+d)}$.
In contrast, for Algorithm \ref{alg:unlearn-noisy-m-a-sgd}, the runtime is at most $O\br{\max\br{\min\bc{\rho \sqrt{T},1},\frac{1}{n}}kmT(\mathfrak{G}+d)}$. 
Our lower bounds will show that we need $\rho \gtrsim \frac{1}{n}$ to get any non-trivial accuracy.
Therefore the maximum is always obtained by $\min\bc{\rho \sqrt{T},1}$. 
Moreover, $k$ is a trivial lower bound on runtime, since we need to observe all $k$ edit requests. 
Hence, we get that the total runtime in expectation, is at most $O\br{\max\bc{k,\min\bc{\rho \sqrt{T},1}\cdot k \cdot \text{Training time}}}$. 
\end{proof}

%% file: sections/subsections/runtime-space.tex
\subsection{Space complexity}
\label{sec:space-complexity}
In this work, the objective was not to optimize the memory used, but rather, to study if the problem can be solved computationally efficiently, no matter how much (reasonable) memory the algorithm uses. However, we discuss, in this section, that the space complexities of the proposed algorithms, which we will see, is arguably, reasonably small.
We ignore the space used to store the dataset.
In both algorithms, we save a hash-table of iterations to samples - since we do $T$ iterations with $m$ samples each, this takes space of $O(Tm)$ words. 
We also store all the iterates, which are $d$-dimensional vectors, so this takes a space of $O(dT)$ words. 
Finally, we also store a dictionary of iterations to models, which takes  $O(T)$ space. The space complexity therefore is $O(T(\max\bc{m,d})$. Plugging in the values of $T$, we get the following.

\paragraph{Algorithm \ref{alg:sub-sample-GD}:} Plugging $T \leq \frac{\rho n}{m}$ from \cref{prop:upper_bound_subsample}, we get space complexity = $O\br{\rho n \max\bc{1,\frac{d}{m}}}$.
As remarked in \cref{sec:alg-learning-sub-sample-GD}, we can improve the space complexity by not requiring to save all the iterates and yet have the same unlearning runtime. In the proof of \cref{claim:runtime_sub_sample_GD}, we upper bound the recomputation time by a \emph{full} re-computation time - this means that the upper bound on unlearning runtime holds even if the algorithm does full retraining everytime verification fails. The unlearning \cref{alg:unlearn-subsample-gd} can thus be modified as follows: for deletion, instead of continue retraining from iteration $t$ where the deleted point participates, we can just do \emph{full} retraining, with fresh randomness for all mini-batches. For insertion, note that when if condition is met (line 6 in \cref{alg:unlearn-subsample-gd}), we use the iterate $\w_t$ to compute the gradient on the inserted point (line 8 in \cref{alg:unlearn-subsample-gd}); however, if we don't save $\w_t$, we can just compute it on the fly by doing a full retraining with the same old mini-batches. After $\w_t$ is computed, we just continue as in \cref{alg:unlearn-subsample-gd}.

With the above modification, we only need to save a hash-table of used samples to binary values which correspond to whether they were used or not, which takes $O(Tm)$ words, and a $d$ dimensional model. Hence, the space complexity of \cref{alg:unlearn-subsample-gd} is   $O(Tm+d)$ words.

\paragraph{Algorithm \ref{alg:noisy-m-a-sgd}:} 
From Proposition \ref{prop:upper-bound-noisy-m-a-sgd}, note that if $m\leq O(T^3)$, $T = \frac{\rho n}{md}$, and therefore, $dT =  \frac{\rho n}{m}$.
If we use the largest mini-batch size $m= O(T^3)$, then $T = \sqrt{\frac{\rho n}{\sqrt{d}}}$, and hence $dT = d^{3/4}\sqrt{\rho n}$.
Therefore, the space complexity is $O(T\max\bc{m,d}) \leq O\br{\max\bc{\frac{(\rho n)^2}{d},d^{3/4}\sqrt{\rho n}}}$ words.

%% file: sections/otheralgorithms.tex
\section{Other algorithms and batch unlearning}
To demonstrate the generality of our framework, we give two more algorithms. The first is \emph{noisy-m-SGD} which is the same as \cref{alg:noisy-m-a-sgd} but without acceleration, and the second is \emph{quantized-m-SGD}, based on randomized quantization. We note that both algorithms have worse theoretical guarantees than \cref{alg:noisy-m-a-sgd}, however the first establishes our claim that acceleration is beneficial, whereas the second shows how a previous work of \cite{ginart2019making} for $k$-means clustering, can, not only be seen as a special case of our framework, but also extended to general convex risk minimization problems. 
Moreover, in the second case, we consider a more general setup of \emph{batch} edit requests, and show that our techniques are flexible enough to easily generalize to the batch variant.
\subsection{\emph{noisy-m-SGD}}
\label{sec:noisy-m-sgd}

\begin{algorithm}[ht]
\caption{\emph{noisy-m-SGD}($\w_{t_0},t_0$)}
\label{alg:noisy-m-sgd}
\begin{algorithmic}[1]
\REQUIRE{Initial model $\w_{t_0}$, data points $\bc{\z_1,\ldots, \z_n}$, $T, \eta, m$}
\STATE $\w_0 = 0$
\FOR{$t=t_0,t_0+1,\ldots, T$}
\STATE Sample mini-batch $b_t$ of size $m$ uniformly randomly
\STATE $\g_t = \frac{1}{m}\sum_{j \in b_t} \nabla f(\w_t,\z_j) $
\STATE Sample $\theta_t \sim \cN(0,\sigma^2\bbI_d)$
\STATE $\w_{t+1} = \cP\br{\w_t - \eta \br{ \g_t +\theta_t}}$
\STATE Save($b_t, \theta_t,\w_t, \g_t$)
\ENDFOR
\ENSURE{$\hat \w_S = \frac{1}{T}\sum_{t=1}^{T+1}\w_t$}
\end{algorithmic}
\end{algorithm}

\begin{proposition}
\label{prop:upper_bound_noisy-m-sgd}
Let $f(., \z)$ be an $L$-smooth $G$-Lipschitz convex function $\forall \ \z$.
Algorithm \ref{alg:noisy-m-sgd}, run with $t_0=1, \eta = \min\bc{\frac{1}{2L}, \frac{D}{\br{\frac{G}{\sqrt{m}} +\sigma}\sqrt{T}}}$, $\sigma = \frac{8\sqrt{T}G}{n\rho}$, and $T\geq \frac{(\rho n)^2}{16m^2}$ outputs $\hat \w_S$ which is $
\min\bc{1,\rho}$-TV stable and satisfies $\E{\hat F_S(\hat \w_S) - \hat F_S(\w^*_S)} \lesssim \frac{GD\sqrt{d}}{\rho n}$ 
\end{proposition}
\begin{proof}[Proof of \cref{prop:upper_bound_noisy-m-sgd}]
The $TV$-stability guarantee of $\rho = \frac{8\sqrt{T}G}{n}$ follows exactly as in the proof of \cref{prop:upper-bound-noisy-m-a-sgd}.
We now proceed to the accuracy guarantee, which follows simply by guarantee of SGD on smooth convex functions. We have already shown in \cref{prop:upper-bound-noisy-m-a-sgd} that the gradients are unbiased and its variance bounded by $\frac{2G^2}{m} + \frac{2G^2}{m} + \sigma^2d$.

Therefore, using Theorem 4.1 in \cite{allen2018make} with step-size $\eta \leq \frac{1}{2L}$, we have
\begin{align*}
   \E{ \hat F(\hat \w) - \hat F(\w^*)} \leq O\br{ 2\eta\cV^2 + \frac{D^2}{\eta T}} = O\br{2\eta \br{\frac{2G^2}{m} +\sigma^2d} + \frac{D^2}{\eta T}}
\end{align*}

Let $\tilde G^2 = \frac{2G^2}{m} +\sigma^2d$, balancing the trade-off in $\eta$ gives us $\eta = \frac{D}{\tilde G\sqrt{T}}$. Therefore setting $\eta = \min\bc{\frac{1}{2L},\frac{D}{\tilde G\sqrt{T}}}$ gives us 
\begin{align*}
   \E{ \hat F(\hat \w) - \hat F(\w^*)} 
   &\leq O\br{\frac{LD^2}{T} + \frac{\tilde GD}{\sqrt{T}}} \leq O\br{\frac{LD^2}{T}  + \frac{GD}{\sqrt{Tm}} +  \frac{\sigma\sqrt{d} D}{\sqrt{T}}} \\
  &\leq O\br{\frac{LD^2}{T}  + \frac{GD}{\sqrt{Tm}} + \frac{G D\sqrt{d}}{n \rho}} 
\end{align*}

Finally, the condition in the sub-sampling amplification $\frac{8G^2}{m^2\sigma^2} \leq 1.256$ again becomes $T\geq \frac{(n\rho)^2}{16m^2}$.
\end{proof}

We now show that \cref{alg:noisy-m-sgd} achieves the same upper bound on excess empirical risk as \cref{alg:noisy-m-a-sgd}.

\begin{corollary}
\label{cor:upper-bound-noisy-m-sgd}
Let $f(., \z)$ be an $L$-smooth $G$-Lipschitz convex function $\forall \ \z$.
Algorithm \ref{alg:noisy-m-sgd}, run with $m \geq \min \bc{\frac{d}{16}, \frac{1}{4}\br{\frac{(\rho n) G\sqrt{d}}{LD}}^{1/2}}$, $\eta = \min\bc{\frac{1}{2L}, \frac{D}{\br{\frac{G}{\sqrt{m}} +\sigma}\sqrt{T}}}$, $\sigma = \frac{8\sqrt{T}G}{n\rho}$, and $T=\max\bc{\frac{(\rho n)^2}{md}, \frac{LD\rho n}{G\sqrt{d}}}$ outputs $\hat \w_S$ which is $
\rho$-TV stable and satisfies $\E{\hat F_S(\hat \w_S) - \hat F_S(\w^*_S)} \lesssim \frac{GD\sqrt{d}}{\rho n}$ 

\end{corollary}
\begin{proof}[Proof of \cref{cor:upper-bound-noisy-m-sgd}]
We start with the result in \cref{prop:upper_bound_noisy-m-sgd}.
Note that as long as $\frac{GD}{\sqrt{mT}} \gtrsim \frac{LD^2}{T} \iff m \lesssim \frac{TG^2}{LD}$, the second term is larger than the first.
We balance the two trade-offs in $T$.
Optimizing the trade-off between second and third term gives us
$\frac{GD}{\sqrt{mT}} = \frac{GD\sqrt{d}}{\rho n} \iff T = \frac{(\rho n)^2}{md}$;
and optimizing the second trade-off gives us
$\frac{\sqrt{d}}{\rho n} = \frac{LD^2}{T} \iff T = \frac{LD^2(\rho n)}{\sqrt{d}}$. Hence setting $T = \max\br{\frac{(\rho n)^2}{md},\frac{LD^2(\rho n)}{\sqrt{d}} }$ yields an expected excess empirical risk of $O\br{\frac{GD\sqrt{d}}{n \rho}}$. 

We now look at the condition $T\geq \frac{(n\rho)^2}{16m^2}$ given in \cref{prop:upper_bound_noisy-m-sgd}, with $T$ set as  $T = \max\br{\frac{(\rho n)^2}{md},\frac{LD^2(\rho n)}{\sqrt{d}} }$.
We therefore require $\frac{(\rho n)^2}{md} \geq \frac{(\rho n)^2}{16m^2} \iff m \geq \frac{d}{16}$, as well as $\frac{LD(\rho n)}{G\sqrt{d}} \geq  \frac{(\rho n)^2}{16m^2}  \iff  m \geq \frac{1}{4}\br{\frac{(\rho n) G\sqrt{d}}{LD}}^{1/2}$ - this recovers the condition $m \geq \min \bc{\frac{d}{16}, \frac{1}{4}\br{\frac{(\rho n) G\sqrt{d}}{LD}}^{1/2}}$ in the Proposition statement.
Hence, combining all the above arguments, we get that for any  $m \geq \min \bc{\frac{d}{16}, \frac{1}{4}\br{\frac{(\rho n) G\sqrt{d}}{LD}}^{1/2}}$  , setting $T = \max\bc{\frac{(\rho n)^2}{md},\frac{LD(\rho n)}{G\sqrt{d}}}$, yields an expected excess empirical risk of $O\br{\frac{GD\sqrt{d}}{n \rho}}$.
\end{proof}

\begin{algorithm}[ht]
    \caption{Unlearning for \emph{noisy-m-SGD}}
    \label{alg:unlearn-noisy-m-sgd}
    \begin{algorithmic}[1]
        \REQUIRE{Delete point with index $j$ or insert $\z$ (with index $n+1$) for \emph{noisy-m-SGD}}
        \FOR{$t=1,2\ldots, T$}
        \STATE Load$\br{\theta_t, \w_t, b_t, \g_t)}$
        \IF{\textit{deletion} and $j\in b_t$}
        \STATE Sample $i \sim \text{Uniform}([n]\backslash b_t)$
        \STATE $ \g_t' = \g_t -   \frac{1}{m}\br{\nabla f(\w_t,\z_j) -  \nabla f(\w_t,\z_i)}$
        \STATE Save($ \g_t', b_t\backslash \bc{j} \cup \bc{i} $)
        \ELSIF{\textit{insertion}  and Bernoulli$\br{\frac{m}{n+1}}$}
        \STATE Sample $i \sim \text{Uniform}(b_t)$
        \STATE $\g_t' = \g_t-  \frac{1}{m}\br{\nabla f(\w_t,\z_i) -  \nabla f(\w_t,\z)}$
        \STATE Save($ \g_t', b_t\backslash \bc{i} \cup \bc{n+1} $)
        \ELSE
        \CONTINUE
        \ENDIF
        \STATE $\xi_t = \g_t + \theta_t$
        \IF{Uniform$\br{0, 1} \geq \frac{\phi_{\cN(\g_t',\sigma^2\bbI)}(\xi_t)}{ \phi_{\cN(\g_t,\sigma^2\bbI)}(\xi_t)}$}
        \STATE $ \xi'_t = \text{reflect}(\xi_t,\g_t', \g_{t})$
        \STATE $\w_{t+1} = \w_t - \eta  \xi_t'$
        \STATE $\text{Save}(\xi'_t)$
        \STATE \emph{noisy-m-SGD($\w_{t+1},t+1$)} \hfill{\textit{ // Continue retraining, on current dataset}}
        \BREAK
        \ENDIF
        \ENDFOR
    \end{algorithmic}
\end{algorithm}

\begin{remark}
 The choice of $T$ in Proposition \ref{prop:upper-bound-noisy-m-a-sgd} yields that the largest mini-batch size that can be set, without hurting runtime, is $m = \frac{\rho n G}{\sqrt{d} L D}$.
 Furthermore, the condition $m \geq \min \bc{\frac{d}{16}, \frac{1}{4}\br{\frac{(\rho n) G\sqrt{d}}{LD}}^{1/2}}$ becomes $T\geq \br{\frac{\sqrt{d}LD}{4G}}^2$.
\end{remark}

We now state and prove the main theorem for this section.
\begin{theorem}
\label{thm:main-thm-noisy-sgd}
Let $f(\cdot, \z)$ be an $L$-smooth $G$-Lipschitz convex function $\forall \ \z$.
For any $\frac{1}{n} \leq \rho\leq 1$, using \cref{alg:noisy-m-sgd} as the learning algorithm and \cref{alg:unlearn-noisy-m-sgd} as its unlearning algorithms, then given a stream of edit requests,
\begin{enumerate}
    \item Satisfies exact unlearning at every point in the stream.
    \item At time $i$ in the stream of edit requests, outputs $\hat \w_{S^i}$, such that if $\br{\frac{L^{1/2}D^2 \sqrt{d}}{G (\rho n)}}^{2/3} \leq \frac{GD}{\sqrt{\rho n}}$, then its with excess empirical risk bounded as,
    \begin{align*}
    \E{\hat F_S(\hat \w_{S^i}) - \hat F_{S}(\w_{S^i}^*)} \lesssim
    \br{\frac{L^{1/2}D^2 \sqrt{d}}{G (\rho n)}}^{2/3}
\end{align*}
\item For $k$ edit requests, the  expected total unlearning runtime is $O(\max\bc{\rho k \cdot \text{Training time},k})$  
\end{enumerate}

\end{theorem}
\begin{proof}[Proof of \cref{thm:main-thm-noisy-sgd}]
We proceed as in the proof of \cref{thm:main-result}. For any $0< \tilde \rho \leq 1$, from \cref{prop:upper_bound_noisy-m-sgd}, the output $\hat\w_{S}$ is $\tilde \rho$-TV stable, and the excess empirical risk using \cref{alg:noisy-m-sgd} on a dataset $S$ on $n$ points, is bounded as,
\begin{align*}
    \mathbb{E}\hat F_S(\hat \w_S) - \hat F_S(\w^*_S) \leq  O\br{\frac{LD^2}{T}  + \frac{GD}{\sqrt{Tm}} + \frac{G D\sqrt{d}}{n \tilde \rho}} 
\end{align*}

It can be easily verified that \cref{prop:unlearn-noisy-m-A-SGD} and \cref{prop:unlearn-noisy-m-A-SGD} still holds for \emph{noisy-m-SGD}, which gives us that the algorithm satisfies exact unlearning at every time in the stream, proving the first part of the claim, Moreover, its recompute probability bounded by $O(\tilde \rho k \sqrt{T})$ and therefore the unlearning runtime bounded by $O(\max \bc{k,\tilde \rho k \sqrt{T}\cdot \text{Training time}}$. Substituting $\tilde \rho = \frac{\rho}{\sqrt{T}}$, and using the largest mini-batch size $m = \br{\frac{G}{LD}}^2T$, the upper bound on   excess empirical risk becomes $\frac{LD^2}{T} + \frac{G D\sqrt{d}\sqrt{T}}{n \rho} $. Optimizing the trade-off, we have $\frac{LD^2}{T} = \frac{LD^2}{T} \iff T = \br{\frac{LD (\rho n)}{G\sqrt{d}}}^{2/3}$, and the excess empirical risk bound upper bound is $\frac{LD}{T} = \br{\frac{L^{1/2}D^2 \sqrt{d}}{G (\rho n)}}^{2/3}$.
Note that this also proves the third part of the claim. Furthermore, as in the proof of \cref{thm:main-result}, it can be verified that the condition $T\geq \frac{(\tilde \rho n)^2}{16m^2}$ is equivalent to $\br{\frac{L^{1/2}D^2 \sqrt{d}}{G (\rho n)}}^{2/3} \leq \frac{1}{\sqrt{\rho n}}$, which just means that the excess empirical risk of \emph{noisy-m-SGD} is at most that of \emph{sub-sample-GD}.
Finally, the upper bound holds for any point in the stream using the assumption that the number of samples are between $\frac{n}{2}$ and $2n$, thereby establishing the second claim.
\end{proof}

\subsection{\emph{quantized-m-SGD}}
\label{sec:qsgd}

The work of \cite{ginart2019making} considers unlearning in $k$-means clustering. The key algorithmic technique is randomized quantization of vectors to a $\tau$-lattice.  
The intuition is that if the vector is an average of $n$ data points which are bounded in norm, then upon changing one data point, the vectors $O\br{\frac{1}{n}}$ close.
Therefore, if the lattice is sufficiently coarse, then it would ensure that both are mapped up the same point in the lattice. However, if we consider deterministic quantization, then there exists points such that for any $\epsilon>0$, shifting the point by $\epsilon$ changes the quantized point. Therefore, we first shift the lattice by a uniformly random phase, which ensures that such a situation occurs with a small probability. 

In their application of $k$-means clustering, this vector is a cluster centroid, which is an average of the data points in the cluster. We apply this idea to convex risk minimization problems, wherein we quantize average gradients, which by the Lipshcitzness assumption are bounded in norm.

We now introduce the quantization operation formally.
Given a vector $\x$, let $\theta \sim \text{Unif}\left[-\frac{1}{2},-\frac{1}{2}\right]^d$, consider the quantization given by:
\begin{align*}
    Q_{\theta}(\x) = \tau\br{\theta + \arg \min_{j \in \Z^d} \br{\x - \tau(\theta + j)}}
\end{align*}

We now state a result about the quantization operation.

\begin{lemma}
\label{claim:quantize}
Let $B_\delta(\u)$ denote the Euclidean call of radius $\delta$ centered at $\u$.
The following holds for the quantization operation,
\begin{enumerate}
    \item For any $\x$, $\E{Q_\theta(\x)}= \x$ and $\E{\norm{\x-\E{\x}}}^2 \leq \tau^2d $
   \item For any vector $\u$, $\P{ \exists \v \in B_\delta(\u):Q_\theta(\u)\neq \Q_\theta(\v)} \leq \frac{2d \delta}{\tau}$
\end{enumerate}
\end{lemma}

\begin{proof}[Proof of \cref{claim:quantize}]
Note that for a given $\x$, $\Q_\theta(\x) \sim \text{Unif}\left[\x-\frac{\tau}{2},\x+\frac{\tau}{2}\right]^d$, hence $\E{Q_\theta(\w)}= \w$. Furthermore, since $\w-\E{\w} \sim \text{Unif}\left[-\frac{\tau}{2},\frac{\tau}{2}\right]^d$, we have $\E{\norm{\w-\E{\w}}}^2 =d \mathbb{E}\br{\w_1-\E{\w_1}}^2 =  \frac{d\tau^2}{12}$.
The second part of the claim is Lemma C.2 in \cite{ginart2019making}.
\end{proof}

To see why \cite{ginart2019making} is a special case of our framework, note that the total variation distance between two random variables is at most the probability of disagreement under \emph{any} coupling. \cite{ginart2019making} uses the same quantization randomness (used for training) for verifying after the edit request - this  corresponds to a trivial coupling between the quantization randomness, hence the total variation distance between the outputs is bounded by the upper bound on the probability that the quantized points change (see \cref{claim:quantize}). This establishes that it is a TV stable method. Finally, as said before, using the same quantization randomness corresponds to a trivial coupling, but can be shown to be maximal since the probability distribution is uniform around the to-be-quantized point. Therefore, we have that \cite{ginart2019making} uses a maximal coupling based unlearning method.

\paragraph{Batch unlearning:} 
We consider a batch unlearning setup, wherein instead of observing an insertion or deletion request, we observe a batch edit request with insertions and deletions. We demonstrate that our general approach of coupling mini-batch indices is flexible enough to handle this variant naturally. The batch unlearning ideas and results extend to other algorithms: \emph{noisy-m-A-SGD}, \emph{noisy-m-SGD} and \emph{subsample-GD}. We also note that the computational benefit of batch unlearning as opposed to handling edits one by one is only a constant factor, which at best is two.

We now discuss how we extend the randomized quantization idea to convex risk minimization.
In our learning algorithm \emph{quantized-m-SGD}, at each iteration, we draw a mini-batch of $m$ samples, uniformly randomly from $n$ samples, use it to compute the gradient on the previous iterate , quantize using a randomly sampled phase, and update. Algorithm \ref{alg:unlearn_qsgd} implements the above procedure.

\begin{algorithm}[ht]
\caption{\emph{quantized-m-SGD}}
\label{alg:qsgd}
\begin{algorithmic}[1]
\REQUIRE{Initial model $\w_1$, data points $\bc{\z_1,\ldots, \z_n}$,$T,\eta$}
\FOR{$t=1,2\ldots, T$}
\STATE Sample mini-batch $b_t$ of size $m$ uniformly randomly
\STATE Sample $\theta_t \sim \text{Unif}\left[-\frac{1}{2},\frac{1}{2}\right]^d$
\STATE $\g_t = \frac{1}{m}\sum_{j \in b_t} \nabla f(\w_t,\z_j)$
\STATE $\w_{t+1} = \w_t - \eta Q_{\theta_t}\br{\g_t}$
\STATE Save($\theta_t,\w_t, b_t, \g_t$)
\ENDFOR
\ENSURE{$\hat \w_S = \frac{1}{T+1}\sum_{j=1}^{T+1} \w_j$}
\end{algorithmic}
\end{algorithm}

We first prove a lemma which bounds the total variation distance between outputs generated by \emph{quantized-m-SGD} on arbitrarily differing datasets - these can be thought of as arising after a batch edit request.

\begin{lemma}
\label{lemma:tv-stability-qsgd}
Let $S$ and $S'$ be two datasets of $n$ and $n+k_2$ points respectively, such that  $S$ has $k_1$ points which differ from $S'$ i.e. $\abs{S\backslash S'}=k_1$, therefore $S$ and $S'$ differ by $k_1+k_2$ points.
Let $\bc{\w_j}_{j=1}^T$ and $\bc{\w_j'}_{j=1}^T$ be iterates of \emph{quantized-m-SGD} on datasets $S$ and $S'$ respectively. The total variation distance between distribution of average iterates $\hat \w_S$ and $\hat \w_{S'}$ is bounded as,

\begin{align*}
    \text{TV}(\hat \w_S, \hat \w_{S'}) \leq 
    \frac{4GTd(k_1+k_2)}{n\tau} 
\end{align*}
\end{lemma}

\begin{proof}[Proof of \cref{lemma:tv-stability-qsgd}]
Without loss of generality, we enumerate $S$ and $S'$ into subsets as follows: let $S_1$ and $S_1'$ be the first $n-k_1$ elements of $S$ and $S'$ which are the same. Let $S_2$ and $S_2'$ be the next $k_1$ differing elements in $S$ and $S'$ respectively. Finally, let $S_3'$ be the last $k_2$ elements of $S'$.

We look at iteration $t$ of \emph{quantized-m-SGD} and fix the previous model $\w_t = \w_t' = \w$. We will now compute the conditional total variation distance between $\w_{t+1}$ and $\w_{t+1}'$. Note that since the only randomness is in the sub-sampling and quantization, we can compute the total variation distance between sub-sampled quantized gradients on fixed $\w$ for both datasets, and this will lower bound total variation distance between the iterates $\w_{t+1}$ and $\w_{t+1}'$ by data processing inequality.
Let $b^{(1)}$ and $b^{(2)}$ be a uniform sample of $m$ points from datasets $S$ and $S'$ respectively. 
For a fixed $\w$, let the gradient on $S$ indexed by $b^{(1)}$ be denoted as $\g_{b^{(1)}}^S(\w) = \frac{1}{m}\sum_{j \in b^{(1)}} \nabla f(\w,\z_{j}^S)$, and similarly for $S'$. 
Let $P$ and $Q$ denote the probability distribution of $\g_{b^{(1)}}^S(\w)$ and $\g_{b^{(2)}}^{S'}(\w)$ respectively.
We have the following claim, which we will prove via mathematical induction on $k_2$: for any measurable set $R$, for any $k_2$, $ \abs{P(R)-Q(R)} \leq \frac{4Gdk_1}{n\tau} + \sum_{j=1}^{k_2}\frac{4Gd}{(n+j)\tau}$.

\paragraph{Base case 1: $\mathbf{k_2=0:}$} 
Firstly note that both $b^{(1)} \sim \text{Unif}([n],m)$ and  $b^{(2)} \sim \text{Unif}([n],m)$, and consider the trivial coupling $b^{(1)}=b^{(2)}=b$, where $b \sim \text{Unif}([n],m)$, be a uniform sample of $m$ points from $[n]$. 
We now use the fact that total variation distance is at most the probability of disagreement for any coupling. This gives us that 
\begin{align*}
   \text{TV}(Q_\theta(\g_{b^{(1)}}^S(\w)), Q_\theta(\g_{b^{(2)}}^{S'}(\w)) \leq \P{Q_\theta(\g_b^S(\w)) \neq Q_\theta(\g_b^{S'}(\w))}
\end{align*}
We will focus on upper bounding the right hand side.
The proof follows by using the quantization guarantee (Claim \ref{claim:quantize}) combined amplification from subsampling. 
Without loss of generality, assume that the first $k_1$ samples in $S$ and $S'$ are the ones that differ. Fix the random (uniform) sample $b$ of indices - suppose for this fixed value of $b$, exactly $j$ differing data points are sampled. From $G$ Lipschitzness, and that we have exactly $j$ differing data points,  $\norm{\g_b^S(\w) - \g_b^{S'}(\w) } \leq \frac{2Gj}{m}$. Hence, applying Claim \ref{claim:quantize}, we have that  $$\P{Q_\theta(\g_b^S(\w)) \neq Q_\theta(\g_b^{S'}(\w)) \Big\vert b \text{ producing } j \text{ differing samples}}  \leq \frac{4 Gdj}{\tau m}.$$ 
We will now integrate with respect to the randomness in $b$ - for this, we need to calculate the probability that a sample of $b$ (uniform $m$ out of $n$) produces exactly $j$ differing data points, call it $p(j)$. By direct computation, we have that $p(j) = \frac{{k_1 \choose j}{n-k_1 \choose m-j}}{{n \choose m}}$. Hence we have, \begin{align*}
    \P{Q_\theta(\g_b^S(\w)) \neq Q_\theta(\g_b^{S'}(\w))} &= \sum_{j=0}^{k_1} p(j) \P{Q_\theta(\g_b^S(\w)) \neq Q_\theta(\g_b^{S'}(\w)) \Big \vert \ b \ \text{produces } j \ \text{differing samples}} \\
    & \leq \sum_{j=0}^{k_1} \frac{{{k_1} \choose j}{n-{k_1} \choose m-j}}{{n \choose m}} \frac{4 Gdj}{\tau m}  
      =  \frac{m{k_1}}{n} \frac{4 Gd}{\tau m}  =  \frac{4 Gd{k_1}}{\tau n} 
\end{align*}
where the second last equality is a consequence of Vandermonde's identity, as we show below. We need to show that $\sum_{j=0}^{k_1} \frac{{{k_1} \choose j}{n-{k_1} \choose m-j}j}{{n \choose m}} =  \frac{m{k_1}}{n} \iff \sum_{j=0}^{k_1} {{k_1} \choose j}{n-{k_1} \choose m-j}j = \frac{m{k_1}}{n} {n \choose m} = {k_1}{n-1 \choose m-1}$. This holds because,
\begin{align*}
    \sum_{j=0}^{k_1} {{k_1} \choose j}{n-{k_1} \choose m-j}j &=   \sum_{j=0}^{k_1} \frac{{k_1}}{j}{{k_1}-1 \choose j-1}{n-{k_1} \choose m-j}j = {k_1}\sum_{j=0}^{k_1}{{k_1}-1 \choose j-1}{n-{k_1} \choose m-j} \\
    &={k_1}\sum_{j=0}^{{k_1}-1}{{k_1}-1 \choose j}{n-{k_1} \choose (m-1)-j} = {k_1}{n-1 \choose m-1}
\end{align*}
where in the second last equality, we re-indexed the sum which removes the first element, but it was zero anyway, and the last equality follows from Vandermonde's identity.

\paragraph{Base case 2: $\mathbf{k_2=1:}$} In this case, $S'$ has one more element that $S'$ - let this point be denoted as $q$.
In this case, the probability distribution using $S'$ has the form $Q = \br{1-\frac{m}{n+1}}Q_1+\frac{m}{n+1}Q_2$, where $Q_1$ is the probability distribution conditioned on the event that $q$ is sub-sampled, and $Q_2$ is the probability distribution conditioned on the complementary event. For any measurable set $R$, we have, 
\begin{align*}
    \abs{P(R)-Q(R)} = \abs{\br{1-\frac{m}{n+1}}Q_1(R)+\frac{m}{n+1}Q_2(R)-P(R)}
\end{align*}
Note that $Q_1,Q_2$ and $P$ are all probability distributions over $n$ elements. Furthermore, $P$ and $Q_1$ are probability distributions over $k_1$ differing elements, therefore we can use base case $k_2=0$ to get that $\abs{P(R)-Q_1(R)} \leq \epsilon_1:=\frac{4Gdk_1}{n\tau}$. We therefore get,
\begin{align*}
    \abs{P(R)-Q(R)} &\leq \abs{\br{1-\frac{m}{n+1}}Q_1(R)+\frac{m}{n+1}Q_2(R)-Q_1(R)+\epsilon_1} \\
    & = \abs{\frac{m}{n+1}\br{Q_2(R)-Q_1(R)} + \epsilon_1}
\end{align*}
Finally note that $Q_1$ and $Q_2$ are probability distributions over $n$ such that upon sub-sampling $m$ elements, there is exactly one differing element, therefore
we get, $\abs{Q_2(R)-Q_1(R)} \leq \frac{4Gd}{m\tau}$. We therefore have that 
\begin{align*}
  \abs{P(R)-Q(R)} &\leq  \frac{m}{n+1}\frac{4Gd}{m\tau} + \frac{4Gdk_1}{n\tau} = \frac{4Gd}{(n+1)\tau} + \frac{4Gdk_1}{n\tau} 
\end{align*}

\paragraph{Induction Hypothesis:} Suppose the following holds for $k_2 \leq \tilde k$: for any measurable set $R$, $\abs{P(R)-Q(R)} \leq \frac{4Gdk_1}{n\tau} + \sum_{j=1}^{\tilde k}\frac{4Gd}{(n+j)\tau}$. 
\paragraph{Induction Step: $k_2 = \tilde k+1$}
Let the \emph{last} element of $S'$ be $q$. As in the base case, we decompose the distribution $Q$ into a mixture of two components based on whether $q$ is sampled or not. We have $Q = \br{1- \frac{m}{n+\tilde k+1}} Q_{1} + \frac{m}{n+\tilde k+1}Q_2$. Note that $Q_1$ is a probability distribution which does not use the last element of $S'$. Therefore we can use Induction hypothesis which gives us that $\abs{Q_1(R)-P(R)}\leq \frac{4Gdk_1}{n\tau} + \sum_{j=1}^{\tilde k}\frac{4Gd}{(n+j)\tau}$. We therefore get, 
\begin{align*}
    \abs{P(R)-Q(R)} &= \abs{P(R)-\br{1- \frac{m}{n+\tilde k+1}} Q_{1}(R) - \frac{m}{n+\tilde k+1}Q_2(R)} \\
    & \leq \abs{\frac{m}{n+\tilde k+1}\br{Q_1(R)-Q_2(R)} + \frac{4Gdk_1}{n\tau} + \sum_{j=1}^{\tilde k}\frac{4Gd}{(n+j)\tau}} \\
    &\leq \abs{\frac{m}{n+\tilde k+1} \frac{4Gd}{m\tau} + \frac{4Gdk_1}{n\tau} + \sum_{j=1}^{\tilde k}\frac{4Gd}{(n+j)\tau}} \\
    & = \frac{4Gdk_1}{n\tau} + \sum_{j=1}^{\tilde k+1}\frac{4Gd}{(n+j)\tau}
\end{align*}
where in the last inequality, as in the base case, we used that fact that distributions $Q_1$
 and $Q_2$ differ because in one we subsample the last element where as in the other we don't, so from Claim \ref{claim:quantize}, for two data sets of size $m$ differing in one element, the failure probability is $\frac{4Gd}{m\tau}$. This completes the induction argument.
We bound the sum simply as $\sum_{j=1}^{k_2}\frac{4Gd}{(n+j)\tau} \leq \frac{4Gdk_2}{n\tau}$, which gives us that the whole term is bounded by $\frac{4Gd(k_1+k_2)}{n\tau}$.

The above, by an application of data processing inequality, shows that the conditional TV distance between $\w_{t+1}$ and $\w_{t+1}'$ is at most $\frac{4Gd(k_1+k_2)}{n\tau}$. Note that the upper bound holds uniformly over all conditioning events. 
Moreover, from the maximal coupling characterization of TV distance, we have that for any coupling of $\w_{t+1}$ and $\w_{t+1}'$,
the conditional probability of disagreement is at most $\frac{4Gd(k_1+k_2)}{n\tau}$. Consider the coupling which just concatenates all these couplings, then an application of union bound over the $T$ iterates, the joint probability of disagreement under this coupling is at most $\frac{4Gd(k_1+k_2)T}{n\tau}$ which gives us our upper bound on TV distance between joint iterates.
Finally, by data processing inequality, the same upper bound holds for the average iterates which finishes the proof.
\end{proof}

We now establish the guarantees on the learning \cref{alg:qsgd}.
To handle batch edit request, we extend the notion of exact unlearning with one edit request to batch request: we term it exact batch unlearning. We similarly also extend $\rho$-TV-stability to $(k_1,k_2,\rho)$-TV stability, which is $\rho$-TV stability under arbitrary $k_1$ deletions and $k_1+k_2$ insertions, as well as $k_1$ insertions and $k_1+k_2$ deletions.
\begin{proposition}
\label{prop:upper-bound-qsgd}
Let $f(., \z)$ be an $L$-smooth $G$-Lipschitz convex function $\forall \ \z$.
Algorithm \ref{alg:noisy-m-sgd}, run with $\eta = \min\bc{\frac{1}{2L}, \frac{D}{\br{\frac{G}{\sqrt{m}} +\tau \sqrt{d}}\sqrt{T}}}$, $\tau =  \frac{4GdT}{\rho n}$, and $T = \max \bc{\frac{(\rho n)}{d^{3/2}\sqrt{m}},\br{\frac{LD (\rho n)}{Gd^{3/2}}}^{2/3} }$ outputs $\hat \w_S$ which is $(k_1,k_2,
(k_1+k_2)\rho)$-TV stable and satisfies $ \E{ \hat F_S(\hat \w_S) - \hat F_S(\w^*_S)}  \lesssim \br{\frac{\sqrt{L}GD^2 d^{3/2}}{\rho n}}^{2/3}$.
    
\end{proposition}
\begin{proof}[Proof of \cref{prop:upper-bound-qsgd}]

The $(k_1,k_2,
(k_1+k_2)\rho)$-TV stability guarantee follows from \cref{lemma:tv-stability-qsgd} by taking a supremeum over all datasets $S$ and $S'$ of sizes $n$ and $n+k_2$ (or $n-k_2)$ to get that that TV stability is uniformly upper bound by $\frac{4GTd(k_1+k_2)}{n\tau} =(k_1+k_2)\rho$, where the equality follows upon setting $\tau = \frac{4GdT}{\rho n}$.
For the excess empirical risk bound, we use the guarantee on excess empirical risk of SGD on smooth convex functions (for example, Theorem 4.1 from \cite{allen2018make}), combined with the fact in \cref{claim:quantize} that quantization produces unbiased estimates of the gradient with bounded variance $\cV^2 \leq \frac{2G^2}{m} + \tau^2 d = \frac{2G^2}{m} +\frac{16G^2d^3T^2}{(\rho n)^2}$. Therefore, choosing step size $\eta \leq \frac{1}{2L}$, we get

\begin{align*}
   \E{ \hat F_S(\hat \w_S) - \hat F_S(\w^*_S)} \leq O\br{ 2\eta\cV^2 + \frac{D^2}{\eta T}} = O\br{2\eta \br{\frac{2G^2}{m} +\frac{16G^2d^3T^2}{(\rho n)^2}} + \frac{D^2}{\eta T}}
\end{align*}

Define $\tilde G^2 = 2 \br{\frac{2G^2}{m} +\frac{16G^2d^3T^2}{(\rho n)^2}} $ and set $\eta = \min\bc{\frac{1}{2L}, \frac{D}{\tilde G\sqrt{T}}}$, which makes the upper bound
\begin{align*}
   \E{ \hat F_S(\hat \w_S) - \hat F(\w^*_S)} &\leq O\br{ 2\eta\cV^2 + \frac{D^2}{\eta T}} = O\br{\frac{LD^2}{T} + \frac{\tilde GD}{\sqrt{T}}}\\
   &\leq O\br{\frac{LD^2}{T}+\frac{GD}{\sqrt{Tm}} + \frac{GD d^{3/2}\sqrt{T}}{(\rho n)}}
\end{align*}

Balancing the trade-off between the las two terms gives us $T=\frac{(\rho n)}{d^{3/2}\sqrt{m}}$. Similarly, balancing the trade-off between the first and last term gives us $T = \br{\frac{LD (\rho n)}{Gd^{3/2}}}^{2/3}$. Hence setting $T = \max \bc{\frac{(\rho n)}{d^{3/2}\sqrt{m}},\br{\frac{LD (\rho n)}{Gd^{3/2}}}^{2/3} }$ gives us that the expected excess empirical risk is bounded by $\br{\frac{\sqrt{L}GD^2 d^{3/2}}{\rho n}}^{2/3}$ and completes the proof.
\end{proof}

\begin{remark}
We see that the TV stability parameter above is $(k_1+k_2)\rho$ as opposed to $(k_1+2k_2)\rho$ which is what we would obtain with $\rho$-TV stability for one edit request and using the triangle inequality of TV distance (see \cref{remark:tv_upto_k}).
\end{remark}

\begin{remark}
The largest mini-batch size, without hurting runtime, is $m = \br{\frac{G}{LD}}^2T = \br{\frac{G^2 (\rho n)}{(LD)^2d^{3/2}}}^{2/3}$, which gives us $T=\br{\frac{LD (\rho n)}{Gd^{3/2}}}^{2/3}$.
\end{remark}

We now proceed to unlearning. The unlearning algorithm (\cref{alg:unlearn_qsgd}) upon observing an edit request comprising of both insertions and deletions, first couples the mini-batch indices (described formally in the next paragraph), and computes the gradient on the new mini-batch It then uses the \emph{same} quantization randomness as in training, and checks if the quantized point changes. If it does, in any iteration, then it calls recompute.
The use of the same quantization randomness corresponds to a trivial coupling between the quantization randomness.
We explain the coupling procedure in detail below.

\paragraph{Batch coupling: } We setup some notation. Consider the training dataset $S$ and dataset realized after the batch edit request $S'$.
 Given a vector $\w$, let $Q_{\theta_1}(\g_{b_1}^S(\w))$ denote the quantized gradient vector where $\theta_1$ is the quantization randomness and $b_1$ is the mini-batching randomness on dataset $S'$. Similarly, $Q_{\theta_2}(\g_{b_1}^{S'}(\w))$ denotes the quantized vector with $\theta_2$ as the quantization randomness and $b_2$ as the mini-batching randomness on dataset $S'$. We couple $\theta_1$ and $\theta_2$ by considering the trivial coupling $\theta_1 = \theta_2$ i.e. the joint probability measure is defined only on the diagonal of the product measure. To couple the mini-batch indices, we consider two cases: if the training dataset $S$ has less more or more points than $S'$. For simplicity, \cref{alg:unlearn_qsgd} is the pseudo-code corresponding only to the first case.

In the first case, suppose $S$ has $n$ points and $S'$ has $n+k_2$ points, realized after $k_1 $ deletions and $k_1+k_2$ insertions. Without loss of generality, order the two datasets as follows: the first $n-k_1$ points in $S$ and $S'$ are the same, call these $S_1=S_1'$, next we have the last $k_1$ points of $S$, and arbitrary $k_1$ points of $S'$ - call these $S_2$ and $S'_2$, and moreover let the mapping of indices from $S_2 \rightarrow S_2'$ by denoted by $\iota$. Finally we have the rest of $k_2$ points of $S'$, call this $S_3'$. 
In the following discussion, and in \cref{alg:unlearn_qsgd}, when we consider elements of these sets, we mean their indices.
As before, let $\mu_{n,m}$ and $\mu_{n+k_2,m}$  denote the probability measures correspondingly to sampling $m$ elements uniformly from a discrete universe of size $n$ (i.e. $S$) and $n+k_2$ (i.e. $S'$) respectively.
 These sub-sampling measures are coupled in the following way in \cref{alg:unlearn_qsgd}.
We first sample $b^{(1)}\sim \mu_{n,m}$ (during training). Let $b^{(1)}_2$ be the $m_2$ indices in $S_2$: replace these by the corresponding indices in $S_2'$ i.e. $\iota(b^{(1)} \cap S_2)$.
Next, sample $b\sim \mu_{n+k_2,m}$: let $b_3$ be the $m_3$ indices which are in $S_3'$. We now resample  $b^{(2)}_1 \sim \text{Unif}(b^{(1)},m_3)$ - these are indices used in training, which are now to be replaced.  Define $b^{(2)} = \br{b^{(1)} \backslash \bc{b^{(1)} \cap S_2} \cup  \bc{\iota(b^{(1)} \cap S_2)}}\backslash  b^{(2)}_1 \cup \bc{S_3' \cap b}$. Let the distribution of $b^{(2)}$ produced in the above way be denoted as $\mu^{\text{edit}}_{n, m}$.
We now show that $(b^{(1)}, b^{(2)})$ is indeed a coupling of $\mu_{n,m}$ and $\mu_{n+k_2,m}$.

\begin{claim}
\label{claim:qsgd_coupling1}
  With the construction described above, we have that $b^{(1)} \sim \mu_{n,m}$ and $b^{(2)} \sim \mu^{\text{edit}}_{n, m} = \mu_{n+k_2,m}$.
\end{claim}

\begin{proof}[Proof of \cref{claim:qsgd_coupling1}]
$b^{(1)} \sim \mu_{n,m}$ follows trivially by construction. For the other part, for any set $E$ of $m$ indices arising from the coupling construction, let $E_1$ be the set of $m-m_3$ points from $S \cup S_2'$ and $E_2$ be the set of $m_3$ points from $S_3'$. Since these $m_3$ points of $E_2$ need to be selected when sampling $b$, the probability of sampling these points is $\frac{{n \choose m-m_3}}{{n+k_2 \choose m}}$, where the numerator denotes the number of ways to sample from $S' \backslash S_3'$. For the points in $E_1$, these come from $b^{(1)}$ and replacement using $S_2'$ (which is a deterministic operation). 
Hence, probability of $E_1$ is $\frac{{n-(m-m_3) \choose m_3}}{{n \choose m}}$, where the numerator denotes the number of ways to sample rest of elements not in $E_1$ when sampling $b^{(1)}$. Finally, we need to consider the re-sampling step i.e sampling $b^{(2)}_1$ - note that the draw of $E_1$ and $E_2$ fixes the set produced by this re-sampling, and thus its probability is $\frac{1}{{m \choose m_3}}$. This gives us
\begin{align*}
    \mu^{\text{edit}}_{n, m}(E) &= \frac{{n \choose m-m_3}}{{n+k_2 \choose m}} \cdot \frac{{n-(m-m_3) \choose m_3}}{{n \choose m}} \cdot \frac{1}{{m \choose m_3}} \\
    &= \frac{1}{{n+k_2 \choose m}} \cdot \frac{n! (n-(m-m_3))! m! (n-m)! m_3! (m-m_3)!}{(m-m_3)! (n-(m-m_3))! m_3! (n-m)! n! m!} = \frac{1}{{n+k_2 \choose m}} = \mu_{n+k_2,m}(E)
\end{align*}
\end{proof}

In the second case, $S$ has more samples than $S'$ - let number of samples in $S$ be $n$, and in $S'$ be $n-k_2$ and there $k_1$ samples in $S'$ not in $S$. As before we order the sets as: let $S_1=S_1'$ be the $n-k_2-k_1$ samples which are the same in both $S$ and $S'$. Let $S_2$ be the next $k_1$ samples in $S$, which correspond to $S_2'$, the rest of $k_1$ samples in $S'$ - the mapping from $S_2$ to $S_2'$ being $\iota$. Finally let $S_3$ be the rest of $k_2$ samples in $S$.
We first sample $b^{(1)} \sim \mu_{n,m}$ (during training).  
Let $b^{(1)}_2$ be the $m_2$ indices in $S_2$: replace these by the corresponding indices in $S_2'$ i.e. $\iota(b^{(1)} \cap S_2)$.
Let $b^{(1)}_3$ denote the sub-sampled indices which are in the last $k_2$ indices of $S$, and let $m_3 = \abs{b^{(1)}_3}$.  We re-sample $m_3$ indices as $b = \text{Unif}((S' \backslash \iota(b^{(1)} \cap S_2) \backslash b^{(1)}),m_3)$. 
Finally, define $b^{(2)} = \br{b^{(1)} \backslash \bc{b^{(1)} \cap S_2} \cup  \bc{\iota(b^{(1)} \cap S_2)}}\backslash  b^{(1)}_3 \cup b$. Let the distribution of $b^{(2)}$ produced in the above way be denoted as $\mu^{\text{edit}}_{n, m}$.
We now show that $(b^{(1)}, b^{(2)})$ is indeed a coupling of $\mu_{n,m}$ and $\mu_{n-k_2,m}$.

\begin{claim}
\label{claim:qsgd_coupling2}
  With the construction described above, we have that $b^{(1)} \sim \mu_{n,m}$ and $b^{(2)} \sim \mu^{\text{edit}}_{n, m} = \mu_{n-k_2,m}$.
\end{claim}
\begin{proof}[Proof of \cref{claim:qsgd_coupling2}]
$b^{(1)} \sim \mu_{n,m}$ follows trivially by construction. For the other part, let $E$ be a set of $m$ indices from $[n-k_2]$. 
Note that \emph{any} number of points in $E$ can arise due to re-sampling (i.e. when sampling $b'$), hence we need to consider all such possibilities - let $m_3$ be the number of indices in $E$ produced via re-sampling. Fixing one of ${m \choose m_3}$ combinations, the probability that it was re-sampled is $\frac{1}{{n-k_2-(m-m_3) \choose m_3}}$. From the rule of sum, the probability that \emph{any} $m_3$ sized set was produced via re-sampling is $\frac{{m \choose m_3}}{{n-k_2-(m-m_3) \choose m_3}}$. 
For each such set, it could arise from any of $m_3$ points from $k_2$, which gives us ${k_2 \choose m_3}$ possibilities. The probability of choosing any such set, when sampling $b^{(1)}$, is $\frac{{k_2 \choose m_3}}{{n \choose m}}$. We now combine these and apply the rule of sum on different choices of $m_3$, from $0$ to $m$. We get,

\begin{align*}
    \mu^{\text{edit}}_{n, m} (E) &= \sum_{m_3=0}^m \frac{{m \choose m_3}}{{n-k_2-(m-m_3) \choose m_3}} \cdot \frac{{k_2 \choose m_3}}{{n \choose m}} \\
    & = \frac{1}{{n \choose m}}\sum_{m_3=0}^m\frac{m!(n-k_2-m)!m_3! }{m_3!(m-m_3)! (n-k_2-m+m_3)!} {k_2 \choose m_3}\\
    & = \frac{1}{{n \choose m}}\sum_{m_3=0}^m \frac{m!(n-k_2-m)!}{(n-k_2)!} \cdot \frac{(n-k_2)!}{(m-m_3)! (n-K_2-(m-m_3))!}{k_2 \choose m_3}\\
    & = \frac{{n-k_2 \choose m}}{{n \choose m}} \sum_{m_3=0}^m {n-k_2 \choose m-m_3} {k_2 \choose m_3}  = \frac{{n-k_2 \choose m}}{{n \choose m}} \cdot {n \choose m} ={n-k_2 \choose m} = \mu_{n-k_2,m} 
\end{align*}
where the third last equality follows from Vandermonde's identity.
\end{proof}

\begin{algorithm}[!htpb]
    \caption{Batch unlearning for \emph{quantized-m-SGD}}
    \label{alg:unlearn_qsgd}
    \begin{algorithmic}[1]
        \REQUIRE{Edit request produces dataset $S'$ of $n+k_2$ points, with $k_1$ deletions and $k_1+k_2$ insertions; let $S_1, S_2$ and $S_1', S_2'$ and $S_3'$ be partitions of  $S$ and $S'$ respectively, as defined in ``Batch coupling"}
        \FOR{$t=1,2\ldots, T$}
        \STATE Load($\theta_t,\w_t, b_t, \g_t$)
        \STATE $b \sim \text{Unif}(S',m)$
        \STATE $m_3 = \abs{\bc{x \in S_3 \cap b}}$
        \STATE $b^{(2)}_1 \sim \text{Unif}(b_t,m_3)$
        \STATE $\g_t' = \g_t -  \frac{1}{m}\br{\sum\limits_{j \in b_t \cap S_2}\nabla f(\w_t,\z_j) +  \sum\limits_{j \in   \iota(b_t \cap S_2)} \nabla f(\w_t,\z_j) - \sum\limits_{j \in b^{(2)}_1}\nabla f(\w_t,\z_j) + \sum\limits_{j \in S_3 \cap b}\nabla f(\w_t,\z_j)} $
        \STATE $b^{(2)}_t = \br{b_t \backslash \bc{b_t \cap S_2} \cup  \bc{\iota(b_t \cap S_2)}}\backslash  b^{(2)}_1 \cup \bc{S_3' \cap b}$
        \STATE Save($ \g_t', b^{(2)}_t$)
        \IF{$Q_{\theta_t}\br{\g_t} \neq Q_{\theta_t}\br{\g_t'}$}
        \STATE \emph{quantized-m-SGD} \hfill{\textit{// Recompute on current dataset}}
        \BREAK
        \ENDIF
        \ENDFOR
    \end{algorithmic}
\end{algorithm}

We now state the main result about unlearning.

\begin{proposition}
\label{prop:unlearn-qsgd}
(Algorithm \ref{alg:qsgd}, Algorithm \ref{alg:unlearn_qsgd})  satisfies exact batch unlearning. 
Moreover, for $k$ batch edit requests, where the $i^{\text{th}}$ request comprises of $k^i_1$ deletions and $k^i_1+k^i_2$ insertions, or $k^i_1$ insertions and $k^i_1+k^i_2$ deletions, Algorithm \ref{alg:unlearn_qsgd} recomputes with probability at most $2\sum_{i=1}^k(k^i_1+k^i_2)\rho$.
\end{proposition}

\begin{proof}[Proof of \cref{prop:unlearn-qsgd}]
We consider one batch edit request of $k_1$ deletions and $k_1+k_2$ insertions (case 1) and $k_1$ insertions and $k_1+k_2$ deletions (case 2).
We have that applications of Claims \ref{claim:qsgd_coupling1} and \ref{claim:qsgd_coupling1} give us that mini-batches are transported, for cases 1 and 2 respectively.
Moreover, since we consider a trivial coupling of quantization randomness, we can consider it part of the (randomized) algorithmic map. Therefore, as in the proof of \cref{prop:unlearn-subsample-GD}, transportation of mini-batches suffices to give us that \cref{alg:unlearn_qsgd} satisfies exact unlearning.
Repeated application of the above generalizes it to arbitrary $k$ edits. 
We now proceed to bound the probability of recompute directly for a batch edit request.
For a fixed model $\w$, and a fixed iteration, we fix
the mini-batches $(b^{(1)},b^{(2)})$ such that $b^{(1)}$ and $b^{(2)}$ differ by $j$ indices.
From \cref{claim:quantize}, we have

\begin{align*}
     \mathbb{P}_{\theta,(b^{(1)},b^{(2)})}\left[Q_\theta(\g_{b^{(1)}}^S(\w)) \neq Q_\theta(\g_{b^{(2)}}^{S'}(\w)) | (b^{(1)},b^{(2)}) \text{ such that }b^{(1)}, b^{(2)}\text{ differ in }j \text{ indices} \right]\leq  \frac{4Gdj}{m\tau}
\end{align*}

We now integrate over the conditioning event. To do this, we need to compute the probability of the event that sampling $(b^{(1)},b^{(2)})$ generates $j$ differing indices - denote this as $p(j)$. 

Since we have two case for coupling constructions, we consider each one by one. We first look at the second case:
from construction of the coupling, it is easy to verify that $j$ differing indices can be produced when, for any $i$, $b^{(1)}$ samples $i$ elements from the $k_1$ differing items  and $j-i$ indices from the last $k_2$ indices, for any $i$ from $0$ to $j$.  Hence, by direct computation, we have

\begin{align*}
    p(j) = \sum_{i=0}^j\frac{{k_1 \choose i} {k_2 \choose j-i} {n-(k_1+k_2) \choose m-j}}{{n \choose m}} = \frac{{n-(k_1+k_2) \choose m-j}}{{n \choose m}}\sum_{i=0}^j{k_1 \choose i} {k_2 \choose j-i} = \frac{{n-(k_1+k_2) \choose m-j}{k_1+k_2 \choose j}}{{n \choose m}}
\end{align*}

where the last equality follows from Vandermonde's identity. Plugging this in the following, we have,
\begin{align*}
    &\mathbb{P}_{\theta,(b^{(1)},b^{(2)})}\left[Q_\theta(\g_{b^{(1)}}^S(\w)) \neq Q_\theta(\g_{b^{(2)}}^{S'}(\w))\right] \\
    &= \sum_{j=0}^{k_1+k_2} p(j)  \mathbb{P}_{\theta,(b^{(1)},b^{(2)})}\left[Q_\theta(\g_{b^{(1)}}^S(\w)) \neq Q_\theta(\g_{b_2}^{S'}(\w)) |b^{(1)}, b^{(2)}\text{ differ in }j \text{ indices} \right] \\
    &\leq \sum_{j=0}^{k_1+k_2}\frac{{n-(k_1+k_2) \choose m-j}{k_1+k_2 \choose j}}{{n \choose m}} \frac{4Gdj}{m\tau} 
    = \frac{m(k_1+k_2)}{n}\frac{4Gd}{ m\tau} \\
    & = \frac{4Gd(k_1+k_2)}{n\tau} \leq \frac{k\rho}{T}
\end{align*}
where the second equality is a consequence of Vandermonde's identity proved in \cref{lemma:tv-stability-qsgd} (Base case $k_2=0$) and the last inequality follows by plugging in $\tau = \frac{4GdT}{\rho n}$.

We now look at the first case (when $S$ is smaller than $S'$), which is slightly more involved. Let $i_1$ denote the number of indices in $b^{(1)} \cap S_1$,
and let $i_2$ be the number of indices in $b^{(2)} \cup S_3'$. Furthermore, since we resample $i_2$ indices from $b^{(1)}$, let $i_3$ be the number of indices from $S_1$ which are re-sampled. It can be verified that if $b^{(1)}$ and $b^{(2)}$ differ in $j$ indices, then we need to have $i_1+i_3=j$. This is because it can happen that both $i^{(2)}$ is large, but upon re-sampling, it chooses elements from $k_1$, which does not increase the number of different indices between  $b^{(1)}$ and $b^{(2)}$ Also, note that by construction $i_3\leq i_2\leq j$. Hence the probability $p(j)$, by direct computation is, 
\begin{align*}
    p(j) &= \sum_{i_1,i_2,i_3=0, i_1+i_3=j, i_3\leq i_2\leq j}^j 
    \frac{{k_1 \choose i_1}{n-k_1\choose m-i_1}}{{n \choose m}}
    \cdot \frac{{k_2 \choose i_2} {n \choose m-i_2}}{{n+k_2 \choose m}} \cdot 
    \frac{{i_1 \choose i_2-i_3}}{{m \choose i_2}} \\
    & = \sum_{i_1=0}^j\sum_{i_2=j-i_1}^j 
    \frac{{k_1 \choose i_1}{n-k_1\choose m-i_1}}{{n \choose m}}
    \cdot \frac{{k_2 \choose i_2} {n \choose m-i_2}}{{n+k_2 \choose m}} \cdot 
    \frac{{i_1 \choose i_2-(j-i_1)}}{{m \choose i_2}}
\end{align*}
where in the second equality, we substituted $i_3 = j-i_1$. We now claim that $p(j) = \frac{{n-k_1 \choose m-j}{k_1+k_2 \choose j}}{{n+k_2 \choose m}}$, which we will argue via a double counting argument. Note that it suffices to show that 
$ \sum_{i_1=0}^j\sum_{i_2=j-i_1}^j 
    {k_1 \choose i_1}{n-k_1\choose m-i_1}
    \cdot {k_2 \choose i_2} {n \choose m-i_2}
    \frac{{i_1 \choose i_2-(j-i_1)}}{{m \choose i_2}} = {n-k_1 \choose m-j}{k_1+k_2 \choose j}{n \choose m}$
. Consider set $A$ of $n+k_2$ elements, composed of $A_1$ of $n-k_1$, $A_2$ of $k_1$ and $A_3$ of $k_2$ elements, and a $B$ of $n$ elements, composed of $B_1$ of $n-k_1$ and $B_2$ of $k_1$ elements. Note that the expression ${n-k_1 \choose m-j}{k_1+k_2 \choose j}{n \choose m}$ is the size of number of combinations of $2m$ elements, $m$ each from $A$ and $B$ such that the number of elements from $A_2 \cup A_3$ is $j$.
We will show that the other expression also counts this set, via basic combinatorial rules. For this, consider combinations of $m$ elements from $B$ $A_3$ and  such that we have $i_2$ elements from $A_3$ and the rest $m-i_2$ from $B$. Also, consider combinations of $m$ elements from $A_1 \cup A_2$ which consists of $i_1$ elements from $A_2$ the rest from $A_1$. We now modify these as follows, out of $m$ elements from $A$, select $i_2$ elements and replace thse from elements from $A_3$ - not that if it turns out that out of $i_2$ selected, $j-i_1$ are from $A_1$, then the number of elements from $A_2 \cup A_3$ after replacement becomes exactly $j$. However, also note that for each such combination arising, there are ${m \choose i_2}$ combinations of samples from $A_1$ and $A_2$, which give the same \emph{final} combination after replacement. Hence, we need to apply the rule of division, so as not to repeatedly count the same combination. Finally,
using the rule of sum to consider all possible values of $i_1$ and $i_2$ retrieves the expression $ \sum_{i_1=0}^j\sum_{i_2=j-i_1}^j 
    {k_1 \choose i_1}{n-k_1\choose m-i_1}
    \cdot {k_2 \choose i_2} {n \choose m-i_2}
    \frac{{i_1 \choose i_2-(j-i_1)}}{{m \choose i_2}} = {n-k_1 \choose m-j}{k_1+k_2 \choose j}{n \choose m}$ and completes the argument.

We again plug in the above in the following expression to get,
\begin{align*}
     &\mathbb{P}_{\theta,(b^{(1)},b^{(2)})}\left[Q_\theta(\g_{b^{(1)}}^S(\w)) \neq Q_\theta(\g_{b^{(2)}}^{S'}(\w))\right] \\
    &= \sum_{j=0}^{k_1+k_2} p(j)  \mathbb{P}_{\theta,(b^{(1)},b^{(2)})}\left[Q_\theta(\g_{b^{(1)}}^S(\w)) \neq Q_\theta(\g_{b_2}^{S'}(\w)) |b^{(1)}, b^{(2)}\text{ differ in }j \text{ indices} \right] \\
    &\leq \sum_{j=0}^{k_1+k_2}\frac{{n-k_1 \choose m-j}{k_1+k_2 \choose j}}{{n+k_2 \choose m}} \frac{4Gdj}{m\tau} 
    = \frac{m(k_1+k_2)}{n}\frac{4Gd}{ m\tau} \\
    & = \frac{4Gd(k_1+k_2)}{n\tau} 
\end{align*}
where the second equality is again a consequence of Vandermonde's identity as in \cref{lemma:tv-stability-qsgd}, and the last inequality follows by plugging in $\tau = \frac{4GdT}{\rho n}$. Finally, we condition on the iterates till iteration $t$, which gives us the conditional probability of the iterates differing at iteration $t$ is at most $\frac{(k_1+k_2)\rho}{T}$. Taking a union bound over all $T$ iterations 
gives us that probability is at most $(k_1+k_2)\rho$. Finally, we extend it to $k$ edit request, by using the fact, by assumption than the number of data points at any point in the stream is between $\frac{n}{2}$ and $2n$. This, with the result for one edit request, directly give us the probability to recompute is at most $2\sum_{i=1}^k(k^i_1+k^i_2)\rho$.
\end{proof}

We now state and prove the main result.

\begin{theorem}
\label{thm:main-thm-qsgd}
Let $f(\cdot, \z)$ be an $L$-smooth $G$-Lipschitz convex function $\forall \ \z$.
For any $\frac{1}{n} \leq \rho < \infty$, using \cref{alg:qsgd} as the learning algorithm and \cref{alg:unlearn_qsgd} as its unlearning algorithm, then given a stream of batch edit requests,
\begin{enumerate}
    \item Satisfies exact batch unlearning at every point in the stream.
    \item At time $i$ in the stream of edit requests, outputs $\hat \w_{S^i}$, such that its excess empirical risk bounded as,
    \begin{align*}
    \E{\hat F_S(\hat \w_{S^i}) - \hat F_{S}(\w_{S^i}^*)} \lesssim
    \br{\frac{\sqrt{L}GD^2 d^{3/2}}{\rho n}}^{2/3}
\end{align*}

\item For $k$ batch edit requests, where the $i^{\text{th}}$ request comprises of $k^i_1$ deletions and $k^i_1+k^i_2$ insertions, or $k^i_1$ insertions and $k^i_1+k^i_2$ deletions, the  expected total unlearning runtime is \\$O(\max\bc{\min\bc{\rho,1} \sum_{i=1}^k (k_1^1+k_2^i) \cdot \text{Training time},k})$  

\end{enumerate}
\end{theorem}
\begin{proof}[Proof of \cref{thm:main-thm-qsgd}]
The first and the second claims follow from \cref{prop:unlearn-qsgd} and \cref{prop:upper-bound-qsgd} respectively combined with the assumption that the number of samples at every point in the stream is between $\frac{n}{2}$ and $2n$.
Finally, as in the the proof \cref{claim:runtime_noisy_A_SGD} for runtime \emph{noisy-m-A-SGD}, we can use the same data-structures together with the fact the quantization operation takes $O(d)$ time, to get that the claimed runtime. These together finish the proof of \cref{thm:main-thm-qsgd}.
\end{proof}

%% file: sections/lowerboundsempiricalrisk.tex
\section{Lower bounds on excess empirical risk}
\label{sec:lower-bounds}

Give a convex function $f(\cdot, \z)$, we consider empirical risk minimization on a dataset of $n$ points.
We assume $f(\cdot, \z)$ is $1$-Lipschitz for all $\z$, and diam$(\cW) \leq 1$. This is only for simplification as the bounds scale naturally with these constants, as discussed in \cite{bassily2014private}. We look at algorithms, which given two datasets $S$ and $S'$ of size $n$ differing by one point,  \emph{disagree} only on a set of measure at most an $\rho$. 

We have from the optimal transport connection that this requirement is equivalent to the total variation distance being at most $\rho$. We want to understand then what is the lower bound on excess empirical risk:
 \begin{align*}
  \underset{\Delta(S,S')=1}{\sup} \mathbb{P}\left[\cA(S') \neq \cA(S)\right] \leq \rho  &\iff  \underset{\Delta(S,S')=1}{\sup}\text{TV}(\cA(S),\cA(S')) \leq \rho  \\
  &\implies
  \E{\left[\text{excess empirical risk}\right]} \geq \alpha(\rho,n,d)
\end{align*}

We focus on proving the implication. \cite{bassily2014private} gave lower bounds on accuracy for DP algorithms by providing a reduction to computing mean of the dataset. We present and give the proof of the reduction, adapted to our context, for completeness. The reduction is that if we have a TV-stable algorithm for empirical risk minimization for a particular $f$ with some accuracy, then we have a TV-stable algorithm for mean computation problem  with \emph{certain} accuracy.
We will look at mean computation problem over datasets with norm of the mean being $\Theta(M)$, for some given $M$. Let $\mu(S) = \frac{1}{n}\sum_{j=1}^n\z_i$ denote the mean of dataset $S=\bc{\z_1,\z_2,\cdots \z_n}$.

Let the optimal accuracy of such a mean computation problem be denoted as follows:

\begin{align*}
    \alpha_{\text{mean}}^2(n,\rho, d,M):=\min_{\cA: \rho\text{-TV-stable}} \max_{\substack{S=\{\z_i\}_{i\in [n]}: \norm{\z_i}\leq 1,\\ M/2\leq \norm{\mu(S)}\leq 2M}} \bbE_\cA \norm{\cA(S)-\frac{1}{n}\sum_{i=1}^n\z_i}^2
\end{align*}

\begin{proposition}
\label{prop:lower_bound_reduction}
For any $\rho$-TV stable algorithm $\cA$, there exists a $1$-Lipschitz convex function $f$, a constraint set $\cW$ with diameter($\cW)\leq 1$ and a dataset $S$ of $n$ data point such that 
$$\hat F_S(\cA(S)) - \hat F_S(\w^*) \geq  \max_M\bc{\frac{\alpha_{\text{mean}}^2(n,\rho, d,M)}{2M}}$$
\end{proposition}

\begin{proof}[Proof of Proposition \ref{prop:lower_bound_reduction}]
We follow the proof in \cite{bassily2014private}.
Consider dataset $S=\bc{\z_1,\z_2,\ldots, \z_n}$, $\z_i \in \bc{-\frac{1}{\sqrt{d}},\frac{1}{\sqrt{d}}}^d$ - the dataset is therefore constrained to lie in the unit Euclidean ball. Consider the following function $f(\w,\z) = -\ip{\w}{\z}$ with the constraint set $\cW$ being the unit Euclidean ball. It is easy to see that $f(\cdot,\z)$ is $1$-Lipschitz for all $\z$. The empirical risk becomes $\hat F_S(\w) = -\ip{\w}{\frac{1}{n}\sum_{i=1}^n\z_i}$, the minimum of which over the unit ball is $\w_S^* = \frac{\frac{1}{n}\sum_{i=1}^n\z_i}{\norm{\frac{1}{n}\sum_{i=1}^n\z_i}}$. 

Given an algorithm $\cA$ for empirical risk minimization, let the reduced mean estimate be $\hat \mu(S) = \norm{\frac{1}{n}\sum_{i=1}^n\z_i} \cA(S)$. The accuracy (mean-squared error) of $\hat \mu$ is,
\begin{align*}
    \norm{\hat \mu(S) -  \mu(S)}^2 &= \norm{\norm{\frac{1}{n}\sum_{i=1}^n\z_i} \cA(S) - \frac{1}{n}\sum_{i=1}^n \z_i}^2
     = \norm{\norm{\frac{1}{n}\sum_{i=1}^n\z_i} \br{\cA(S) - \frac{\frac{1}{n}\sum_{i=1}^n \z_i}{\norm{\frac{1}{n}\sum_{i=1}^n\z_i}}}}^2 \\
    & = \norm{\frac{1}{n}\sum_{i=1}^n\z_i}^2  \norm{\cA(S) - \w_S^*}^2
     \leq  \norm{\frac{1}{n}\sum_{i=1}^n\z_i} 2\br{\hat F_S(\cA(S)) - \hat F_S(\w^*)}
\end{align*}
where the last inequality follows using the following computation, wherein we use the fact all data point are in the unit ball.
\begin{align*}
    \norm{\frac{1}{n}\sum_{i=1}^n\z_i}^2 \norm{\cA(S) - \w_S^*}^2 &\leq \norm{\frac{1}{n}\sum_{i=1}^n\z_i} 2\br{1 - \ip{\cA(S)}{\w_S^*}} \\&=  2\br{\norm{\frac{1}{n}\sum_{i=1}^n\z_i} - \norm{\frac{1}{n}\sum_{i=1}^n\z_i}\ip{\cA(S)}{\frac{\frac{1}{n}\sum_{i=1}^n\z_i}{\norm{\frac{1}{n}\sum_{i=1}^n\z_i}}}} \\
    & = 2\br{\ip{\w^*_S}{\frac{1}{n}\sum_{i=1}^n\z_i} - \ip{\cA(S)}{\frac{1}{n}\sum_{i=1}^n\z_i}} \\&= 2(\hat F_S(\cA(S)) - \hat F_S(\w^*)
\end{align*}
We therefore get,
$$\hat F_S(\cA(S)) - \hat F_S(\w^*) \geq \frac{1}{2\norm{\frac{1}{n}\sum_{i=1}^n\z_i}}\norm{\hat \mu(S) -  \mu(S)}^2$$

There are two things left to show: a bound on $\frac{1}{2\norm{\frac{1}{n}\sum_{i=1}^n\z_i}}\norm{\hat \mu(S) -  \mu(S)}^2$ and show that 
the reduced algorithm $\hat \mu(S)$ is also $\rho$-TV stable. 
We proceed with the latter:
note that $\hat \mu(S) =\norm{\frac{1}{n}\sum_{i=1}^n\z_i} \cA(S)$. However the term $\norm{\frac{1}{n}\sum_{i=1}^n\z_i}$ depends on the dataset, and even if, for a neighbouring dataset $S'$, $\cA(S)$ and  $\cA(S')$ are $\rho$-close in total variation, this data dependent scaling can potentially increase the distance. However, if instead we define $\hat \mu(S) =M \cA(S)$, where $M$ is a constant, then 
it is indeed $\rho$ TV stable. Moreover, for reasonable values 
of $M$, the there exists dataset for which $\norm{\frac{1}{n}\sum_{i=1}^n\z_i} = \Theta(M)$.
Finally, note that by definition, $\norm{\hat \mu(S) -  \mu(S)}^2 \geq \alpha_{\text{mean}}^2(n,\rho, d,M)$. Taking a max over all $M$ gives us the desired statement.
\end{proof}

\subsection{Lower bound for mean computation}

In this section, we look at the problem of mean computation with TV stability constraint.
Note that to establish lower bounds on excess empirical risk, we need to look at mean computation over data sets with means between $M/2$ and $2M$, for a given $M$. However, we will see the mean computation even over the unit ball has same accuracy convex ERM. We will therefore establish lower bounds for the general mean computation problem, but the construction will use datasets with means $\Theta(M)$ for certain values of $M$.
Given a dataset $S=\bc{\x_1,\x_2,\ldots, \x_n}$, where $\x_i \in \bc{-\frac{1}{\sqrt{d}},\frac{1}{\sqrt{d}}}^d$ for all $i,$ the task is to compute the mean $\mu(S) = \frac{1}{n}\sum_{j=1}\x_j$, while ensuring that the procedure is $\rho$-TV-stable. 
This task is often considered in the differential privacy literature, however with the data points being $\x_j \in \bc{0,1}^d$. The mean computation task then corresponds to releasing all one-way marginals of the database. Since we want to consider data points which lie inside the Euclidean ball, we therefore scale it accordingly. Given an algorithm $\cA(S)$, the accuracy is defined as mean-squared error: $\alpha^2 :=  \alpha^2_{\text{mean}}(n,\rho,d) =\E{\norm{\cA(S)-\mu(S)}^2}$ where the expectation is over the randomization of the algorithm.

We first describe two algorithms for this problem and give upper bounds.

\paragraph{\emph{Subsample-mean}:} Consider an algorithm which sub-samples a $\rho$-fraction of the dataset and outputs the mean on it. 
\begin{claim}
\label{claim:mean_computation_subsample_upper_bound}
The \emph{Subsample-mean} procedure satisfies $\rho$-TV-stability and has accuracy  $\alpha^2 \leq O\br{\frac{1}{\rho n}}$.
\end{claim}
\begin{proof}[Proof of \cref{claim:mean_computation_subsample_upper_bound}]
The $\rho$-TV stability claim follows since TV distance is witnessed by the event that a differing sample is sub-sampled, which happens with probability $\frac{\rho n}{n} =\rho$.
The proof of accuracy follows from the proof of \cref{prop:upper_bound_subsample}, wherein we computed the gradient on uniformly sub-sampled $m$ out of $n$ points -  we showed that the mean on sub-sampled points is an unbiased estimate of the average gradient. Furthermore, since the gradients were bounded as well, the expected accuracy of mean computation is the same as the variance of gradient computation, which we derived to be $O\br{\frac{1}{m}} = O\br{\frac{1}{\rho n}}$.
\end{proof}

\paragraph{\emph{Noisy-mean}:} The algorithm computes the mean and adds $N(0,\sigma^2 I)$ noise to it with $\sigma^2 = \frac{C}{n^2\rho^2}$, where $C$ is an appropriate universal constant. 
 
\begin{claim}
\label{claim:mean_computation_noisy_upper_bound}
The \emph{Noisy-mean} procedure satisfies $\rho$-TV-stability and has accuracy $\alpha^2 \leq O\br{\frac{d}{(\rho n)^2}}$
\end{claim}
\begin{proof}[Proof of \cref{claim:mean_computation_noisy_upper_bound}]
  Since the difference in means of two datasets, in norm, is at most $\frac{2}{n}$, the outputs are two multivariate Gaussians with variance $\sigma^2$ and means separated by $\frac{2}{n}$. From \cite{devroye2018total}, the total variation distance between such Gaussian sis at most $O\br{\frac{1}{n\sigma}} = \rho$. For the accuracy, we have $\alpha^2 = \E{\norm{\mu(S) + \xi - \mu(S)}^2} = \E{\norm{\xi}^2} = d \sigma^2 = O\br{\frac{d}{n^2\rho^2}}$.
\end{proof}

If the above procedures are optimal, then we expect a lower bound of $\alpha^2 \gtrsim \min\bc{\frac{1}{\rho n}, \frac{d}{(\rho n)^2}}$. Equivalently, for a fixed accuracy $\alpha$, we expect a sample complexity lower bound of $n\gtrsim \min\bc{\frac{1}{\rho \alpha^2},\frac{\sqrt{d}}{\rho \alpha}}$.

\subsubsection{Lower bound I}
In this section, we give a $\Omega\br{\frac{1}{\alpha \rho}}$ lower bound on sample complexity.
The key ingredient is the following result, where the proof is based on a simple reduction argument.

\begin{proposition}
\label{prop:lower_bound_dp_reduction}
Suppose there exists a $\rho$-TV-stable algorithm such that for any dataset of $n$ points, it achieves an accuracy of $ \alpha$. Then there exists a $0.1$-TV stable algorithm such the for any dataset of size $\lceil 100n \alpha \rho \rceil$, it  achieves a $0.1$-accuracy.
\end{proposition}

\begin{proof}[Proof of \cref{prop:lower_bound_dp_reduction}]
Let $n' =\lceil 100n \alpha \rho \rceil$.  Consider a dataset $S'$ of $n'$ points. We construct a dataset $S$ of $n$ points by concatenating $K = \lceil 0.1/\rho \rceil$ copies of $S'$ followed by $\lceil\frac{n-Kn'}{2}\rceil$ copies of a constant sample, all ones $(\frac{1}{\sqrt{d}},...,\frac{1}{\sqrt{d}})$ and $\lceil\frac{n-Kn'}{2}\rceil$ copies of a constant sample, all ones $(-\frac{1}{\sqrt{d}},...,-\frac{1}{\sqrt{d}})$.

Consider the algorithm wherein we compute the stable-mean on $D'$ by $\cA'$, defined as computing the stable-mean on $D$ using $\cA$ and adjusting: $$\cA'(D') = \frac{n}{Kn'}\cA(D)$$

Let $\tilde S'$ be a neighbouring dataset of $S'$. 
By construction, note that $S$ and $\tilde S$ differ by $K$ samples. Furthermore, since the algorithm $\cA$ on $D$ is $\rho$-TV stable, on $K$-neighbouring datasets, it is $K\rho$ = $0.1$-TV. This establish the stability part of claim. 
The accuracy, by direct computation is $\E{\norm{\cA'(D')-\mu(D')}^2} = \frac{n^2}{K^2n'^2}\E{\norm{\cA(D)-\mu(D)}^2} \leq (0.1)^2$.
\end{proof}

\begin{theorem}
\label{thm:lower_bound11}
For the $d$-dimensional mean computation problem over the Euclidean ball, there exists a dataset $S$ of $n$ samples with mean $\norm{\mu(S)}=\Theta\br{\frac{1}{\rho n}}$ such that
the accuracy of any $\rho$ TV stable algorithm is $\alpha \geq \Omega\br{\frac{1}{\rho n}}$.
\end{theorem}
\begin{proof}[Proof of \cref{thm:lower_bound11}]
Even for accuracy $\alpha_0= 0.1$ accuracy and $\rho_0=0.1$ stability, we need at least one sample. 
Hence, using \cref{prop:lower_bound_reduction}, we get that sample complexity is $n\geq \Omega\br{\frac{1}{\rho \alpha}}$, which equivalently gives the claimed accuracy lower bound.
Note that for this one-sample dataset $S'$, $\norm{\mu(S')} =1$. 
Finally, from the reduction in \cref{prop:lower_bound_reduction}, the mean of dataset $S$ becomes $\norm{\mu(S)} = \lceil\frac{0.1}{\rho n}\rceil$, which finishes the proof.
\end{proof}

\subsubsection{Lower bound II}
In this section, we will prove the $\Omega\br{\frac{1}{\alpha^2 \rho}}$ lower bound. We first introduce a technical assumption.

\begin{assumption}
\label{ass:lb2}
For any dataset $S$, we assume that the probability distribution $\cA(S)$ is defined over the unit Euclidean ball, is absolutely continuous with respect to the uniform measure (in the unit Euclidean ball) and its
probability density function, with respect to the uniform measure, is bounded by $K$ in absolute value. 
\end{assumption}

As a remark, the above assumption can also be stated with respect to the Lebesgue measure, but then we would get a scaling of $\frac{\pi^{d/2}}{\Gamma\br{1+\frac{d}{2}}}$, which is the Lebesgue volume of the $B_d(0,1)$, to some of our terms. In order to simplify, we therefore use the uniform measure. 

\begin{theorem}
\label{thm:lower_bound2}
Let $n\geq 72, \alpha \leq \frac{1}{4}$ and $\frac{1}{n} \leq \rho \leq  \frac{1}{4}$. 
Let $\cA$ be any $\rho$-TV-stable algorithm satisfying Assumption \ref{ass:lb2} with $K\leq 2^d$.
For large enough dimension $d$, there exists a dataset $S$ of $n$ points with $\norm{\mu(S)} = \Theta\br{\frac{1}{\sqrt{\rho n}}}$ such that accuracy is lower bounded as $\alpha \geq \Omega\br{\frac{1}{\sqrt{\rho n}}}$.
\end{theorem}

\begin{proof}[Proof of \cref{thm:lower_bound2}]
We will prove the result by contradiction. 
Let "$\text{Vol}$" of a set refer to its volume with respect to the uniform measure on the unit ball. Consider the following high-dimensional setup. 
Consider a dataset $S$ (or $S^0$) which mean $\mu(S)$ such that $\norm{\mu(S)}=\Theta\br{\frac{1}{\sqrt{\rho n}}}$. It is easy to construct such datasets by considering points such that sum of $n-\lceil\sqrt{\frac{n}{\rho}}\rceil$ points is $0$ and the rest of points is the same point repeated - this uses the assumption that $\rho \geq \frac{1}{n}$.
Now consider neighbouring datasets $S^i$'s, $i\in [n]$ such that the means of $S^i$'s are all $\frac{1}{n}$ far from that of $S$, in norm.
We also need that the means of any two datasets $\norm{\mu(S^i)-\mu(S^j)} \geq \frac{1}{2n}$ for $i,j=0$ to $n$ and $i\neq j$. It is easy to see the existence of such datasets, by considering the means of $S^i$'s in \textit{near} orthogonal directions to that of $S$, which is possible when $d$ is large enough.

Suppose the algorithm $\cA$ has expected error $\alpha^2$ i.e. $\E{\norm{\cA(S)-\mu(S)}}^2 \leq \alpha^2$ with $ 72 \leq n \leq \frac{1}{\alpha^2\rho}$. 
 Consider $B_d\br{\mu(S),\frac{1}{K^{1/d}}}$, the $d$ dimensional Euclidean ball centered at $\mu(S)$ of radius $\frac{1}{K^{1/d}}$.
From Markov's inequality, we have that $\P{\cA(S)\not \in B_d\br{\mu(S),\frac{1}{K^{1/d}}}} \leq \P{\cA(S)\not \in B_d\br{\mu(S),\frac{1}{2}}} = \P{\norm{\cA(S)-\mu(S)}^2 \geq \frac{1}{4}}\leq \frac{\E{\norm{\cA(S)-\mu(S)}}^2}{(1/4)} \leq 4\alpha^2$, where in the first inequality, we used the assumption $K\leq 2^d$. Therefore, we have $\P{\cA(S) \in B_d\br{\mu(S),\frac{1}{K^{1/d}}}} \geq 1-4\alpha^2$. 

We now setup some additional notation. Let $A_i$ denote the set $B_d\br{\mu(S^i),\frac{1}{K^{1/d}}} \backslash \br{ \cup_{j=0,j \neq i}^n B_d\br{\mu(S^j),\frac{1}{K^{1/d}}}}$ i.e. the region in the ball $B_d(\mu(S^i)$ which is not contained in any of the other balls.  
Let $B_{ij}$ denote the region of intersection between $B_d\br{\mu(S^i),\frac{1}{K^{1/d}}}$ and $B_d\br{\mu(S^j),\frac{1}{K^{1/d}}}$ where $i\neq j$ and $i$ and $j$ go from $0$ to $n$. Note that set $B_{ij}$ is constituted of two spherical caps. By construction the centers of the intersecting spheres are at least $\frac{1}{2n}$ apart. 
To study the properties of such a set, we define \textbf{cap} as the region in a $d$ dimensional sphere of radius $\frac{1}{K^{1/d}}$ which intersects with another sphere of the same radius but with centers being apart by $1/2n$.
From known results \cite{chudnov1991game}, the volume of \textbf{cap} is asymptotic to $\text{Vol}\br{B_d\br{0,\frac{1}{K^{1/d}}}}\br{1-\Phi\br{\frac{\sqrt{d}}{2n}}}$ as $d \rightarrow \infty$ 
where $\Phi$ is the cumulative distribution function of a standard normal random variable. We therefore have that $ \lim_{d \rightarrow \infty } \text{Vol(\textbf{cap})} \sim \lim_{d \rightarrow \infty }\text{Vol}\br{B_d\br{0,\frac{1}{K^{1/d}}}}\br{1-\Phi\br{\frac{\sqrt{d}}{2n}}} = 0$. 
Furthermore, using Assumption \ref{ass:lb2}, we have $\P{\cA(S) \in \text{cap}} \leq K \text{Vol(\textbf{cap})}\sim K\text{Vol}\br{B_d\br{0,\frac{1}{K^{1/d}}}}\br{1-\Phi\br{\frac{\sqrt{d}}{n}}} \leq K \br{\frac{1}{K^{1/d}}}^d\br{1-\Phi\br{\frac{\sqrt{d}}{n}}} = \br{1-\Phi\br{\frac{\sqrt{d}}{n}}}$ as $d\rightarrow \infty$.
Since $\Phi(t) = \mathbb{P}_{g \sim \cN(0,1)}[g \leq t]$, we have that $1 - \Phi(t) = \mathbb{P}_{g \sim \cN(0,1)}[g  > t] \leq \frac{e^{-t^2/2}}{\sqrt{2\pi}t} $ where the last inequality follows from standard bounds on tails of normal distribution (See Proposition 2.1.2 in \cite{vershynin2018high}). Therefore, we have $\P{\cA(S) \in \textbf{cap}} \lesssim \frac{ne^{-d/4n^2}}{\sqrt{d}}$. For constant $\epsilon>0$, choosing $d \gtrsim 4n^2\ln{\br{n/\epsilon}}$ ensures that $\P{\cA(S) \in \textbf{cap}} \leq \frac{\epsilon}{2n}$ for large enough $n$ (to be specified later). Since $B_{ij}$ is made up of two conjoined caps, this gives us that for any $j=0$ to $n$ and $i\neq j$, we have that $\P{\cA(S^i)\in B_{ij}} \leq \frac{\epsilon}{n}$. 
Finally, we look at $A_i$'s by removing the mass of all $B_{ij}$'s, and using a union bound, we get that $\P{\cA(S^i) \in A_i} = \P{\cA(S^i) \in B_d\br{\mu(S^i),\frac{1}{K^{1/d}}}} - \P{ \cA(S^i) \in \cup_{j=0,j\neq i}^n B_{ij}} \geq 1-4\alpha^2 - \sum_{j=0,j\neq i}^{n} \P{\cA(S^i) \in B_{ij}} \geq 1-4\alpha^2 -\epsilon \geq \frac{1}{2}-4\alpha^2 $ where the last inequality holds for $\epsilon \leq \frac{1}{2}$. We now evaluate how large $n$ we need for this regime of $\epsilon$: recall that we set $d\gtrsim 4n^2\ln{\br{n/\epsilon}}$, this gives 
$\frac{ne^{-d/4n^2}}{\sqrt{d}} \leq \frac{\epsilon}{2n\sqrt{ \ln{\br{n/\epsilon}}}}$. We want the right hand side to be at most $\frac{\epsilon}{2n}$ for $\epsilon \leq 1/2$. Plugging in this worst-case value of $\epsilon$, we get the condition $\ln{(2n)} \geq 1$ which holds for any $n\geq 1.4$ and therefore is valid by our assumption of $n$.

We now use the fact that $A_i's$ are disjoint by construction. Therefore the total measure of $\cA(S)$ on union of $A_i's$ is at most $1$ i.e $ \P{\cA(S) \in \cup_{i=1}^n A_i} = \sum_{i=1}^n \P{\cA(S) \in A_i} \leq 1$. Furthermore, since $\cA(S)$ is $\rho$-TV stable, we have that $\P{\cA(S) \in A_i} \geq  \P{\cA(S^i) \in A_i} - \rho$.
Combining this with the previous analysis which gives a lower bound on $\P{\cA(S^i) \in A_i}$ yields
\begin{align}
\label{eqn:lb2}
  n\br{\frac{1}{2}-4\alpha^2-\rho} \leq  \sum_{i=1}^n\P{\cA(S^i) \in A_i} - \rho \leq \sum_{i=1}^n \P{\cA(S) \in A_i} \leq 1
\end{align}

We now proceed in two cases:

\textbf{Case 1:} Suppose $72 \leq n \leq \frac{17}{4\alpha^2}$. The latter condition gives us that $4\alpha^2 \leq \frac{17}{n}$. Using \cref{eqn:lb2} gives us $n(1/2-4\alpha^2-\rho)\leq 1 \iff 4\alpha^2 \geq \frac{1}{2}-\frac{1}{n}-\rho$. Upper bounding $4\alpha^2$ by $\frac{17}{n}$ gives us that $\frac{18}{n} \geq \frac{1}{2} - \rho \iff n\leq \frac{18}{(1/2-\rho)} \leq 72$ where in the last inequality we used $\rho \leq 1/4$. This gives us a contradiction.

\textbf{Case 2:} Suppose $ \frac{17}{4\alpha^2} \leq n $. We again start with \cref{eqn:lb2} which gives us $n(1/2-4\alpha^2-\rho)\leq 1 \iff \rho \geq \frac{1}{2}-\frac{1}{n}-4\alpha^2$. We want to prove the right hand side is at least $\frac{1}{n \alpha^2}$, which would give us that $n\geq \frac{1}{\rho \alpha^2}$. Suppose this is not true i.e. $\frac{1}{2} - \frac{1}{n} -4\alpha^2 \leq \frac{1}{n \alpha^2} \iff \frac{n-2-8\alpha^2n}{2n} \leq \frac{1}{n\alpha^2} \iff \alpha^2n(1-8\alpha^2) \leq 2(1+\alpha^2) \iff n \leq \frac{2(1+\alpha^2)}{\alpha^2(1-8\alpha^2)}$. Finally using the fact that $\alpha \leq \frac{1}{4}$ gives that $n \leq \frac{2(1+1/16)}{\alpha^2(1-8/16)} \leq \frac{17}{4\alpha^2}$ which yields a contradiction. 

Hence, we see that with $n>72$ samples and accuracy $\alpha^2$, we have established  that $n\geq \frac{1}{\rho \alpha^2}$ and so $\alpha \geq \frac{1}{\sqrt{\rho n}}$. 
\end{proof}

%% file: sections/populationrisk.tex
\section{Excess population risk bounds}
\label{sec:pop-risk}

The goal in machine learning is (population) risk minimization.
The population risk of $\w$, denoted by $F(\w)$ is defined as $F(\w) := \Eu{\z \sim \cD}{f(\w,\z)}$, where $\cD$ is an unknown probability distribution over data points. Analogously, given an output of algorithm $\cA$ on dataset $S =\bc{\z_i}_i$ where $\z_i \sim \cD$ i.i.d., denoted as $\cA(S)$, we will give guarantees on the \emph{expected excess population risk}, defined as $\E{F(\cA(S))- F(\w^*)}$, where $\w^*$ is the population risk minimizer: $\w^* \in \arg \min_{\w \in \cW}  F(\w)$, and the expectation is taken with respect to randomness in  algorithm $\cA$ as well as sampling $S$.

\subsection{Upper bounds}
In this section, we will bound the expected excess population risk  appealing to connections between algorithmic stability and generalization \citep{bousquet2002stability}. We first define uniform stability.

\begin{definition}[Uniform stability]
Let $\cA:\cZ^n \rightarrow \cW$ be an algorithm and $\cA(S)$ denotes its output on dataset $S$. We say that $\cA$ is $\epsilon_{\text{stable}}(n)$-uniformly stable if for any datasets $S$ and $S'$ of $n$ points such that they differ by one data point (i.e. $\Delta(S,S')= 2$), we have $\sup_{z\in \cZ}\mathbb{E}_{\cA}\left[f(\cA(S),z)- f(\cA(S'),z)\right]\leq \epsilon_{\text{stable}}(n)$
\end{definition}

A classical result \citep{bousquet2002stability} shows that expected excess population risk is at most uniform stability + expected excess empirical risk: i.e. any $\w \in \cW$, we have
\begin{align*}
    \mathbb{E}\left[F(\cA(S)) - F(\w)\right] \leq 
    \epsilon_{\text{stable}}(n)
    + \mathbb{E}\left[\hat F_S(\cA(S)) - \hat F_S(\w) \right]
\end{align*}

\begin{theorem}[Upper bound]
\label{thm:pop-risk-upper-bound}
There exists a $\rho$ TV stable algorithm, such that for any function $f(\cdot, \z)$ which is $L$-smooth $G$-Lipschitz convex $\forall \ \z$ and any dataset $S$ of $n$ points, it outputs $\hat \w_S$ which satisfies the following. 
\begin{align*}
    \E{ F(\hat \w_S) -  F(\w^*)} \lesssim \frac{GD}{\sqrt{n}}+ GD\min\bc{\frac{1}{\sqrt{\rho n}}, \frac{\sqrt{d}}{\rho n}}
\end{align*}
\end{theorem}

\begin{proof}[Proof of \cref{thm:pop-risk-upper-bound}]
We use \emph{sub-sample-GD} (\cref{alg:sub-sample-GD}) and \emph{noisy-m-SGD} (\cref{alg:noisy-m-sgd}). From Lemma 3.2 in \cite{bassily2019private}, we have that $\epsilon_{\text{stable}}(n) \leq \frac{G^2\eta T}{n}$. From \cref{prop:upper_bound_subsample}, we set $\eta = \min \bc{\frac{1}{2L},\frac{D}{G\sqrt{T}}}$, and $T= \frac{DL\sqrt{\rho n}}{G}$. We therefore have $\epsilon_{\text{stable}}(n) \leq   \frac{G^2T}{2Ln} = \frac{GD\sqrt{\rho}}{\sqrt{\rho n}}$. Using the excess empirical risk bound from \cref{prop:upper_bound_subsample}, and the fact that $\rho\leq 1$, the excess population risk is bounded as,
\begin{align*}
     \E{ F(\hat \w_S) -  F(\w^*)}  \leq \frac{GD\sqrt{\rho}}{\sqrt{n}} + \frac{GD}{\sqrt{\rho n}} \leq \frac{GD}{\sqrt{n}}+\frac{GD}{\sqrt{\rho n}}
\end{align*}

For \emph{noisy-m-SGD}, we need to balance the trade-offs more directly. In  \cref{prop:upper_bound_noisy-m-sgd}, we arrived at that when using $\eta\leq \frac{1}{2L}$, the expected excess empirical risk is bounded by $\eta \cV^2 + \frac{D^2}{\eta T}$. Using the uniform stability bound of $\frac{G^2T \eta }{n}$, the expected excess population risk is bounded as,
\begin{align*}
    \E{ F(\hat \w_S) -  F(\w^*)} \leq \frac{G^2T \eta }{n} + \eta \cV^2 + \frac{D^2}{\eta T} = \eta \br{\frac{G^2 T}{n} + \cV^2} + \frac{D^2}{\eta T}
\end{align*}
Define $\tilde G^2=  \br{\frac{G^2 T}{n} + \cV^2}$, where $\cV^2 \lesssim G^2 + \sigma^2d \lesssim G^2 + \frac{G^2 Td}{n^2 \rho}$. Setting $\eta = \min\bc{\frac{1}{2L},\frac{D}{\tilde G \sqrt{T}}}$, we get,
\begin{align*}
    \E{ F(\hat \w_S) -  F(\w^*)} &\leq \frac{LD^2}{T} + \frac{\tilde G D}{\sqrt{T}}\\ 
    &\lesssim \frac{LD^2}{T}   +\frac{G\sqrt{T}D}{\sqrt{T}\sqrt{n}} + \frac{GD}{\sqrt{T}} + \frac{GD\sqrt{T}\sqrt{d}}{\rho n\sqrt{T}}  \\
    &=  \frac{LD^2}{T}  + \frac{GD}{\sqrt{n}} + \frac{GD}{\sqrt{T}} +\frac{GD \sqrt{d}}{\rho n}
\end{align*}

Setting $T = \max \bc{ \min \bc{\sqrt{n},\frac{\rho n}{\sqrt{d}}}, \frac{LD}{G}\min \bc{\sqrt{n},\frac{\rho n}{\sqrt{d}}}}$, and combining the two results finishes the proof.
\end{proof}

\subsection{Lower bounds}

In this section, we will prove a lower bound on excess population risk for any $\rho$-TV stable algorithm. As before, we will consider the Lipschitz constant $G$ and diameter $D$ to be both $1$, as the bounds scale naturally with these constants. We first define the following quantity, which denotes the lower bound on expected excess empirical risk of $\rho$-TV-stable algorithm with $n$ points.

\begin{align*}
   \hat\alpha(n,\rho) := \inf_{\cA: \rho\text{-TV-stable}}\sup_{S: \abs{S}=n} \mathbb{E}_\cA \hat F_S(\cA(S)) - \hat F_S(\hat \w_S) 
\end{align*}

\begin{theorem}
\label{thm:popln_risk_lower_bound}
For the problem of stochastic convex optimization, there exists a data distribution $\cD$, such that any $\rho$-TV-stable algorithm $\cA$ incurs expected excess population risk, bounded as follows
\begin{align*}
    \underset{S \sim \cD^n, \cA}{\mathbb{E}}F(\cA(S)) - F(\w^*) \geq \max \bc{\Omega \br{\frac{1}{\sqrt{n}}},  \hat\alpha(n,\rho) }
\end{align*}
\end{theorem}

\begin{proof}[Proof of Theorem \ref{thm:popln_risk_lower_bound}]
The $\frac{1}{\sqrt{n}}$ term follows directly since it is the lower bound for any algorithm, and so applies to $\rho$-TV stable algorithms as well.
We now focus on the second term $\hat\alpha(n,\rho)$
The proof is based on a standard reduction argument: if there is $\rho$-TV stable algorithm, which with $n$ i.i.d. samples from any distribution, achieves an expected excess population risk less than $\hat\alpha(n,\rho)$, then there is an $\rho$-TV stable algorithm which achieves an expected excess empirical risk less than $\hat\alpha(n,\rho)$ on any dataset of $n$ samples. Since the latter contradicts the definition of $\hat\alpha(n,\rho)$, this gives us that the expected excess population risk is at least or equal to $\hat\alpha(n,\rho)$. We now focus on the proof of the reduction. Consider a dataset $S$ of $n$ points. Consider $\cA$ as the following algorithm: sample $n$ i.i.d. samples from $S$, call this set $\tilde S$, and run \emph{some} $\rho$-TV algorithm $\tilde \cA$ on $\tilde S$. For a fixed $\tilde S$, from TV-stable property of $\tilde \cA$, for any neighbouring dataset $\tilde S'$ with one point differing, we have that $\text{TV}(\tilde \cA(\tilde S),\tilde \cA(\tilde S')) \leq \rho$.
Furthermore, using the \emph{group} property of TV-stability, for any dataset $\tilde S'$, we have $\text{TV}(\tilde \cA(\tilde S),\tilde \cA(\tilde S')) \leq \Delta(S,S')\rho$.  Using the maximal coupling characterization of total variation distance, we have that there exists a coupling $\tilde \pi$ of random variables $\tilde \cA(\tilde S)$ and $\tilde \cA(\tilde S')$ such that $\text{TV}(\tilde \cA(\tilde S),\tilde \cA(\tilde S')) = \mathbb{E}_{\tilde \pi}\mathbb{1}\bc{\tilde \cA(\tilde S) \neq \tilde \cA(\tilde S')}$. We now show that the algorithm $\cA$ is also $\text{TV}$-stable for dataset $S$.

Consider dataset $S'$ of $n$ points which differs from $S$ in the first sample. We now generate $\tilde S'$ by drawing $n$ i.i.d samples from $S'$. For this, consider the following coupling:  we draw $n$ i.i.d samples from $S$, call it $\tilde S$. For every draw of the first sample, replace it by the first sample of $S'$, call it $\tilde S'$. It is easy to check the $\tilde S$ and $\tilde S'$ are i.i.d. samples from $S$ and $S'$ respectively. We now proceed to show the $\cA$ is $\rho$-TV stable. We will use the fact the total variation distance is at most the probability of disagreement under any coupling. The coupling $\pi$ we consider is that we first generate $\tilde S$ and $\tilde S'$ using the aforementioned coupling, and then use the coupling $\tilde \pi$ which achieves total variation distance for worst-case fixed neighbouring datasets $\tilde S$ and $\tilde S'$. We have,
\begin{align*}
    \text{TV}(\cA(S),\cA(S')) &\leq  \mathbb{E}_{\pi}\mathbb{1} \bc{\cA(S) \neq \cA(S')} =  \mathbb{E}_{\pi}\mathbb{1}\bc{\tilde \cA(\tilde S) \neq \tilde \cA(\tilde S')}  \\
   & \leq \mathbb{E}_{\pi} \sup_{\tilde S, \tilde S'}\mathbb{1}\bc{\tilde \cA(\tilde S) \neq \tilde \cA(\tilde S')}  \leq   \mathbb{E}_{\pi} \Delta(\tilde S,\tilde S')\rho = \rho
\end{align*}
where the last equality follows from direct computation of $\Delta(\tilde S,\tilde S')$: number of differing samples, under coupling $\pi$.

We now proceed to the accuracy guarantee. From straight-forward computation, the excess population risk, under the sampling of $\tilde S$, is $\hat F_S(\tilde \cA(\tilde S)) -  \hat F_S(\w^*_S)$ - this is the excess empirical risk for dataset $S$. So if we have an upper bound on excess population risk using algorithm $\tilde \cA$, we have an upper bound on excess empirical risk for dataset $S$. This completes the reduction argument and hence the proof.
\end{proof}

%% file: sections/approximateunlearning.tex
\section{Algorithms for approximate unlearning}
\label{sec:approx-unlearning}
We first define the notion of approximate unlearning based on differential privacy (DP).
\begin{definition}[$(\epsilon,\delta)$-approximate-unlearning]
We say a procedure $(\mathbf{A},\mathbf{U})$ satisfies $(\epsilon,\delta)$-approximate-unlearning \unlearning~if for any 
$S, S' \subset \cX^*$ such that $\Delta(S,S')=1$ and for any measurable event $\cE \in \text{Range}(\mathbf{U}) \cap \text{Range}(\mathbf{A})$, with probability at least $1-\delta$,
\begin{align*}
   \e^{-\epsilon}\P{\mathbf{U}(\mathbf{A}(S),S'\backslash S \cup S\backslash S') \in \cE} \leq \P{\mathbf{A}(S') \in \cE} \leq e^{\epsilon}  \P{\mathbf{U}(\mathbf{A}(S),S'\backslash S \cup S\backslash S') \in \cE}
\end{align*}
\end{definition}

We now define $(g,\epsilon,\delta)$-group differential privacy. 
\begin{definition}[$(g,\epsilon,\delta)$-group differential privacy]
    An algorithm $\mathcal{A}$ satisfies $(\epsilon,\delta)$-differential privacy if for any two datasets $S$ and $S'$ such that $\Delta(S,S')\leq g$, for any measurable event $\cE \in \text{Range}(\cA)$, it satisfies
\begin{align*}
   \mathbb{P}(\mathcal{A}(S)\in \cE) \leq e^\epsilon \mathbb{P}(\mathcal{A}(S')\in \cE)+\delta
\end{align*}
\end{definition}

\begin{remark}\cite{dwork2014algorithmic}
If an algorithm satisfies $(\epsilon,\delta)$-DP, then for any $g \in \bbN$, it satisfies $(g, g \epsilon, g e^{(g-1)\epsilon}\delta)$-group differential privacy.
\end{remark}

We now define \emph{privateCompute} oracle which, basically is a differentially private solver for the said task.

\begin{definition}[\emph{privateCompute}$(S,\epsilon,\delta$) oracle]
    For a problem instance, given a dataset $S$ of $n$ points, and privacy parameters $\epsilon$ and $\delta$, a \emph{privateCompute} oracle outputs a $(\epsilon,\delta)$-differentially private solution with accuracy $\alpha_{\text{private}}(n,\epsilon,\delta)$
\end{definition}

We now give a very simple algorithm (\cref{alg:interpolate}) based on the observation above using privateCompute oracle calls.

\begin{algorithm}[!htpb]
\caption{Approximate unlearning}
\label{alg:interpolate}
\begin{algorithmic}[1]
\REQUIRE{$\epsilon,\delta,\rho$}
\STATE $i \leftarrow 0$
\STATE $\hat \w_S \leftarrow$ \emph{PrivateCompute} $\br{S,\rho \alpha, \beta_{\text{group}}\br{\left\lfloor \frac{1}{\rho}\right\rfloor,\rho \alpha, \rho \beta}}$
\STATE \textit{// Observe $k$ edit requests}
\WHILE{$t=1,2,\ldots, k$} 
\STATE $S_t \leftarrow \text{Update dataset}$(edit request)
\STATE $i+=1$
\IF{$i=\left\lfloor \frac{1}{\rho} \right\rfloor$}
\STATE $\hat \w_{S_t} \leftarrow$\emph{PrivateCompute}$\br{S_t,\rho \alpha, \beta_{\text{group}}\br{\left\lfloor \frac{1}{\rho}\right\rfloor,\rho \alpha, \rho \beta}}$
\STATE $i \leftarrow 0$
\ENDIF
\ENDWHILE
\end{algorithmic}
\end{algorithm}

\begin{theorem}
\label{claim:interpolate}
Given a set of $n$ data points to start with, and observing a stream of $k$ requests, at any time $t$ in the stream, the following hold about Algorithm \ref{alg:interpolate}:
\begin{enumerate}
    \item It satisfies $(\epsilon,\delta)$-approximate unlearning.
     \item The unlearning runtime for $k$ requests is at most $2\rho k$ privateCompute oracle calls.
    \item The accuracy is at most $\alpha_\text{private}\br{\frac{n}{2},\rho \epsilon,\delta_{\text{group}}\br{\left\lfloor \frac{1}{\rho}\right \rfloor, \rho \epsilon, \rho \delta}}$.
\end{enumerate}
\end{theorem}

\begin{proof}[Proof of \cref{claim:interpolate}]
Consider a point $t$ in the stream, and let $j$ be such that $j\left\lfloor \frac{1}{\rho}\right \rfloor \leq t \leq (j+1)\left\lfloor \frac{1}{\rho}\right \rfloor $. 
Since the algorithm uses \emph{privateCompute} with parameters $\rho \epsilon$ and $\delta_{\text{group}}\br{\left\lfloor \frac{1}{\rho}\right \rfloor, \rho \epsilon, \rho \delta}$, it satisfies $(\rho \epsilon,\delta_{\text{group}}\br{\left\lfloor \frac{1}{\rho}\right \rfloor , \rho \epsilon, \rho \delta}$ differential privacy and hence $\br{\left\lfloor \frac{1}{\rho}\right \rfloor ,\alpha,\delta}$-group privacy. 
 Therefore, for any such $t$, since the number of requests after time $j \left\lfloor \frac{1}{\rho}\right \rfloor$ is less that or equal to  $\left\lfloor \frac{1}{\rho}\right \rfloor$, this implies it satisfies $(\alpha,\delta)$-approximate unlearning.
For the second part of the claim, note that  for $k$ updates, the number of times the algorithm calls \emph{privateCompute} is $\frac{k}{\left\lfloor \frac{1}{\rho}\right \rfloor}$.
Note that $\frac{1}{\rho} \geq 1$, so if $1 \leq \frac{1}{\rho} < 2$, then $\left\lfloor \frac{1}{\rho} \right\rfloor =1$, which gives that the update complexity is $k \leq 2k$. However, if $\frac{1}{\rho} \geq 2$, we have that $\frac{k}{\left\lfloor \frac{1}{\rho} \right\rfloor} \leq \frac{k}{\br{\frac{1}{\rho}-1}} \leq 2k\rho$, which gives the  update complexity is at most $2 \rho k$ in both cases.
For the third part of the claim, at time $j\left\lfloor \frac{1}{\rho}\right \rfloor \leq t \leq (j+1)\left\lfloor \frac{1}{\rho}\right \rfloor $,
the private estimator is computed with $n\br{j\left\lfloor \frac{1}{\rho}\right \rfloor)} \geq \frac{n}{2}$, by assumption.
Moreover the privacy parameters of the algorithm are $\rho \epsilon$ and $\delta_{\text{group}}\br{\left\lfloor \frac{1}{\rho}\right \rfloor, \rho \epsilon, \rho \delta}$ which gives the claimed accuracy bound.
\end{proof}

As an example, consider $\rho = \frac{1}{\sqrt{k}}$,
we first do \emph{privatecompute} with parameters $(\epsilon/\sqrt{k},,\delta_{\text{group}}(\epsilon/\sqrt{k},\delta/\sqrt{k},\lfloor \sqrt{k}\rfloor))$. Since after $\lfloor \sqrt{k} \rfloor$ edit requests, we would no longer satisfy the unlearning guarantee, so we now need to do \emph{privateCompute} again. However note that we would only need to do \emph{privateCompute} $\sqrt{k}$ times which gives the update computation cost.

\paragraph{Example: Convex ERM.} For convex ERM, we can use \cite{bassily2014private} to instantiate the oracle. In this case, accuracy $\alpha_{\text{private}}$ is the excepted excess empirical risk, which is $\alpha_{\text{private}}(n,\epsilon,\delta) = O\br{\frac{GD\sqrt{d}\sqrt{\log{1/\delta}}}{n \epsilon}}$. Using \cref{alg:interpolate}, given $0\leq \rho \leq 1$, at any point in the stream, we have,
\begin{align*}
    \E{\hat F_S(\hat \w_S)- \hat F_S(\w^*_S)} &\leq \alpha_\text{private}\br{\frac{n}{2},\rho \epsilon,\delta_{\text{group}}\br{\left\lfloor \frac{1}{\rho}\right \rfloor, \rho \epsilon, \rho \delta}} \\
    &\leq O\br{\frac{GD \sqrt{d} \sqrt{\log{(1/\rho \delta) \exp{\rho \epsilon\br{\left \lfloor\frac{1}{\rho}\right \rfloor -1}}}}}{n \rho \epsilon}} \\
    & \leq  O\br{GD \br{\frac{\sqrt{d}\sqrt{\log{1/\rho \delta}}}{n \rho \epsilon} + \frac{\sqrt{d}}{ n \rho \sqrt{\epsilon}}}} 
    \\
     & \leq  O\br{\frac{GD\sqrt{d}\sqrt{\log{1/\rho \delta}}}{n \rho \epsilon}} 
\end{align*}
where the last inequality holds when $\frac{\epsilon}{\log{1/\rho \delta}} \leq O(1)$, which usually is the case in DP, and so is a reasonable regime. We now compare against \cite{neel2020descent} - we ignore $G,D$ and $\text{log}$ factor in both the bounds. To have the same runtime, we need $\rho k Tm = k^2 n \iff \rho = \frac{kn}{Tm} = \frac{kd}{\epsilon^2n}$, where in the last equality we substituted $Tm = \frac{(\epsilon n)^2}{d}$, parameters for the DP convex ERM algorithm. Our accuracy bound is $O\br{\frac{\sqrt{d}}{n \rho \epsilon}} = O\br{\frac{\epsilon}{k \sqrt{d}}}$, which is smaller than that of \cite{neel2020descent}, when $\frac{\epsilon}{k \sqrt{d}} \leq \br{\frac{\sqrt{d}}{n k \epsilon}}^{2/5} \iff \epsilon^7 \leq \frac{\sqrt{d}^7 k^3}{n^2} \iff \epsilon \leq \frac{\sqrt{d}k^{3/7}}{n^{2/7}}$. Hence in regimes where the unlearning parameter $\epsilon$ is small enough, which corresponds to a stronger unlearning criterion, this algorithm is better than that of \cite{neel2020descent}.

%% file: sections/experiments.tex
\section{Experiments}
We run experiments on MNIST \cite{lecun1998mnist}, a standard digit classification computer vision dataset with $10$ classes. We train a logistic regression model, which can be formulated as a smooth convex risk minimization problem. 
Starting with a training dataset of $60$k points, we simulate a stream of 300 deletions of randomly chosen points and 300 insertions of new points, randomly permuted. 
We use Algorithm \ref{alg:noisy-m-sgd} as the learning algorithm, and the corresponding \cref{alg:unlearn-noisy-m-sgd} as the unlearning algorithm. 
We train for $T=200$ iterations, with mini-batch of size $m=50$ with a constant learning rate $\eta = 0.05$.
We run experiments on a range of values of standard deviation $\sigma$ of Gaussian noise, from $0$ to $1.1$ separated by the intervals of size $0.005$.
For every value of $\sigma$, we run 10 instances of the whole unlearning procedure and report average performance: accuracy and number of unstable edits (i.e. number of times a recompute is triggered) , and their standard deviations.
Note that $\sigma=0$ corresponds to standard mini-batch SGD, and therefore the accuracy obtained is the accuracy for the standard method with the aforementioned setting of the hyperparameters. Moreover, the $\sigma=0$ setting also corresponds to Algorithm \ref{alg:sub-sample-GD}, and therefore the corresponding  unlearning algorithm \cref{alg:unlearn-subsample-gd} handles edits for this case.

In \cref{fig:exp1}, we report the test accuracy (fraction of mis-classified samples in the test set) and the  number of unstable edits i.e the number of times a retrain is triggered, as a function of $\sigma$. As expected, as $\sigma$ increases, we get less unstable edits. Interestingly, for small values of $\sigma$, like $0.1$, the degradation in accuracy is not as much as compared to decrease in the the number of unstable edits. Furthermore, recall that the unlearning algorithm triggers a partial recompute - \cref{fig:exp2} plots the average number of iterations done after an unstable edit compared to the number of iterations for a full recompute.

\begin{figure}
\label{fig:experiments}
\centering
\begin{subfigure}{.5\textwidth}
  \centering
  \includegraphics[scale=0.4]{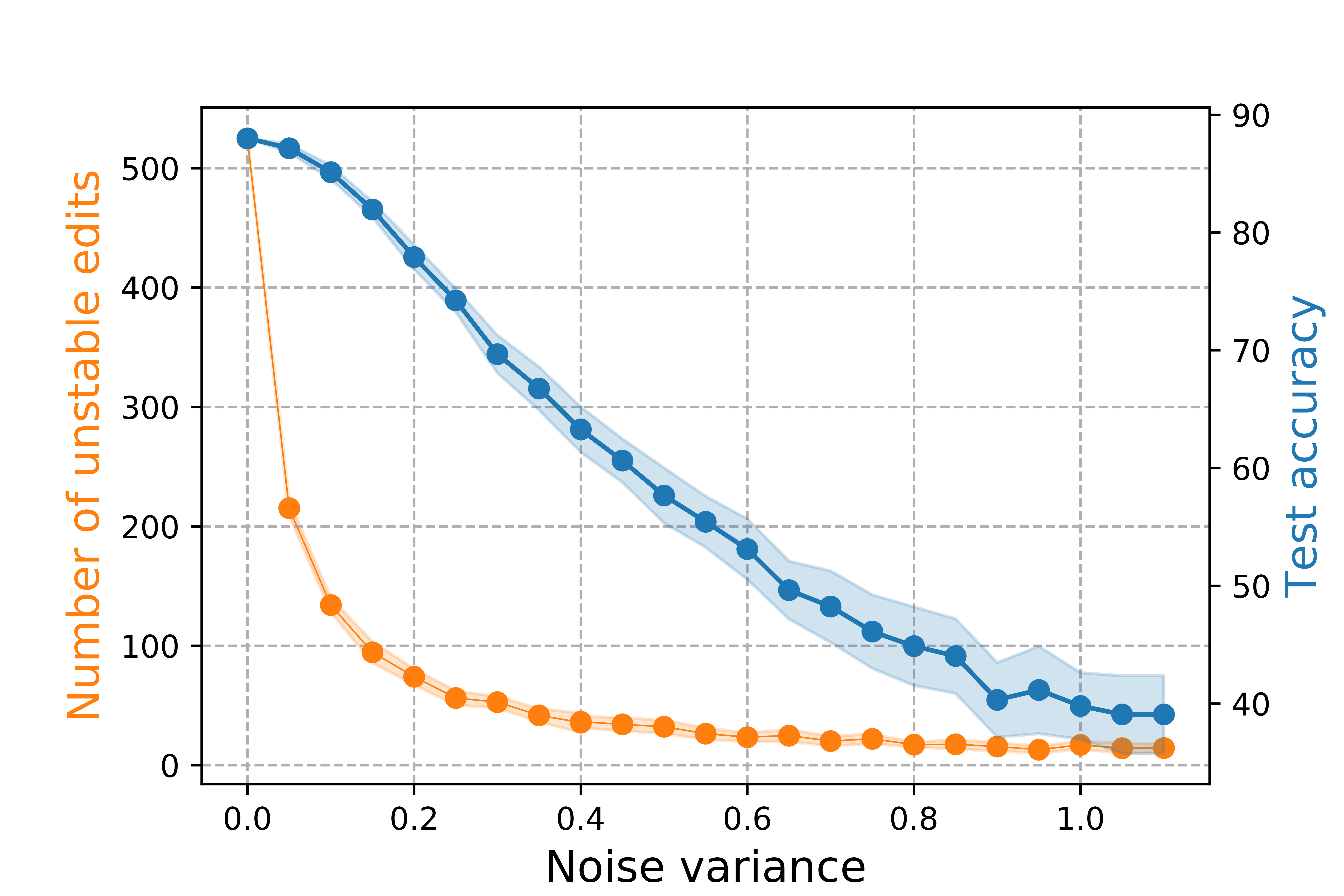}
  \caption{\small Accuracy and number of unstable edits as a function of variance of noise used.}
  \label{fig:exp1}
\end{subfigure}%
\begin{subfigure}{.5\textwidth}
  \centering
  \includegraphics[scale=0.4]{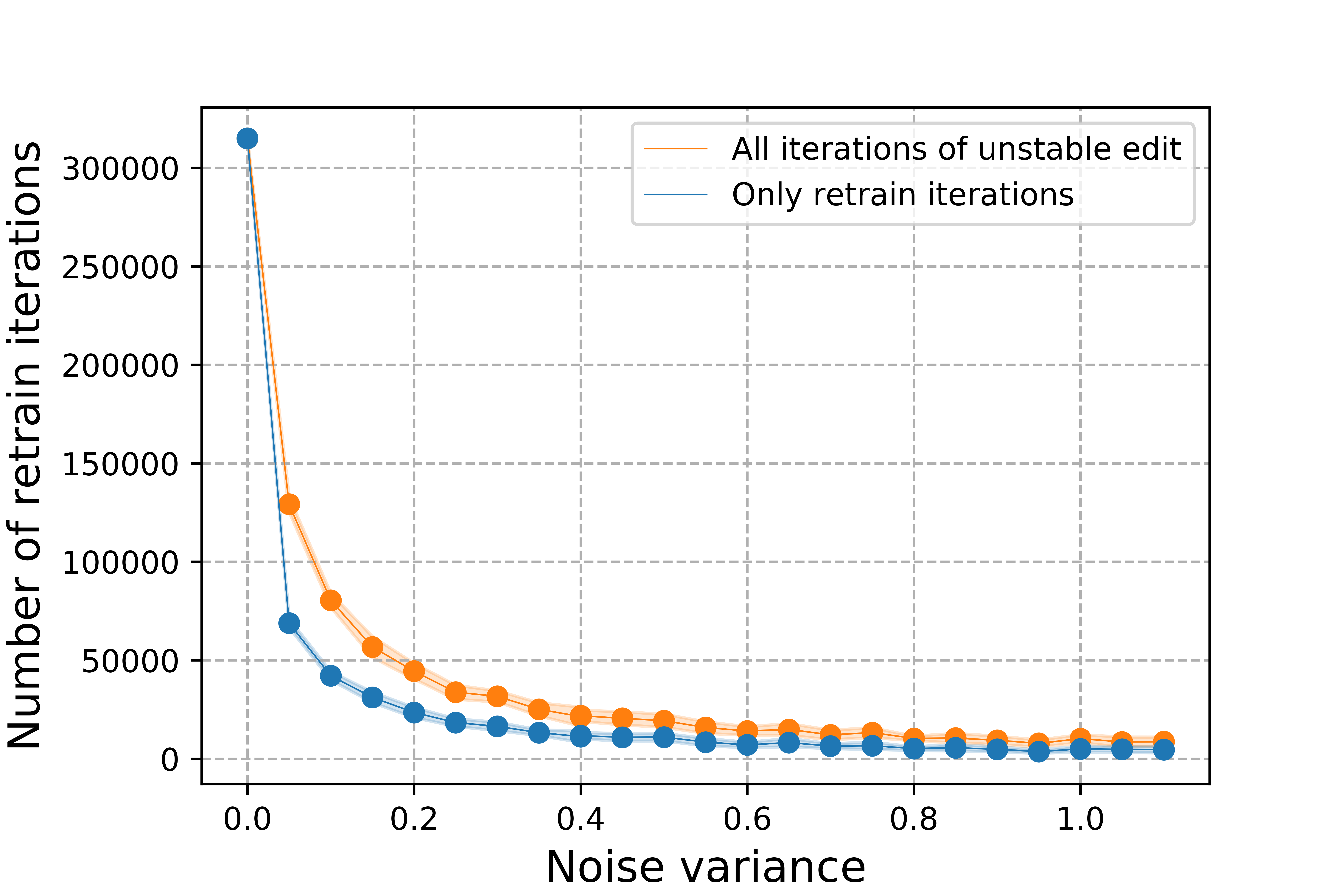}
  \caption{{\small Number of retraining iterations by unlearning algorithm  compared to all full retraining (all iterations)}}
  \label{fig:exp2}
\end{subfigure}
\end{figure}